\documentclass{article}

\PassOptionsToPackage{numbers, compress}{natbib}

\usepackage{arxiv}

\usepackage[pdftex]{graphicx}
\usepackage[pdftex]{color}
\usepackage{wrapfig}
\usepackage{bm,enumerate,amsmath,amssymb,amsthm}
\usepackage{subcaption}
\usepackage{array}
\usepackage{multirow}
\usepackage{algorithm}
\usepackage{algpseudocode}

\usepackage[utf8]{inputenc} 
\usepackage[T1]{fontenc}    
\usepackage{hyperref}       
\usepackage{url}            
\usepackage{booktabs}       
\usepackage{amsfonts}       
\usepackage{nicefrac}       
\usepackage{microtype}      
\usepackage{xcolor}         

\newtheorem{proposition}{Proposition}
\newcommand{\std}[1]{{\scriptsize$\pm$#1}}

\usepackage[numbers]{natbib}

\usepackage{wrapfig}
\usepackage{bm,enumerate,amsmath,amssymb,amsthm}
\usepackage{subcaption}
\usepackage{array}
\usepackage{multirow}
\usepackage{algorithm}
\usepackage{algpseudocode}

\newcommand{\changed}[1]{#1}

\begin{document}
\title{CT-OT Flow: Estimating Continuous-Time Dynamics from Discrete Temporal Snapshots
}

\author{
  Keisuke Kawano$^\star$, Takuro Kutsuna, Naoki Hayashi, Yasushi Esaki, Hidenori Tanaka   \\
  Toyota Central R\&D Labs., Inc., Japan\\
  \texttt{$^\star$kawano@mosk.tytlabs.co.jp} \\
}

\date{}

\maketitle

\begin{abstract}
  In many real-world settings--e.g., single-cell RNA sequencing, mobility sensing, and environmental monitoring--data are observed only as temporally aggregated snapshots collected over finite time windows, often with noisy or uncertain timestamps, and without access to continuous trajectories.
  We study the problem of estimating continuous-time dynamics from such snapshots.
  We present Continuous-Time Optimal Transport Flow (CT-OT Flow), a two-stage framework that (i) infers high-resolution time labels by aligning neighboring intervals via partial optimal transport (POT) and (ii) reconstructs a continuous-time data distribution through temporal kernel smoothing, from which we sample pairs of nearby times to train standard ODE/SDE models.
  Our formulation explicitly accounts for snapshot aggregation and time-label uncertainty and uses practical accelerations (screening and mini-batch POT), making it applicable to large datasets.
  Across synthetic benchmarks and two real datasets (scRNA-seq and typhoon tracks), CT-OT Flow reduces distributional and trajectory errors compared with OT-CFM, [SF]\(^{2}\)M, TrajectoryNet, MFM, and ENOT.
\end{abstract}

\section{Introduction}
In various fields such as biology~\cite{tong2020trajectorynet,tong2024simulation}, urban studies~\cite{niu2023understanding}, and medical data analysis~\cite{bunne2023learning}, data are often collected not continuously but as \emph{discrete temporal snapshots}, i.e., at distinct points in time.
Researchers then need to estimate the underlying dynamics of individual data points from these snapshots.
For example, in single-cell RNA sequencing (scRNA-seq)~\cite{zheng2017massively}, cellular dynamics are inferred from discrete temporal snapshots at different culture durations because cells are destroyed during measurement, making it impossible to track each cell over time.
Importantly, even within the same experimental snapshot, cells may correspond to a range of survival times, since the exact birth times of individual cells cannot be synchronized or directly observed.
Thus, a single time label in practice aggregates cells with heterogeneous times, introducing intrinsic uncertainty into the observation timestamps~\cite{la2018rna}.
Similarly, in urban mobility studies~\cite{zheng2014urban} or environmental sensor data~\cite{yick2008wireless}, it is challenging to continuously track individuals or sensors, and one typically only obtains data sampled at coarse time intervals.

To model such dynamics, continuous normalizing flow~\cite{chen2018neural} and Schr\"odinger bridge~\cite{leonard2013survey}-based methods~\cite{tong2020trajectorynet,tong2024simulation,tong2024improving} have been proposed.
These methods fit ordinary differential equations (ODEs) or stochastic differential equations (SDEs) to temporal snapshots, and thereby infer time-evolving trajectories.
However, they face two major limitations:
\textbf{(1)~Time discretization.} Due to constraints such as measurement costs and data sparsity, temporal snapshots are often recorded at discrete intervals.
Consequently, all data points collected within the same interval are forced to share a single, coarse-grained time label.
For instance, in scRNA-seq experiments~\cite{moon2018embryoid}, cells harvested between day 1 and day 3 may be labeled as $t=1$, and those harvested between day 4 and day 6 as $t=2$.
Methods that ignore this discretization may naively assume that all $t=1$ cells transition directly to $t=2$ along the shortest path in feature space, yielding suboptimal or even misleading dynamics.
\textbf{(2)~Time-label uncertainty.} The recorded times themselves can be noisy.
In scRNA-seq experiments, individual cells’ true ``birth times'' cannot be directly measured, introducing uncertainty into the observation timestamps~\cite{la2018rna}.
Ignoring this uncertainty when estimating dynamics can result in a model that diverges from the true process.

In this work, we propose \textbf{Continuous-Time Optimal Transport Flow (CT-OT Flow)}, a method designed to address both time discretization and time-label uncertainty to recover more accurate dynamics.
CT-OT Flow has two major components.
\textbf{(1) High-resolution time label inference via partial optimal transport.}
CT-OT Flow employs partial optimal transport (POT)~\cite{bonneel2019spot,figalli2010optimal} to infer high-resolution time labels for each data point, even if the raw observations are aggregated into discrete intervals.
By leveraging POT, we can incorporate sample similarities when estimating high-resolution time labels.
\textbf{(2) Kernel-based construction of data distribution in continuous-time.}
CT-OT Flow then builds data distributions in continuous-time via a kernel-based smoothing approach.
This step enables robust estimation of ODE/SDE even in the presence of substantial noise or uncertainty in the times.
Unlike existing methods~\cite{tong2024simulation,tong2024improving} that simply ignore potential misalignment in time labels, CT-OT Flow can approximate real-world dynamics more accurately.

Figure~\ref{fig:intro}(a) illustrates a toy 2-D example, where the arrow represents the true dynamics and points indicate observations with discrete timestamps, shown in purple and orange.
Even though the true dynamics follow a curved path (Fig.~\ref{fig:intro}(a)), estimating an ODE directly from observations with discrete timestamps yields a spurious linear flow (Fig.~\ref{fig:intro}(b)).
CT-OT Flow first infers high-resolution time labels, indicated by the point colors in Fig.~\ref{fig:intro}(c), and then reconstructs a continuous trajectory that matches the ground truth.

\textbf{Contributions.} (1) We introduce CT-OT Flow, which first applies optimal transport theory to infer high-resolution time labels (\S\ref{sec:step1}) and then uses kernel-based smoothing to generate continuous-time data distributions (\S\ref{sec:step2}).
This framework allows for robust ODE/SDE estimation in scenarios where observation times are both discretized and noisy.
(2) Although CT-OT Flow is formulated as a mixed-integer linear programming problem, we provide a relaxed formulation and an efficient algorithm that can solve it within practical computation times (\S\ref{sec:step1}).
(3) Through numerical experiments on multiple tasks, including a real scRNA-seq dataset, we demonstrate that CT-OT Flow achieves more accurate dynamics estimation than conventional methods (\S\ref{sec:experiments}).
The source code has been included in supplementary materials. Upon acceptance, we will publicly release the code.

\begin{figure}[t]
  \centering
  \includegraphics[width=\textwidth]{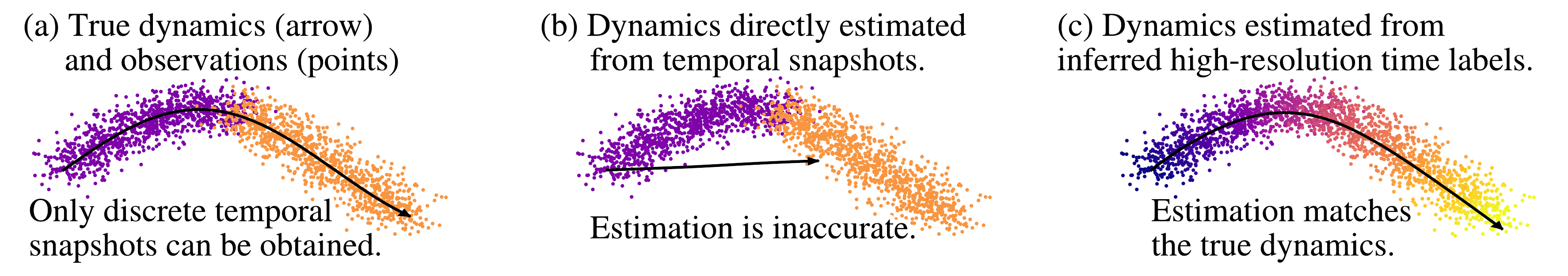}
  \caption{Motivating example for CT-OT Flow.
    (a) True dynamics (arrow) and observations (points).
    (b) Dynamics estimated directly from the discrete timestamps are inaccurate.
  (c) CT-OT Flow: Dynamics recovered accurately follow the ground truth by inferring high-resolution time labels. \label{fig:intro}}
\end{figure}

\section{Related work}
\paragraph{Dynamics estimation from temporal snapshots}
When individual trajectories cannot be observed directly, numerous methods have been proposed for estimating dynamics from temporal snapshots~\cite{tong2020trajectorynet,tong2024simulation,chen2018neural,rohbeck2025modeling,lipman2023flow,liu2023flow,korotin2023light,koshizuka2022neural,gushchin2024entropic,klein2024genot}.
Neural ODE~\cite{chen2018neural}, Flow Matching~\cite{lipman2023flow}, and Rectified Flow~\cite{liu2023flow} represent the vector fields as neural networks, enabling learning a smooth trajectory that connects an initial distribution to a terminal one.
Stochastic formulations based on SDEs and Schr\"odinger bridges~\cite{leonard2013survey} model dynamics as stochastic processes~\cite{tong2024simulation,korotin2023light,koshizuka2022neural,klein2024genot,song2021scorebased,neklyudov2023action,chen2023deep}.
To enforce smoothness across multiple time points, multi-marginal OT and spline interpolation have been explored~\cite{chen2023deep, lee2025multimarginal,banerjee2025efficient}.
These methods interpolate trajectories by imposing global smoothness, but they generally assume that accurate time labels are available.
Other works estimate more realistic dynamics by adding data density-based regularization~\cite{tong2020trajectorynet,tong2024simulation,koshizuka2022neural} or by constraining the trajectories to lie on the data manifold~\cite{kapusniak2024metric,chen2023flow}.
However, none of these methods explicitly account for practical data-collection constraints, such as discretized time labels or uncertainty in observation times.
These issues arise from the way temporal snapshots are obtained in practice (e.g., scRNA-seq experiments).
Ignoring such data-level constraints can limit the accuracy of the resulting dynamics estimation.

\paragraph{Optimal transport}
Optimal transport theory provides a framework for transforming one distribution into another by minimizing a chosen transport cost.
In addition to defining metrics such as the Wasserstein distance~\cite{peyre2019computational}, it has also been applied to dynamics estimation~\cite{tong2024improving,lipman2023flow,gushchin2024entropic}.
In particular, OT-CFM~\cite{tong2024improving} employs an optimal transport approach to infer a joint distribution between two datasets, thereby improving the learning efficiency of conditional flow matching~\cite{lipman2023flow}.
Moreover, the relationship between the Schr\"odinger bridge~\cite{leonard2013survey} and the entropic OT problem~\cite{cuturi2013sinkhorn} is well known~\cite{gushchin2024entropic}.
However, these approaches assume that the observation times are explicitly known.
Our work differs fundamentally by inferring high-resolution time label estimates for each data point, thereby accommodating discretized and noisy observation times in dynamics estimation.

\paragraph{Pseudotime analysis}
Because scRNA-seq snapshots are destructive, true temporal trajectories cannot be observed directly.
Consequently, pseudotime analysis methods have been developed to infer a relative ordering of cells from their observed states~\cite{bunne2023learning,la2018rna,Bergen2020,trapnell2014dynamics,qiu2017reversed,sha2024reconstructing}.
For example, Monocle~\cite{trapnell2014dynamics,qiu2017reversed} leverages geometric or graph structures in cell-state space to estimate trajectories, whereas RNA velocity~\cite{la2018rna} and its extension scVelo~\cite{Bergen2020} predict future cell states by comparing spliced and unspliced mRNA abundances.
Related optimal-transport–based approaches couple distributions across known capture times to infer lineage flows and population dynamics from sequential snapshots~\cite{bunne2023learning}.
In summary, pseudotime methods primarily recover a relative cell ordering and local transition structures; they typically do not calibrate to absolute time, explicitly model interval aggregation, or return a generative dynamical law.
By contrast, CT-OT Flow tackles a different task: from discretized, temporally aggregated observations it estimates a generative continuous-time model that supports simulation, interpolation, and forecasting---capabilities that standard pseudotime pipelines do not provide.
The approaches are complementary (pseudotime can aid initialization or visualization), but their objectives, outputs, and evaluation criteria are fundamentally different: pseudotime is an ordering problem; CT-OT Flow is continuous-time dynamics estimation from temporally aggregated data.

\paragraph{Continuous-time representation learning in mobility and environmental domains}
Continuous-time representation learning has also been explored in mobility sensing and environmental monitoring.
In mobility-related domains, Neural ODE/CDE-based models have been used for spatio-temporal traffic forecasting~\cite{choi2022graph}.
In environmental domains, continuous-time neural models have recently been applied to weather and climate forecasting, such as physics-informed Neural ODEs for climate forecasting~\cite{verma2024climode}.
These studies demonstrate the broad relevance of continuous-time modeling for irregularly sampled and spatio-temporal observations, while our work focuses on a complementary setting: estimating continuous-time dynamics from temporally aggregated snapshots with uncertain timestamps.

\section{Preliminaries}
\label{sec:preliminary}
\subsection{Notation}
Vectors are denoted by bold lowercase letters (e.g., $\boldsymbol a$).
Let $\boldsymbol 1_d$ denote the $d$-dimensional vector whose elements are all ones.
We write $\|\cdot\|_2$ for the $L^2$ norm and $\delta_{\boldsymbol a}(\cdot)$ for the Dirac delta function centered at $\boldsymbol a$.
For a set of $N$ data points $X=\{\boldsymbol x^{(i)}\mid i=1,\dots, N\}$, we define the corresponding empirical distribution as $\hat p_{X} = \hat p_{X}(\boldsymbol x) =   \frac{1}{N}\sum_{i=1}^N \delta_{\boldsymbol x^{(i)}}(\boldsymbol x)$.

\subsection{Partial optimal transport}
For two finite sets $X=\{\boldsymbol x^{(1)},\dots, \boldsymbol x^{(N_x)}\}$ and $Y = \{\boldsymbol y^{(1)},\dots,\boldsymbol y^{(N_y)}\}$, we define POT between them as
\begin{align}
  \label{eq:partial}
  \text{POT}_{(\tau_x, \tau_y)} (\hat p_X, \hat p_Y) = \min_{P\in U_{(\tau_x,\tau_y)}} \sum_{i=1}^{N_x}\sum_{j=1}^{N_y} P_{ij}\|\boldsymbol x^{(i)} - \boldsymbol y^{(j)}\|_2^2,
\end{align}
where $P$ is an $N_x\times N_y$ matrix with entries $P_{ij}\in[0,1]$.
The feasible set $U_{(\tau_x,\tau_y)}$ is defined by
\begin{align}
  \label{eq:partial_boundary}
  U_{(\tau_x,\tau_y)} = \left\{P\in [0,1]^{N_x\times N_y} \ \middle | \ P\boldsymbol 1_{N_y} \leq \tau_x \frac{\boldsymbol{1}_{N_x}}{N_x},\ P^\top\boldsymbol 1_{N_x} \leq \tau_y\frac{\boldsymbol{1}_{N_y}}{N_y},\ \boldsymbol 1_{N_x}^\top P \boldsymbol 1_{N_y} = 1 \right\},
\end{align}
where $\tau_x, \tau_y\geq 1$ are scalar weights associated with the two empirical distributions.
Intuitively, POT transports only a fraction $1/\tau_x$ (or $1/\tau_y$) of the total mass (see App.~\ref{appendix:pot} for details).
The 2-Wasserstein distance~\cite{peyre2019computational} between $\hat p_X$ and $\hat p_Y$ can be expressed as $\mathcal W(\hat p_X, \hat p_Y) = \left(\text{POT}_{(1,1)}(\hat p_X, \hat p_Y)\right)^{\frac{1}{2}}$. 

\textbf{Link to our method.}
In CT-OT Flow, we solve the POT problem between two contiguous time-interval snapshots to assign each data point to high-resolution time labels.

\section{Proposed method}
\subsection{Problem setting}
Let $p_t: \mathbb{R}^d \to \mathbb{R}$, $\boldsymbol{x} \mapsto p_t(\boldsymbol{x})$ be the probability density function of the data-generating distribution at time $t$
and $\{\boldsymbol x_t\}$ be a stochastic process on $\mathbb{R}^d$ indexed by $\{t\}$, where $\boldsymbol x_t\sim p_t(\boldsymbol x)$, and $d$ is the dimension of the data.
Consider the distribution of data observed over an interval from $t$ to $t+\Delta t$:
\begin{align}
  \label{eq:integral}
  p_{[t, t+\Delta t]}(\boldsymbol x) & = \frac{1}{Z}\int_{t}^{t+\Delta t} p_{t'}(\boldsymbol x) p(t')dt',
\end{align}
where $\Delta t>0$, $p(t')$ is the probability density function of the observation time, and $Z$ is the normalization constant defined as $Z=\int_{t}^{t+\Delta t}p(t')dt'$.

In this work, we consider \textit{temporal snapshots} collected between discrete timestamps $t_j$ and $t_{j+1}$, where $t_j < t_{j+1}$.
For each $j=1,\dots, T$,  we define the temporal snapshot as
\begin{align}
  X_{[t_j, t_{j+1}]} = \{\boldsymbol x^{(i)}\mid \boldsymbol x^{(i)}\sim p_{[t_j, t_{j+1}]}(\boldsymbol x),\  i=1,\dots,N_{[t_j, t_{j+1}]}\},
\end{align}
where $N_{[t_j, t_{j+1}]}=|X_{[t_j, t_{j+1}]}|$ denotes the number of data points observed within the time interval $[t_j, t_{j+1}]$.
Because the recorded observation times may be imprecise due to noise, $X_{[t_j, t_{j+1}]}$ can also include samples actually drawn from $p_{[t_j-\epsilon_t,\; t_{j+1}+\epsilon_t]}(\boldsymbol x)$ for some $\epsilon_t>0$.

Our goal is to recover the true continuous-time dynamics of $\boldsymbol{x}_t$ by training an ODE/SDE model using temporal snapshots $X_{[t_j,t_{j+1}]}$ for $j=1,\dots, T$.
Most existing methods estimate the ODE/SDE models under the simplifying assumption that data points observed in the interval $[t_j, t_{j+1}]$ are drawn from the single time distribution $p_{t_j}(\boldsymbol{x})$, rather than from the aggregated distribution $p_{[t_j,t_{j+1}]}(\boldsymbol{x})$.
As a result, the trained models can diverge substantially from the true dynamics.
In contrast, our approach first estimates the distributions $p_t(\boldsymbol{x})$ from the temporal snapshots and then fits an ODE/SDE model, yielding a more accurate reconstruction of the underlying process.

\changed{
  \paragraph{Scope and assumptions}
  CT-OT Flow is designed for temporally aggregated snapshots, where each observed set $X_{[t_j, t_{j+1}]}$ represents samples collected over a finite interval rather than an instantaneous distribution at a precise time.
  The method relies on three main assumptions.
  First, the underlying distribution $p_t(\boldsymbol x)$ changes sufficiently smoothly in time so that distributions near the boundary of two contiguous intervals are similar; this assumption may be less appropriate under abrupt regime changes or strongly non-contiguous observations, which we further analyze in Apps.~\ref{appendix:non_contiguous} and~\ref{appendix:failure_case_for_boundary_extraction}.
  Second, the within-interval sampling-time distribution $p(t)$ is either known or approximated as uniform.
  When this approximation is violated, the inferred temporal ordering can remain preserved, but the estimated velocity magnitude may become systematically biased; we analyze this case in App.~\ref{appendix:non_uniform}.
  Third, CT-OT Flow assumes mass preservation at the probability-distribution level and does not explicitly model birth-death or abundance changes.
}

\subsection{Overview of CT-OT Flow}
CT-OT Flow estimates the continuous-time distribution $p_t(\boldsymbol x)$ from observed temporal snapshots and then learns an ODE/SDE from that estimated distribution.
Leveraging the estimated continuous-time distribution enables capturing more precise dynamics.
CT-OT Flow proceeds in three main steps:
\textbf{(1)~High-resolution time label estimation via optimal transport.}
Given two contiguous temporal snapshots $X_{[t_{j-1}, t_j]}$ and $X_{[t_j,t_{j+1}]}$, CT-OT Flow uses optimal transport to infer a high-resolution time label for each data point.
\textbf{(2)~Kernel-based distribution estimation.}
Based on the inferred high-resolution time labels, CT-OT Flow estimates a distribution $\tilde p_t(\boldsymbol x)$ for any time $t$.
Specifically, it applies a kernel smoothing approach in the time dimension.
This step enhances robustness to noise in the recorded observation times and inaccuracies in the inferred time labels.
\textbf{(3)~ODE/SDE learning from the estimated distributions.}
Finally, CT-OT Flow samples from the estimated distributions $\tilde p_t(\boldsymbol x)$ and $\tilde p_{t+\delta t}(\boldsymbol x)$, and trains a model to learn the dynamics.
By choosing a small $\delta t$, we can obtain data pairs at closely spaced time points, thereby fitting an ODE/SDE at a finer temporal scale than conventional approaches, resulting in more precise dynamics.
We describe these steps in detail below.
Pipeline and pseudocode for CT-OT Flow are shown in Fig.~\ref{fig:pipeline} and Algorithm~\ref{algorithm:proposed} in App.~\ref{appendix:algorithm}, respectively.

\begin{figure}
  \centering
  \includegraphics[width=\textwidth]{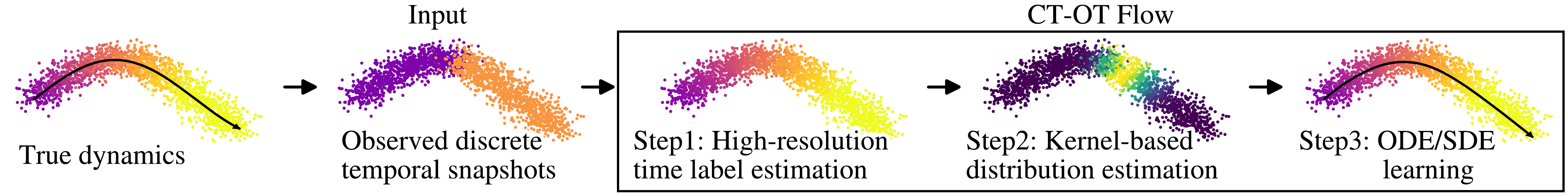}
  \caption{CT-OT Flow Pipeline: POT-based high-resolution time label estimation, kernel-based continuous-time distribution estimation, and ODE/SDE model training.}
  \label{fig:pipeline}
\end{figure}

\subsection{Step 1: High-resolution time label estimation}
\label{sec:step1}
CT-OT Flow first estimates high-resolution time labels for all data points observed in the interval~$[t_j, t_{j+1}]$.
Conventional methods assume that all data points in this interval belong to a single time~$t_j$.
However, in reality, each data point may have a distinct true observation time.
Assuming $p_t(\boldsymbol x)$ varies continuously in $t$, CT-OT Flow infers a high-resolution time label for each data point.

\paragraph{Continuity assumption}
We assume that the underlying distribution $p_t(\boldsymbol x)$ is continuous in $t$, i.e.,
\begin{align}
  \label{eq:continuity_assumption}
  \lim_{\Delta t \to 0} p_{t+\Delta t}(\boldsymbol x)
  = \lim_{\Delta t \to 0} p_{t-\Delta t}(\boldsymbol x)
  = p_t(\boldsymbol x).
\end{align}

As a consequence, the aggregated distribution over a short interval $p_{[t, t+\Delta t]}(\boldsymbol x)$ converges to $p_t(\boldsymbol x)$ as $\Delta t \to 0$, i.e.,
\begin{align}
  \label{eq:continuous}
  \lim_{\Delta t \to +0} p_{[t, t+\Delta t]}(\boldsymbol x)=\lim _{\Delta t \to +0} p_{[t-\Delta t, t]}(\boldsymbol x) = p_t(\boldsymbol x).
\end{align}
For two contiguous intervals $[t_{j-1}, t_j]$ and $[t_j, t_{j+1}]$ with $\Delta t_j, \Delta t_{j+1}>0$, the 2-Wasserstein distance between the distributions of their subintervals $[t_{j}-\Delta t_j, t_j]$ and $[t_j, t_{j}+\Delta t_{j+1}]$ converges to zero:
\begin{align}
  \label{eq:wass_continuous}
  \lim_{\Delta t_j, \Delta t_{j+1}  \to +0} \mathcal W\left(p_{[t_j-\Delta t_{j}, t_j]}(\boldsymbol x), p_{[t_j, t_j+\Delta t_{j+1}]}(\boldsymbol x)\right) = 0.
\end{align}

\paragraph{Identifying boundary subsets}
Given two snapshots $X_{[t_{j-1}, t_{j}]}$ and $X_{[t_{j}, t_{j+1}]}$, CT-OT Flow identifies data points near their boundary $t_j$ based on Eq.~\eqref{eq:wass_continuous}.
Specifically, it finds subsets $S^-\subseteq X_{[t_{j-1}, t_{j}]}$ and $S^+\subseteq X_{[t_{j}, t_{j+1}]}$ that capture the empirical distributions on the intervals $[t_j-\Delta t_j, t_j]$ and $[t_j, t_j+\Delta t_{j+1}]$, respectively.
These subsets are estimated by solving the following optimization:
\begin{align}
  \label{eq:MILP}
  & \min_{S^-\subseteq X_{[t_{j-1}, t_{j}]}, S^+\subseteq X_{[t_{j}, t_{j+1}]}} \mathcal W^2(\hat p_{S^-}, \hat p_{S^+})\ \ \ \ \text{s.t.\ } |S^-| = \left\lceil\frac{N_{[t_{j-1}, t_{j}]}}{K}\right\rceil, |S^+| = \left\lceil\frac{N_{[t_{j}, t_{j+1}]}}{K}\right\rceil,
\end{align}
where $\lceil \cdot \rceil$ is the ceiling function, $\mathcal W^2$ is the squared 2-Wasserstein distance, and $K>0$ is an integer.

Solving Eq.~\eqref{eq:MILP} extracts subsets from the two snapshots whose distributions are close in terms of the Wasserstein distance.
According to Eq.~\eqref{eq:wass_continuous}, increasing $K$ drives the distance between these two distributions toward zero, and thus the extracted subsets are expected to approximate $p_{t_j}(\boldsymbol x)$.
However, because $X_{[t_{j-1}, t_{j}]}$ and $X_{[t_{j}, t_{j+1}]}$ contain only a finite number of points, letting $K\to\infty$ trivially yields $S^-\to\emptyset$ and $S^+\to\emptyset$.
Hence, in practice, we fix a finite $K$ so as to estimate the data points in the contiguous intervals $[t_j-\Delta t_j, t_j]$ and $[t_j, t_j+\Delta t_{j+1}]$.
Our preliminary experiments indicate that tuning $K$ suitably allows for a robust extraction of these subsets; see \S\ref{sec:hyperparameters} for more details.
Figure~\ref{fig:S}~(a) shows an example of $S^-$ and $S^+$ identified by CT-OT Flow.

\begin{figure}
  \centering
  \includegraphics[width=\textwidth]{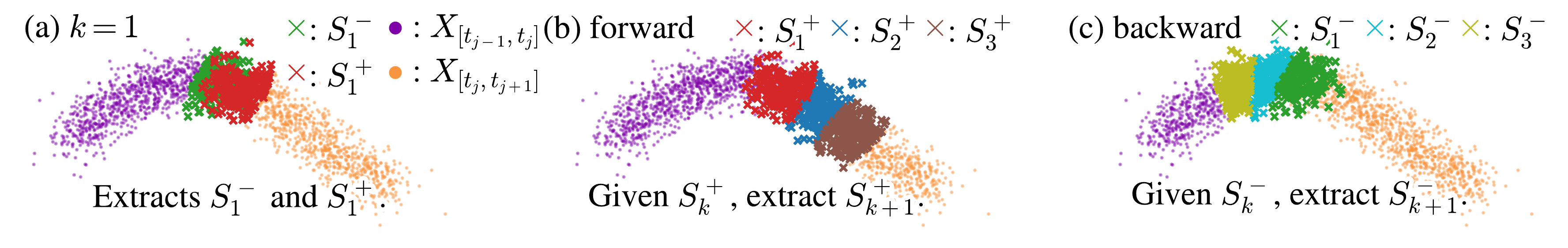}
  \caption{Step 1: high-resolution time label estimation via POT.
    (a) CT-OT Flow first extracts subsets $S^-_1$ and $S^+_1$ near the boundary of two time intervals $X_{[t_{j-1}, t_j]}$ (purple) and $X_{[t_{j}, t_{j+1}]}$ (orange).
    (b) (Forward) CT-OT Flow iteratively identifies the subset $S_{k+1}^+$.
  (c) (Backward) Similarly, CT-OT Flow identifies $S_{k+1}^-$. We set $K=5$ to better visualize the subset extraction process.    \label{fig:S}}
\end{figure}

\paragraph{Mixed-Integer Linear Programming relaxation}
The optimization problem in Eq.~\eqref{eq:MILP} is a mixed-integer linear programming (MILP).
As the dataset size grows, the computational cost escalates, making it increasingly difficult to solve directly.
Therefore, instead of solving Eq.~\eqref{eq:MILP} exactly, CT-OT Flow adopts a continuous relaxation of the problem, formulated as follows:
\begin{align}
  \label{eq:pw}
  \min_{P\in U_{(K,K)}} \sum_{m=1}^{N_{[t_{j-1}, t_{j}]}} \sum_{n=1}^{N_{[t_{j}, t_{j+1}]}} P_{m,n}\|\boldsymbol x^{(m)} - \boldsymbol x^{(n)}\|_2^2,
\end{align}
where $\boldsymbol x^{(m)}\in X_{[t_{j-1},t_{j}]}$ and $\boldsymbol x^{(n)}\in X_{[t_{j},t_{j+1}]}$.
\begin{proposition}
  \label{prop:relaxation}
  The problem in Eq.~\eqref{eq:pw} is the continuous relaxation of the MILP in Eq.~\eqref{eq:MILP}.
\end{proposition}
A proof is provided in App.~\ref{appendix:relaxation}.

Eq.~\eqref{eq:pw} is a POT problem, $\text{POT}_{(K,K)}(\hat p_{X_{[t_{j-1}, t_j]}}, \hat p_{X_{[t_{j}, t_{j+1}]}})$, so its optimal plan $P^*$ provides a \emph{soft} matching between the two snapshots and can be computed within a practical time.
To obtain hard subsets of the desired size, we sort all points by their row (or column) sums in $P^*$ and retain exactly $\lceil N_{[t_{j-1},t_j]}/K\rceil$ (or $\lceil N_{[t_j,t_{j+1}]}/K\rceil$) points on the first (or second) snapshot.
In this way we keep only those samples that participate \emph{most strongly} in the optimal transport, simultaneously satisfying the cardinality constraint and concentrating on points that lie closest to the interval boundary.
The empirical approximation quality of this POT-based relaxation relative to the exact MILP is examined in App.~\ref{app:milp-pot}.

\paragraph{Iterative extraction of subsets}
CT-OT Flow then iteratively applies a similar procedure to identify subsets in each contiguous time sub-interval.
Specifically, for $k=1,\dots, K-2$, we iteratively identify the subsets $S_{k+1}^-\subset X_{[t_{j-1}, t_{j}]}$ (backward) and $S_{k+1}^+\subset X_{[t_{j}, t_{j+1}]}$ (forward) as follows.
\begin{align}
  \label{eq:k2_1}
  \min_{ S_{k+1}^-\subset X_{[t_{j-1}, t_{j}]}\setminus (\cup_{k'=1}^{{k}}S_{k'}^-)} \mathcal W^2(\hat p_{S^-_{k}}, \hat p_{S^-_{k+1}}) \ \text{s.t.}\  |S_{k+1}^-| = \left\lceil\frac{N_{[t_{j-1}, t_{j}]}}{K}\right\rceil\ \  & \text{(backward)}, \\
  \label{eq:k2_2}
  \min_{ S_{k+1}^+\subset X_{[t_{j}, t_{j+1}]}\setminus (\cup_{k'=1}^{{k}}S_{k'}^+)} \mathcal W^2(\hat p_{S^+_{k}}, \hat p_{S^+_{k+1}}) \ \text{s.t.}\  |S_{k+1}^+| = \left\lceil\frac{N_{[t_{j}, t_{j+1}]}}{K}\right\rceil\ \  & \text{(forward)},
\end{align}
where $S_1^-=S^-, S_1^+=S^+$.
By applying the same relaxation as in Eq.~\eqref{eq:MILP}, each step can be solved via POT.
See App.~\ref{appendix:relaxation} for details.
For $k=K$, we take $S_K^- = X_{[t_{j-1}, t_{j}]} \setminus \cup_{k=1}^{K-1}S_k^-$ and $S_K^+ = X_{[t_{j}, t_{j+1}]} \setminus \cup_{k=1}^{K-1}S_k^+$.
Here, we expect $S_{k+1}^-$ and $S_{k}^-$ (and $S_{k+1}^+$ and $S_{k}^+$) to correspond to temporally contiguous subsets.
Figures~\ref{fig:S}~(b) and (c) illustrate examples of $S_k^-$ and $S_k^+$.
In CT-OT Flow, the subsets are thus extracted step-by-step from each boundary, moving backward and forward.

\paragraph{High-resolution time labels}
Through the above procedure, each of $X_{[t_{j-1}, t_{j}]}$ and $X_{[t_{j}, t_{j+1}]}$ is partitioned into $K$ subsets.
We then assign high-resolution time labels to each data point $\boldsymbol x^{(i)}$ according to the subsets.
If the sampling density $p(t)$ inside each coarse interval is known (e.g., from experimental design), high-resolution time labels can be obtained via the inverse CDF:
\begin{align}
  \label{eq:times_backward}
  \tilde{t}^{(i)}_{-}
  & = F_{-,j}^{-1}\!\Bigl(\tfrac{k}{K+1}\Bigr) \in [t_{j-1}, t_j], \ \ \text{if}\ \boldsymbol x^{(i)} \in S^-_k \subset X_{[t_{j-1},t_j]} \ \ \ \text{(backward)}, \\
  \label{eq:times_forward}
  \tilde{t}^{(i)}_{+}
  & =  F_{+,j}^{-1}\!\Bigl(\tfrac{k}{K+1}\Bigr)  \in [t_{j}, t_{j+1}],\ \ \text{if}\ \boldsymbol x^{(i)} \in S^+_k \subset X_{[t_{j},t_{j+1}]} \ \ \ \text{(forward)},
\end{align}
where $F_{-,j}(t) = \frac{1}{Z_{-,j}}\int_{t}^{t_j} p(t')dt'$ and $F_{+,j}(t)= \frac{1}{Z_{+,j}}\int_{t_j}^{t} p(t')dt'$, and $F_{-,j}^{-1}$ and $F_{+,j}^{-1}$ are their inverse functions.
$Z_{-,j}=\int_{t_{j-1}}^{t_j} p(t')dt'$ and $Z_{+,j}=\int_{t_j}^{t_{j+1}} p(t')dt'$ are normalization constants.
Hence, each data point is assigned a high-resolution time label $\tilde t^{(i)} = \tilde t^{(i)}_{-}$ (or $\tilde t^{(i)}_{+}$).
Applying this procedure to every pair of contiguous intervals in $\{X_{[t_j,t_{j+1}]}\}_{j=1}^{T}$ yields high-resolution time labels.
If a point is assigned both $\tilde t^{(i)}_{-}$ and $\tilde t^{(i)}_{+}$, we take their average as the final estimate.

When the true observation-time distribution $p(t)$ is unknown, we simply assume it to be \textbf{uniform within each coarse interval}.
In the following, we focus on the uniform scenario.
\changed{Non-uniform $p(t)$ analysis is provided in App.~\ref{appendix:non_uniform}, where performance degrades under mis-specified $p(t)$ but remains better than baselines.}
Although CT-OT Flow is derived for contiguous intervals, it can also be applied to non-contiguous pairs and often still yields reliable high-resolution time labels (see \S\ref{sec:experiments} and App.~\ref{appendix:non_contiguous}).

\subsection{Step 2: Estimating data distribution in continuous time}
\label{sec:step2}
Using the high-resolution time labels $\tilde t^{(i)}$, CT-OT Flow constructs the data distribution at any time~$t$.
Specifically, we define the \textit{kernel-based time-smoothed empirical distribution} at time $t$:
\begin{align}
  \label{eq:kernel_time_dist}
  \tilde p_t(\boldsymbol x) = \frac{1}{Z_\text{kernel}(t)}\sum_{i=1}^{|X|} k(t, \tilde t^{(i)})\delta_{\boldsymbol x^{(i)}}(\boldsymbol x),
\end{align}
where $X = \cup_{j=1}^T X_{[t_{j}, t_{j+1}]}$ denotes the set of all data points, $Z_\text{kernel}(t)=\sum_{i=1}^{|X|} k(t, \tilde t^{(i)})$, and $k(t, \tilde t^{(i)})$ is a kernel function.
For example, we can use the Gaussian kernel
$k(t, \tilde t^{(i)}) = \exp\left(-|t-\tilde t^{(i)}|^2 / \gamma\right)$,
where $\gamma>0$ controls the degree of smoothing.

This approach yields a robust estimate of the distribution under temporal noise.
In particular, data points near interval boundaries (e.g., around $t_j$) may be assigned to different time intervals if observation times are noisy.
By smoothing along the \emph{time axis}, the estimator becomes tolerant to both high-resolution time label errors and measurement noise.
Note that this smoothing is applied in the temporal dimension, unlike smoothing in the feature space, such as in kernel density estimation~\cite{bw1986density}.
In \S\ref{sec:hyperparameters}, we demonstrate the effect of $\gamma$ on $p_t(\boldsymbol{x})$.

\subsection{Step 3: Dynamics estimation}
CT-OT Flow then estimates the dynamics using $\tilde p_t(\boldsymbol x)$ from Step 2.
Specifically, we draw paired samples from $\tilde p_t(\boldsymbol x)$ and $\tilde p_{t+\delta t}(\boldsymbol x)$ for $\delta t>0$, and fit an ODE/SDE model to them.
A key advantage of CT-OT Flow is that $\delta t$ can be chosen arbitrarily small, allowing more realistic dynamics estimation, whereas conventional methods rely on a fixed set of discrete time points.
Figure~\ref{fig:delta} shows an example with $t=2.0$ and $\delta t=0.3$, where the time intervals of the inputs are $[1,2]$ and $[2,3]$.
\begin{figure}
  \centering
  \includegraphics[width=0.96\textwidth]{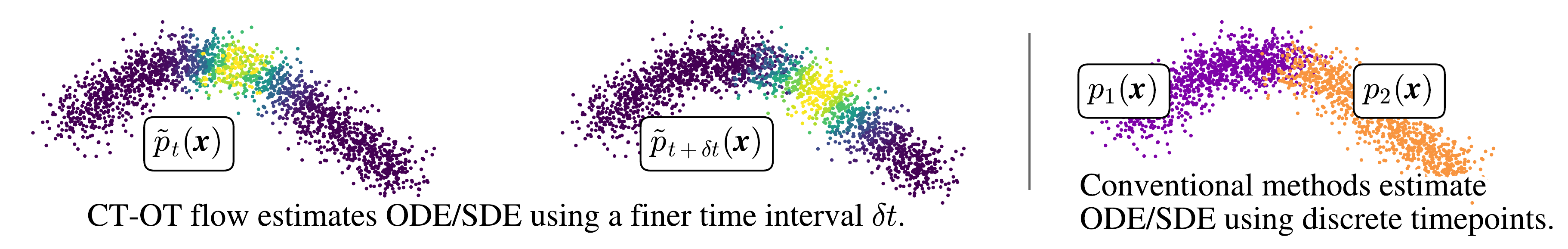}
  \caption{
    Steps 2\&3. A kernel function produces a smoothed empirical distribution $\tilde p_t(\boldsymbol x)$ at any continuous time~$t$.
  We then sample from $\tilde p_t$ and $\tilde p_{t+\delta t}$ to train a dynamics model. Unlike conventional methods that operate only on pre-specified discrete times, CT-OT Flow allows arbitrarily fine choices of $\delta t$, capturing more realistic trajectories.    \label{fig:delta}}
\end{figure}

CT-OT Flow can be combined with existing ODE/SDE models.
For instance, Rectified Flow~\cite{liu2023flow} models the velocity field with a neural network $\boldsymbol v_\theta(\boldsymbol x_t,t)$.
Let $\boldsymbol x^{(0)}$ and $\boldsymbol x^{(1)}$ be samples from $\tilde p_t(\boldsymbol x)$ and $\tilde p_{t+\delta t}(\boldsymbol x)$, respectively.
Their velocity vector is estimated as $\frac{\boldsymbol x^{(1)} - \boldsymbol x^{(0)}}{\delta t}$.
We adapt the Rectified Flow loss to our framework as follows:
\begin{align}
  \label{eq:loss_rf}
  \mathcal L_\text{RF} = \mathbb E_{(\boldsymbol x^{(0)}, \boldsymbol x^{(1)}, t, t')\sim Q}\left[\left\|\frac{\boldsymbol x^{(1)} - \boldsymbol x^{(0)}}{\delta t} - \boldsymbol v_\theta(t'\boldsymbol x^{(1)} + (1-t')\boldsymbol x^{(0)}, t+t'\delta t) \right\|_2^2 \right],
\end{align}
where $Q=\tilde p_t(\boldsymbol x)\otimes \tilde p_{t+\delta t}(\boldsymbol x)\otimes \mathcal U(t|t_1,t_{T+1}-\delta t)\otimes \mathcal U(t'|0,1)$, $\mathcal U$ is a uniform distribution, and $\otimes$ denotes the product of probability distributions.
Thanks to CT-OT Flow, the time variable $t$ is sampled from a \emph{continuous} range rather than a discrete set.
Thus, CT-OT Flow enables learning of dynamics from a continuous temporal perspective.
Similarly, our framework can be integrated with methods such as conditional flow matching~\cite{tong2024improving} or [SF]$^2$M~\cite{tong2024simulation} for ODE/SDE estimation.

\subsection{Computational cost}
The overall runtime of CT-OT Flow depends on three parameters: the number of intervals $T$, the number of samples per interval $N$, and the subdivision factor $K$.
We denote by $C(M,N)$ the cost of solving a single POT instance of size $M\times N$.
While the network simplex algorithm has worst-case super-cubic complexity~\cite{ahuja1993network}, the Sinkhorn iteration runs in $O(MN)$ time per iteration~\cite{cuturi2013sinkhorn}.

\textbf{(Step~1.)}
For $k=1$, we solve one POT problem for each of the $T$ snapshot pairs, costing $TC(N,N)$.
For each subsequent level $k=2,\dots,K-1$ we solve two POT problems (backward and forward) between a block of $\frac{N}{K}$ points and the remaining $\frac{N(K-k+1)}{K}$ points.
This incurs an additional cost of $2T\sum_{k=2}^{K-1} C(\frac{N}{K},\frac{N(K-k+1)}{K})$.
One practical acceleration splits each interval’s points into $\tfrac{N}{m}$ random mini-batches of size $m$, runs CT-OT Flow independently on $\tfrac{N}{m}$ batch pairs, and then concatenates the resulting high-resolution time labels.
With fixed $m$, the cost of Step 1 scales linearly in $N$.
Here, $m$ trades off accuracy (large $m$ preserves global structure) and computation (small $m$ is faster). \\
\textbf{(Step~2.)} Kernel smoothing over the $T$ intervals requires $O(TN)$ time. \\
\textbf{(Step~3.)} The cost matches that of the downstream ODE/SDE learner (e.g., Rectified Flow~\cite{liu2023flow}).

On a single CPU thread using network simplex, Step 1 takes approximately 70 seconds for $N=10^4, T=1, K=100$.
Under a mini-batch setup with $m=10^3$, the runtimes are approximately 3 seconds, 30 seconds, and 400 seconds for $N=10^4$, $10^5$, and $10^6$, respectively.
Note that these computation times are shorter than those of the neural network training phase in existing methods or Step~3, which can range from minutes to hours.
See App.~\ref{appendix:computational_time} for details on computation times when varying $N$ and $K$, as well as the mini-batch approach.

\section{Experiments}
\label{sec:experiments}
We evaluate the performance of CT-OT Flow through numerical experiments.
First, we describe the evaluation metrics in \S\ref{sec:metrics}.
Next, in \S\ref{sec:artificial}, we assess how well each method recovers the true dynamics using synthetic data.
\changed{Performance sensitivity to hyperparameters is examined in \S\ref{sec:hyperparameters}.  }
Finally, we present an application to real scRNA-seq data and meteorological data in \S\ref{sec:scRNA-seq} and \S\ref{sec:typhoon}, respectively.
For these experiments, we compare CT-OT Flow against recent methods, including ODE-based approaches such as I-CFM, OT-CFM~\cite{tong2024improving}, TrajectoryNet-Base~\cite{tong2020trajectorynet},
and Metric Flow Matching (MFM)~\cite{kapusniak2024metric}, as well as SDE-based approaches such as $\text{[SF]}^2$M-I, $\text{[SF]}^2$M-Exact~\cite{tong2024simulation} and Entropic Neural Optimal Transport (ENOT)~\cite{gushchin2024entropic}.
We additionally compared CT-OT Flow with Slingshot~\cite{street2018slingshot}, a widely used pseudotime inference method.
Since Slingshot outputs only a one-dimensional pseudotime ordering of the samples, but not trajectories, we used Slingshot's pseudotime labels as inputs to Step 2 and Step 3 of CT-OT Flow (ODE).
Because Slingshot requires cluster assignments as input, we applied $k$-means clustering ($k$=6) to each snapshot.
The cluster corresponding to the truly earliest time was designated as the root cluster.
We combine CT-OT Flow with OT-CFM~\cite{tong2024improving} to estimate ODEs and with $\text{[SF]}^2$M-Exact~\cite{tong2024simulation} to estimate SDEs.
We use a Gaussian kernel with parameter $\gamma=0.005$ for Step~2, and the subdivision factor is set to $K=100$ for all datasets.
Additional experimental parameters are listed in App.~\ref{appendix:experiments_parameters}.

\subsection{Evaluation metrics}
\label{sec:metrics}
When the true dynamics are known (i.e., when the true data distribution $p_t^*(\boldsymbol x)$ at each time $t$ and the corresponding ODE/SDE describing the movement of data points are available) we evaluate each method by comparing the trajectories derived from the true dynamics to the estimated trajectories.
We assume that the true system produces trajectories $X^{*}_j = [{\boldsymbol x}_{j,0}^*,\dots,{\boldsymbol x}^*_{j,T^*}]$ for $j=1,\dots, N^*$, where $T^*$ is the number of time steps and ${\boldsymbol x}_{j,0}^*$ denotes the mean of the true initial distribution $p^*_0(\boldsymbol x)$.
Each model produces an estimated trajectory $\hat X^{(i)} = [\hat {\boldsymbol x}_0^{(i)},\dots,\hat {\boldsymbol x}_{\hat T_i}^{(i)}]$, where $\hat T_i$ is the number of simulated time steps and $i=1,\dots, \hat N$; the initial point $\hat {\boldsymbol x}_0^{(i)}$ is drawn from $p^*_0(\boldsymbol x)$ (not from $p_{[0,1]}(\boldsymbol x)$).

We employ two metrics to evaluate the accuracy of the estimated dynamics.
The first metric quantifies the discrepancy between the true and estimated trajectories using Dynamic Time Warping~\cite{sakoe1978dynamic}:
\begin{align}
  \mathcal L_\text{DTW} = \frac{1}{\hat N}\sum_{i=1}^{\hat N} \min_{j} \text{DTW}(X_j^*, \hat X^{(i)} )
\end{align}
The second metric computes the average 1-Wasserstein distance~\cite{peyre2019computational} between the true distribution $p^*_{t}(\boldsymbol x)$ and the distribution estimated from the simulated trajectories:
\begin{align}
  \mathcal L_\text{Wass} = \frac{1}{T^*}\sum_{t=1}^{T^*}\mathcal W_1(\hat p_{\hat{\boldsymbol x}_{t'}}, p^*_{t}(\boldsymbol x)), \ \text{where}\ t' = \left\lfloor t\,\frac{\hat T_i}{T^*}\right\rfloor, \hat {\boldsymbol x}_{t'} = \{\hat{\boldsymbol x}_{t'}^{(i)}\}_{i=1}^{\hat N}.
\end{align}
Here, $\lfloor \cdot \rfloor$ represents the floor function and $\mathcal W_1$ is the 1-Wasserstein distance~\cite{peyre2019computational}.

\subsection{Synthetic data experiments}
\label{sec:artificial}
We evaluate dynamics estimation methods using three synthetic datasets; Spiral, Y-shaped, and Arch~\cite{kapusniak2024metric}.
For Spiral and Y-shaped, data are collected over two contiguous time intervals, while the Arch dataset comprises data from two non-contiguous intervals. See App.~\ref{appendix:data_generation} for details.

Figure~\ref{fig:trajectories_spiral} shows estimated trajectories for the Spiral dataset, while Fig.~\ref{fig:appendix_trajectory} in App.~\ref{appendix:additional_results} shows those for the Y-shaped and Arch datasets.
In these figures, the colors of the points indicate discrete time intervals, and for CT-OT Flow, the color of each point reflects its high-resolution time label.
Table~\ref{tab:artificial} summarizes the mean and standard deviation of the metrics over \changed{20} runs with different neural network initializations.
Boldface marks the best mean within each ODE/SDE group.
\begin{figure}
  \centering
  \includegraphics[width=0.9\textwidth]{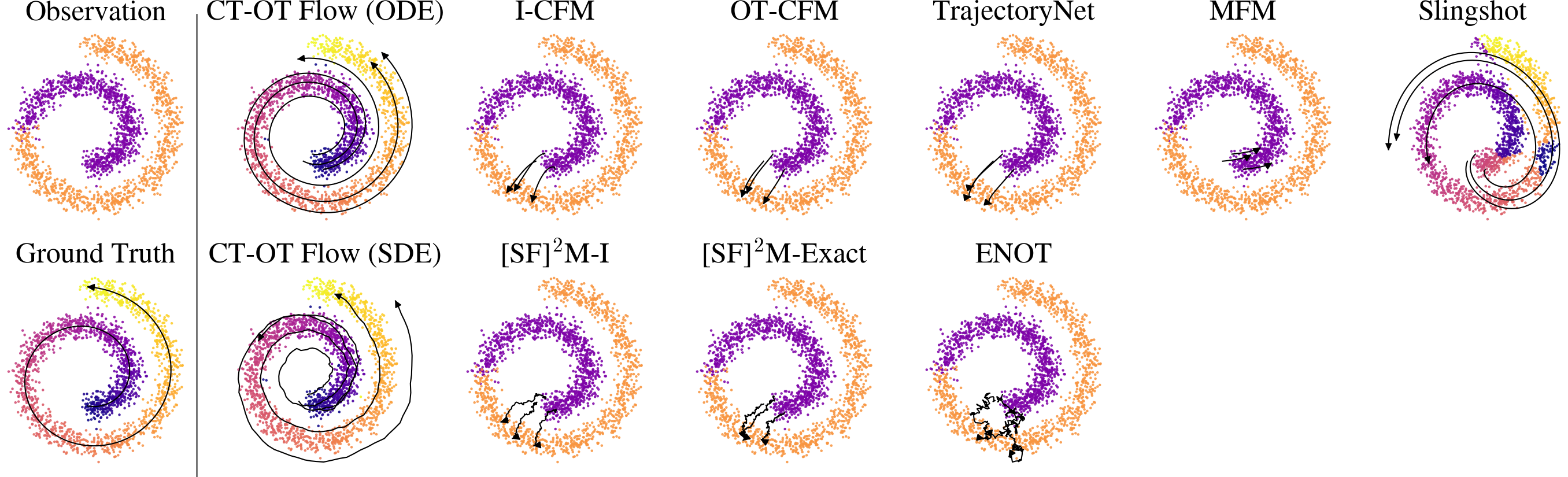}
  \caption{Estimated trajectories on the Spiral dataset. The black lines indicate the true or estimated trajectories, while the color of each point indicates its time label. For CT-OT Flow and Slingshot, the colors represent the estimated high-resolution time labels and pseudotime labels, respectively. The initial points are sampled from $p_0(\boldsymbol x)$ but not from $p_{[0,1]}(\boldsymbol x)$. The proposed method effectively reproduces the true dynamics, whereas conventional methods often yield larger deviations from the ground truth.    \label{fig:trajectories_spiral}}
\end{figure}

These results indicate that CT-OT Flow recovers trajectories that are closer to the true dynamics than those yielded by conventional methods for both the Spiral and Y-shaped datasets, owing to its consideration of discretized times and time uncertainty.
Moreover, even for the Arch dataset, where the observed intervals are non-contiguous, our method achieves accuracy comparable to that of MFM~\cite{kapusniak2024metric} in predicting the trajectories.
See App.~\ref{appendix:non_contiguous} for additional evaluation on data with non-contiguous intervals.
Replacing Step 1 with Slingshot yielded a slight improvement on the Arch dataset but a pronounced degradation on the Spiral dataset.
This contrast stems from a methodological limitation: Slingshot does not model discretized observation times.
Lacking this temporal signal, it struggles to resolve complex manifolds and consequently misaligns transitions on Spiral.
By contrast, CT-OT Flow enforces distributional continuity across contiguous snapshots to estimate instantaneous distributions, enabling it to faithfully capture intricate manifolds.
For clarity, Slingshot alone yields only an ordering--not a predictive dynamical model--and thus cannot predict trajectories for individual data points; those predictions are provided by CT-OT Flow's Steps 2-3.
\changed{Additional results on temporal label estimation accuracy are provided in App.~\ref{app:temporal-label-synthetic}.
We also consider a more challenging setting in which true initial states are unavailable at test time in App.~\ref{app:conditioned-no-initial}.}

\begin{table}[t]
  \centering
  \caption{
    Estimation errors (mean$\pm$std) on the synthetic datasets.
  Bold: best within ODE/SDE groups.}
  \label{tab:artificial}
  \hspace{-5mm}
  \begin{tabular}{ccccccc}
    \toprule
    Dataset$\rightarrow$ & \multicolumn{2}{c}{Spiral} & \multicolumn{2}{c}{Y-shaped} & \multicolumn{2}{c}{Arch} \\
    \cmidrule(lr){2-3}
    \cmidrule(lr){4-5}
    \cmidrule(lr){6-7}
    Method$\downarrow$ Metric$\rightarrow$ & $\mathcal L_\text{DTW}$ & $\mathcal L_\text{Wass}$ & $\mathcal L_\text{DTW}$ & $\mathcal L_\text{Wass}$ & $\mathcal L_\text{DTW}$ & $\mathcal L_\text{Wass}$ \\
    \toprule
    CT-OT Flow (ODE) & \textbf{9.63}\std{1.25} & \textbf{0.31}\std{0.05} & \textbf{10.80}\std{0.93} & \textbf{0.43}\std{0.11} & 6.81\std{1.28} & 0.23\std{0.05} \\
    I-CFM~\cite{tong2024improving} & 50.60\std{0.20} & 1.12\std{0.01} & 21.54\std{0.35} & 0.61\std{0.01} & 29.36\std{0.33} & 0.58\std{0.01} \\
    OT-CFM~\cite{tong2024improving} & 51.41\std{0.02} & 1.17\std{0.00} & 29.42\std{0.20} & 0.90\std{0.00} & 31.65\std{0.12} & 0.62\std{0.00} \\
    TrajectoryNet~\cite{tong2020trajectorynet} & 51.60\std{0.15} & 1.19\std{0.03} & 22.41\std{5.00} & 0.89\std{0.10} & 33.54\std{2.93} & 0.67\std{0.08} \\
    MFM~\cite{kapusniak2024metric} & 51.54\std{0.20} & 1.18\std{0.03} & 30.12\std{0.65} & 0.84\std{0.04} & 13.87\std{1.60} & 0.31\std{0.02} \\
    Slingshot~\cite{street2018slingshot} & 35.15\std{1.85} & 0.93\std{0.07} & 17.85\std{4.28} & 1.13\std{0.43} & \textbf{3.85}\std{0.90} & \textbf{0.17}\std{0.01} \\
    \midrule
    CT-OT Flow (SDE) & \textbf{11.52}\std{1.12} & \textbf{0.36}\std{0.04} & \textbf{11.51}\std{1.04} & \textbf{0.47}\std{0.11} & \textbf{7.94}\std{1.07} & \textbf{0.25}\std{0.04} \\
    $\text{[SF]}^2\text{M}$-I~\cite{tong2024simulation} & 50.04\std{0.22} & 1.12\std{0.01} & 21.20\std{0.32} & 0.60\std{0.01} & 28.62\std{0.32} & 0.59\std{0.01} \\
    $\text{[SF]}^2\text{M}$-Exact~\cite{tong2024simulation} & 51.32\std{0.03} & 1.17\std{0.00} & 28.98\std{0.21} & 0.89\std{0.01} & 30.92\std{0.17} & 0.62\std{0.00} \\
    ENOT~\cite{gushchin2024entropic} & 49.60\std{0.08} & 1.15\std{0.00} & 26.04\std{0.29} & 0.82\std{0.01} & 24.49\std{0.20} & 0.57\std{0.00} \\
    \bottomrule
  \end{tabular}
\end{table}

\subsection{Ablation study: omitting steps in CT-OT Flow}
\label{appendix:ablation_step2}

We quantitatively evaluate how much each stage of CT-OT Flow contributes to the final accuracy \changed{over 20 runs.}
Recall that Step 1 assigns high-resolution time labels, while Step 2 converts them into continuous-time distributions $\tilde p_t(\boldsymbol x)$.
In Step 3, we sample $t\sim\mathcal U(t_1,\,t_{T+1}-\delta t)$ and draw paired samples from $\tilde p_t$ and $\tilde p_{t+\delta t}$ to train the ODE/SDE model.

We first completely omit Step 2.
Instead of constructing continuous distributions, we directly use the high-resolution labels from Step 1.
Specifically, we sample label values $\tilde t_j$ from the set of unique high-resolution time labels, then form empirical distributions $\hat p_{\tilde t_j}(\boldsymbol x)$ and $\hat p_{\tilde t_{j+1}}(\boldsymbol x)$ using only points whose labels equal $\tilde t_j$ and $\tilde t_{j+1}$, respectively.
Table~\ref{tab:ablation_step2} compares the results from \S\ref{sec:artificial} against this ablated variant.
We observe that omitting Step~2 degrades estimation accuracy, underscoring the importance of modeling continuous-time distributions in Step~2.

Next, we omit both Step 1 and Step 2, reverting to standard OT-CFM~\cite{tong2024improving} or [SF]$^2$M~\cite{tong2024simulation}, which operate directly on the coarse snapshots.
Performance drops even further (rows “w/o Step 1\&2”), indicating that high-resolution time label estimation in Step 1 is beneficial, and that Step 2 provides an additional, complementary benefit.

\begin{table}[t]
  \centering
  \caption{Estimation errors (mean$\pm$std) for ablation study on the synthetic datasets.
  Bold: best within ODE/SDE groups. w/o denotes ``without''.}
  \label{tab:ablation_step2}
  \begin{tabular}{ccccccc}
    \toprule
    Dataset$\rightarrow$                   & \multicolumn{2}{c}{Spiral} & \multicolumn{2}{c}{Y-shaped} & \multicolumn{2}{c}{Arch} \\
    \cmidrule(lr){2-3}
    \cmidrule(lr){4-5}
    \cmidrule(lr){6-7}
    Method$\downarrow$ Metric$\rightarrow$ & $\mathcal L_\text{DTW}$ & $\mathcal L_\text{Wass}$ & $\mathcal L_\text{DTW}$ & $\mathcal L_\text{Wass}$ & $\mathcal L_\text{DTW}$ & $\mathcal L_\text{Wass}$ \\
    \toprule
    CT-OT Flow (ODE)       & \textbf{9.63}\std{1.25} & \textbf{0.31}\std{0.05} & \textbf{10.80}\std{0.93} & \textbf{0.43}\std{0.11} & \textbf{6.81}\std{1.28} & \textbf{0.23}\std{0.05} \\
    w/o Step 2 (ODE)       & 12.48\std{2.92}          & 0.35\std{0.07}          & 16.06\std{2.95}          & 0.81\std{0.25}          & 9.24\std{2.22} & 0.48\std{0.11} \\
    w/o Step 1\&2 (ODE)    & 51.41\std{0.02} & 1.17\std{0.00} & 29.42\std{0.20} & 0.90\std{0.00} & 31.65\std{0.12} & 0.62\std{0.00} \\
    \midrule
    CT-OT Flow (SDE)       & \textbf{11.66}\std{1.06} & \textbf{0.36}\std{0.04}          & \textbf{11.47}\std{0.78} & \textbf{0.45}\std{0.09} & 11.43\std{2.63}         & \textbf{0.34}\std{0.08} \\
    w/o Step 2 (SDE)       & 13.10\std{2.35}          & \textbf{0.36}\std{0.06}          & 16.73\std{2.80}          & 0.84\std{0.25}          & \textbf{9.62}\std{1.59}          & 0.49\std{0.11} \\
    w/o Step 1\&2 (SDE)    & 51.32\std{0.03} & 1.17\std{0.00} & 28.98\std{0.21} & 0.89\std{0.01} & 30.92\std{0.17} & 0.62\std{0.00} \\
    \bottomrule
  \end{tabular}
\end{table}

\subsection{\changed{Performance sensitivity to hyperparameters}}
\label{sec:hyperparameters}
We next examine the sensitivity of CT-OT Flow to the two main hyperparameters: the subdivision factor $K$ in Step~1 and the kernel smoothing parameter $\gamma$ in Step~2.
Here we focus on the quantitative trends, while Appendix~\ref{appendix:hyperparameter_details} provides additional visualizations and interpretation.

\paragraph{Effect of the subdivision factor \texorpdfstring{$K$}{K}}
We examine how the inferred high-resolution time labels change as the subdivision factor $K$ in CT-OT Flow varies.
Figure~\ref{fig:spearman} reports the Spearman rank correlation between the inferred and true times, averaged over 10 runs with different random seeds.
When $K$ is too small, the temporal resolution is coarse, resulting in weaker agreement with the true times.
As $K$ increases, the resolution improves and the correlation correspondingly increases.
However, when $K$ becomes excessively large, the estimates become more sensitive to noise, and the accuracy gradually degrades.
Thus, rather than taking $K \to \infty$, choosing $K$ within an appropriate range leads to more accurate time estimation.
Figure~\ref{fig:ablation_diff_k} shows how the downstream prediction error changes with $K$ when CT-OT Flow is combined with the ODE/SDE estimation settings, while keeping the other parameters the same as in \S\ref{sec:artificial}.
Consistent with the time-label analysis, CT-OT Flow achieves accurate dynamics prediction only when $K$ is chosen appropriately.
The performance degrades when $K$ is too small, particularly around $K \approx 1$, highlighting the importance of Step 1 in CT-OT Flow.
We also illustrate the inferred high-resolution time labels on the three synthetic datasets in Figure~\ref{fig:diff_K} (App.\ref{appendix:hyperparameter_details}) for different values of $K$.

\begin{figure}
  \centering
  \begin{subfigure}[b]{0.31\textwidth}
    \centering
    \includegraphics[width=0.9\textwidth]{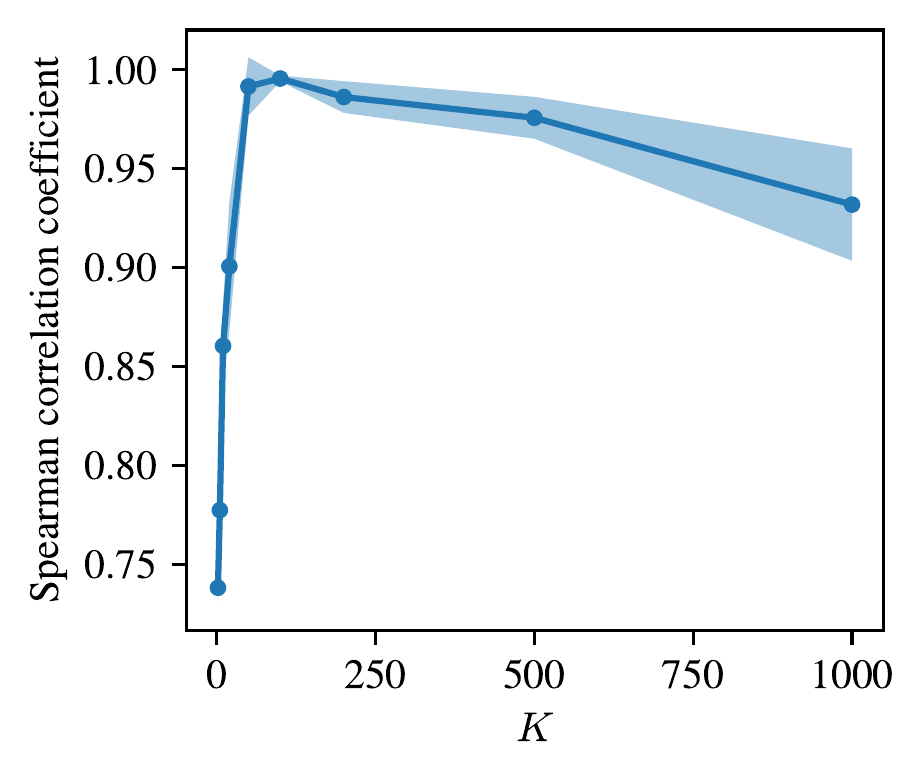}
    \caption{Spiral}
  \end{subfigure}
  \begin{subfigure}[b]{0.31\textwidth}
    \centering
    \includegraphics[width=\textwidth]{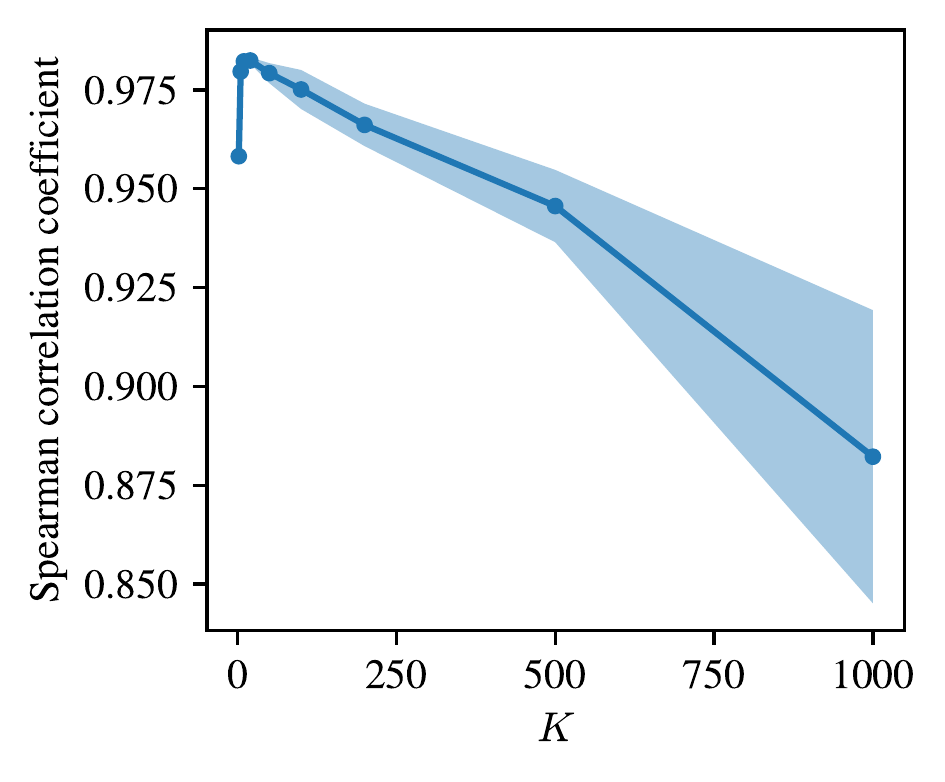}
    \caption{Y-shaped}
  \end{subfigure}
  \begin{subfigure}[b]{0.31\textwidth}
    \centering
    \includegraphics[width=\textwidth]{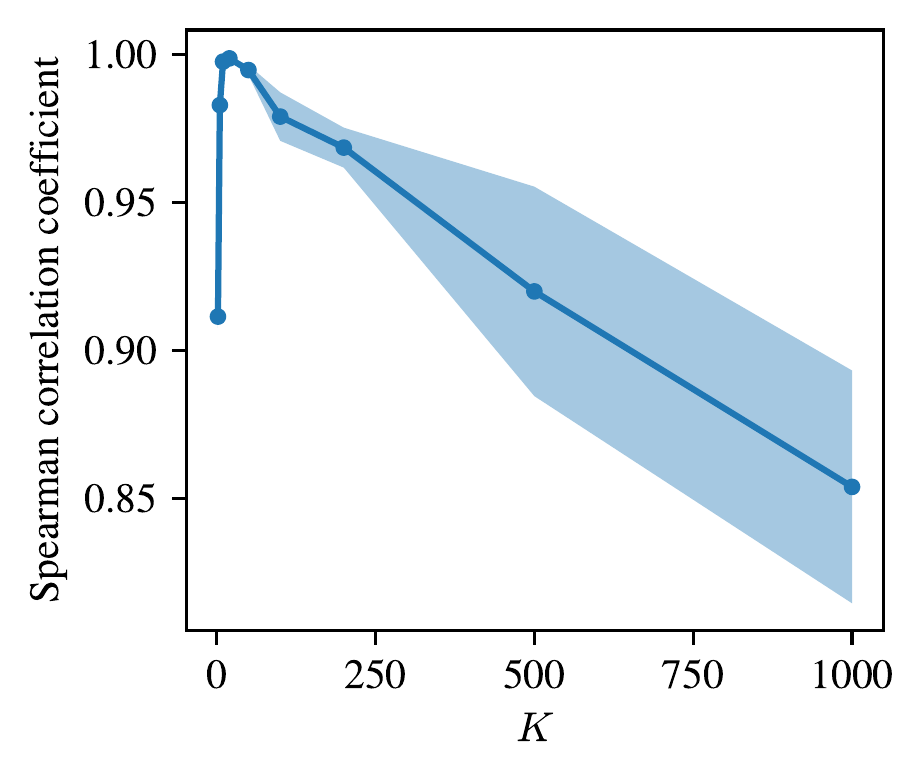}
    \caption{Arch}
  \end{subfigure}
  \caption{Spearman correlation between estimated high-resolution time labels and true times with varying $K$. The lines and areas represent the average and standard deviations, respectively.    \label{fig:spearman}
  }
\end{figure}

\begin{figure}
  \centering
  \includegraphics[width=0.65\textwidth]{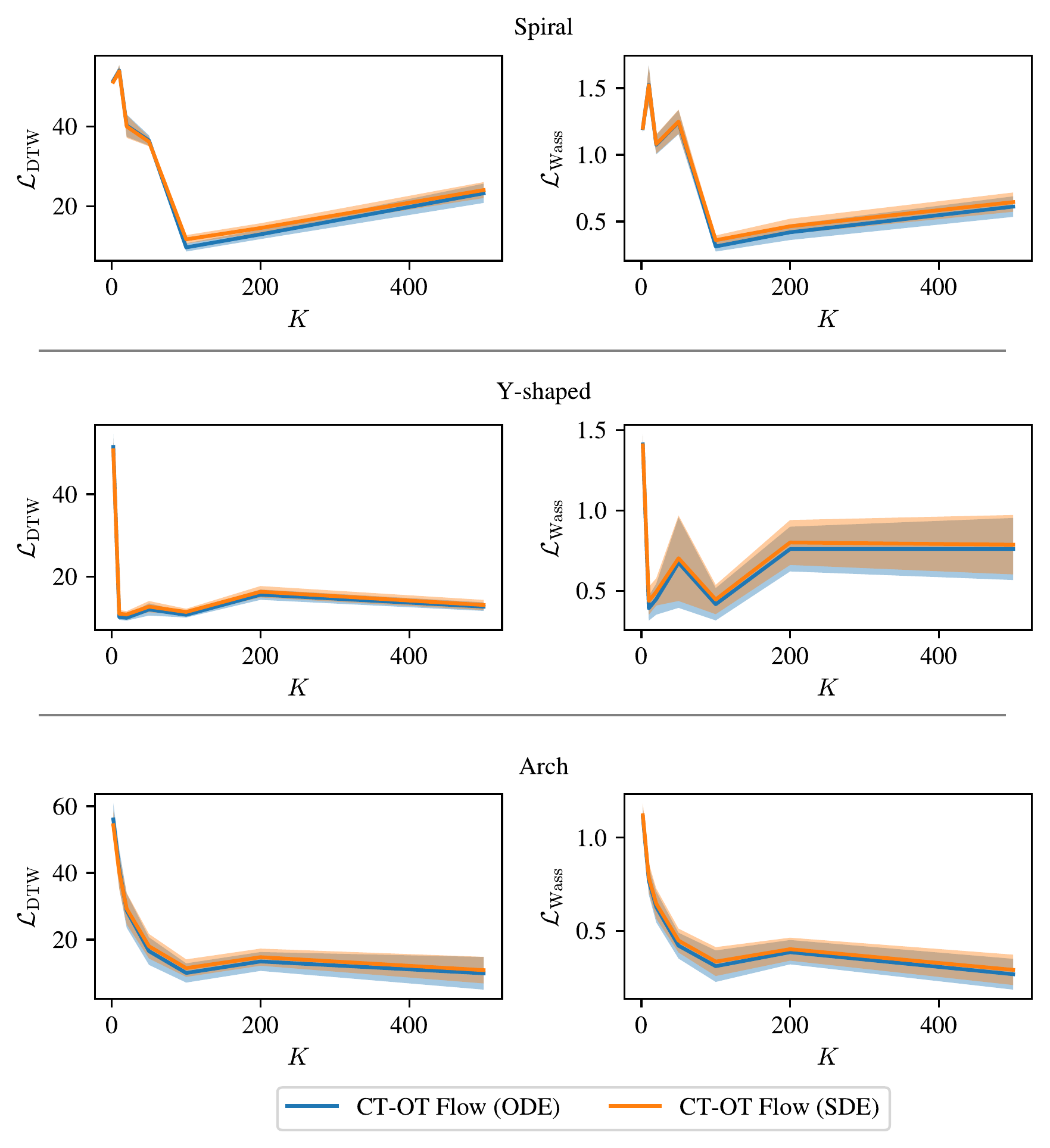}
  \caption{Prediction errors with varying $K$.    \label{fig:ablation_diff_k}
  }
\end{figure}

\paragraph{Effect of the kernel smoothing parameter \texorpdfstring{$\gamma$}{gamma}}
We investigate how varying the kernel smoothing parameter $\gamma$ affects the estimation results of CT-OT Flow.
When $\gamma$ is small, the estimated distribution $p_t(\boldsymbol x)$ has nonzero density only around data points whose high-resolution time label is close to $t$.
Conversely, when a large $\gamma$ is used, the estimated density spreads more broadly, incorporating density contributions from data points with high-resolution time labels farther from $t$.
Figure~\ref{fig:ablation_diff_gamma} illustrates how the prediction error changes in practice as $\gamma$ varies.
We run CT-OT Flow (ODE/SDE) with all other parameters as in \S\ref{sec:artificial}, varying only $\gamma$.
As shown in Figure~\ref{fig:ablation_diff_gamma}, large $\gamma$ values increase prediction errors.
This is due to excessive smoothing, which diminishes the contribution of the estimated high-resolution time labels.
On the other hand, using a very small $\gamma$ results in increased errors for the Y-shaped and Arch datasets.
As illustrated in Fig.~\ref{fig:spearman}, the accuracy of high-resolution time label estimation is lower for the Y-shaped and Arch datasets compared to Spiral.
When $\gamma$ is too small, the dynamics estimation becomes more sensitive to errors in the inferred high-resolution time labels, which likely leads to the observed degradation in overall performance.
These results demonstrate that Step 2’s smoothing mitigates time-label errors and enhances overall estimation accuracy.

\begin{figure}
  \centering
  \includegraphics[width=0.65\textwidth]{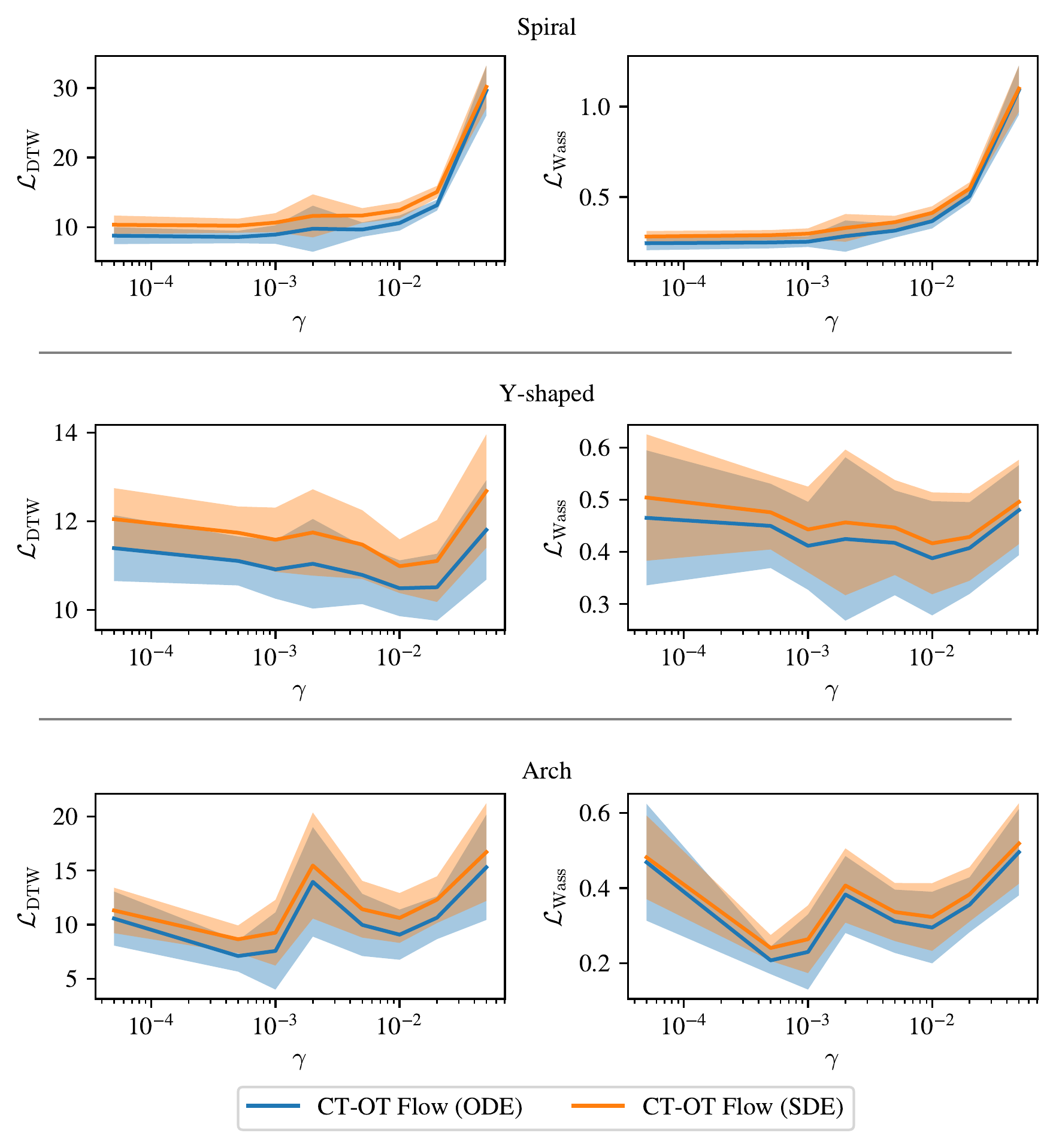}
  \caption{Prediction errors with varying $\gamma$.    \label{fig:ablation_diff_gamma}
  }
\end{figure}

\subsection{Application to scRNA-seq data}
\label{sec:scRNA-seq}
In this section, we evaluate CT-OT Flow on real scRNA-seq data.
We generate time-discretized datasets from the Bifurcation dataset~\cite{Bargaje2017cell,sha2024reconstructing} and embryoid body (EB) dataset~\cite{moon2018embryoid} by grouping the first three (for Bifurcation) or two (for EB) time points as one interval and the next three as a second interval, yielding two-interval datasets for training.
Bifurcation data are four-dimensional, while EB data are two-dimensional (see App.~\ref{appendix:data_generation} for details).

After training on the time-discretized dataset, we evaluate each method by simulating trajectories on the held-out test data, starting from samples at the earliest original time point.
Since true dynamics are unknown, we treat the empirical distribution at each original time point as ground truth and report $\mathcal L_\text{Wass}$.
Table~\ref{tab:bifurcation} summarizes the evaluation scores over 10 runs.
CT-OT Flow achieves highly accurate reconstruction of time-evolving data distributions from the temporal snapshots.
Figure~\ref{fig:scrna_trajectories} further illustrates the predicted trajectories.
In the Bifurcation and EB datasets, which are constructed from scRNA-seq data, the model reconstructs plausible developmental trajectories and temporal progression from discretized temporal snapshots.
These results demonstrate that CT-OT Flow can generalize to complex dynamics and real-world datasets, highlighting the robustness of the proposed method under various conditions.
\changed{We also provide qualitative results on the original unaggregated EB dataset in App.~\ref{appendix:eb_unaggregated}.}

\begin{table}
  \centering
  \caption{Estimation errors ($\mathcal L_{\text{Wass}}$) on the scRNA-seq datasets.
  Bold: best within ODE/SDE groups.}
  \label{tab:bifurcation}
  \begin{tabular}{ccc}
    \toprule
    Method$\downarrow$ Dataset$\rightarrow$                 & Bifurcation             & EB             \\
    \toprule
    CT-OT Flow (ODE)                                        & \textbf{0.64}\std{0.04} & \textbf{0.94}\std{0.11} \\
    I-CFM~\cite{tong2024improving}                          & 0.71\std{0.00}          & 1.00\std{0.01} \\
    OT-CFM~\cite{tong2024improving}                         & 0.70\std{0.00}          & 1.06\std{0.01} \\
    TrajectoryNet~\cite{tong2020trajectorynet}         &             0.68\std{0.04}          & 1.11\std{0.04} \\
    MFM~\cite{kapusniak2024metric}                          & 0.69\std{0.01}          & 1.03\std{0.05} \\
    Slingshot~\cite{street2018slingshot} & 0.77\std{0.08} & 2.74\std{0.20} \\
    \midrule
    CT-OT Flow (SDE)                                        & \textbf{0.63}\std{0.04} & \textbf{0.92}\std{0.12} \\
    $\text{[SF]}^2\text{M}$-I~\cite{tong2024simulation}     & 0.71\std{0.00}          & 0.99\std{0.01} \\
    $\text{[SF]}^2\text{M}$-Exact~\cite{tong2024simulation} & 0.71\std{0.01}          & 1.06\std{0.01} \\
    ENOT~\cite{gushchin2024entropic}                        & 0.70\std{0.01}          & 1.02\std{0.02} \\
    \bottomrule
  \end{tabular}
\end{table}

\begin{figure}
  \centering
  \begin{subfigure}[b]{\textwidth}
    \centering
    \includegraphics[width=0.9\textwidth]{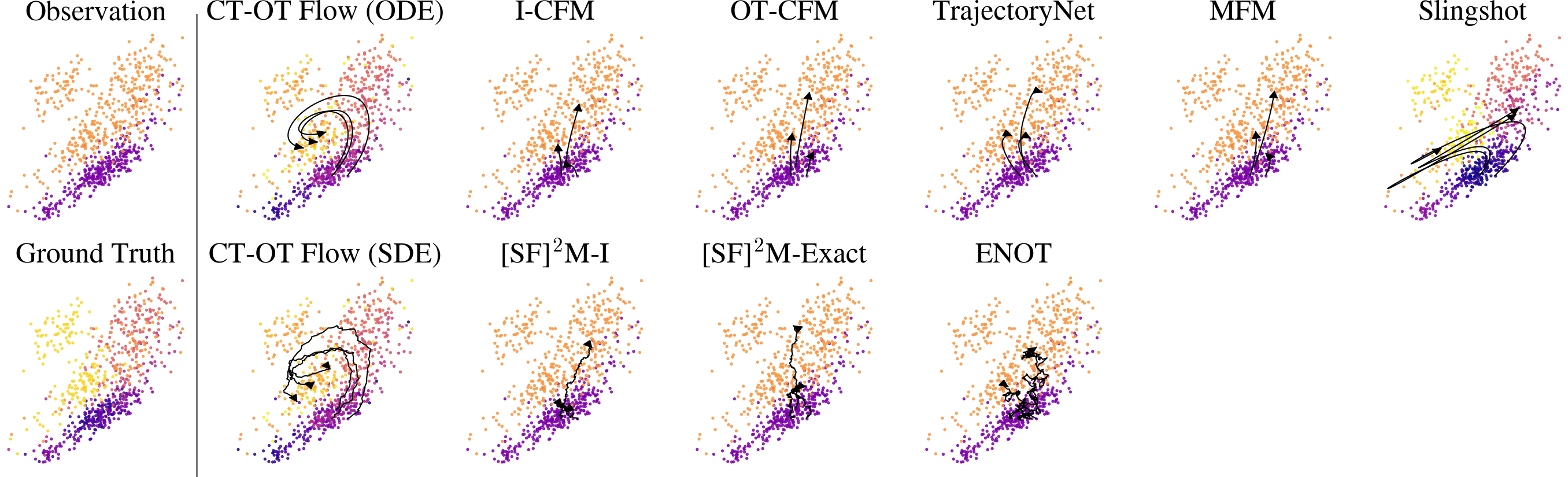}
    \caption{Bifurcation}
  \end{subfigure}
  \begin{subfigure}[b]{\textwidth}
    \centering
    \includegraphics[width=0.9\textwidth]{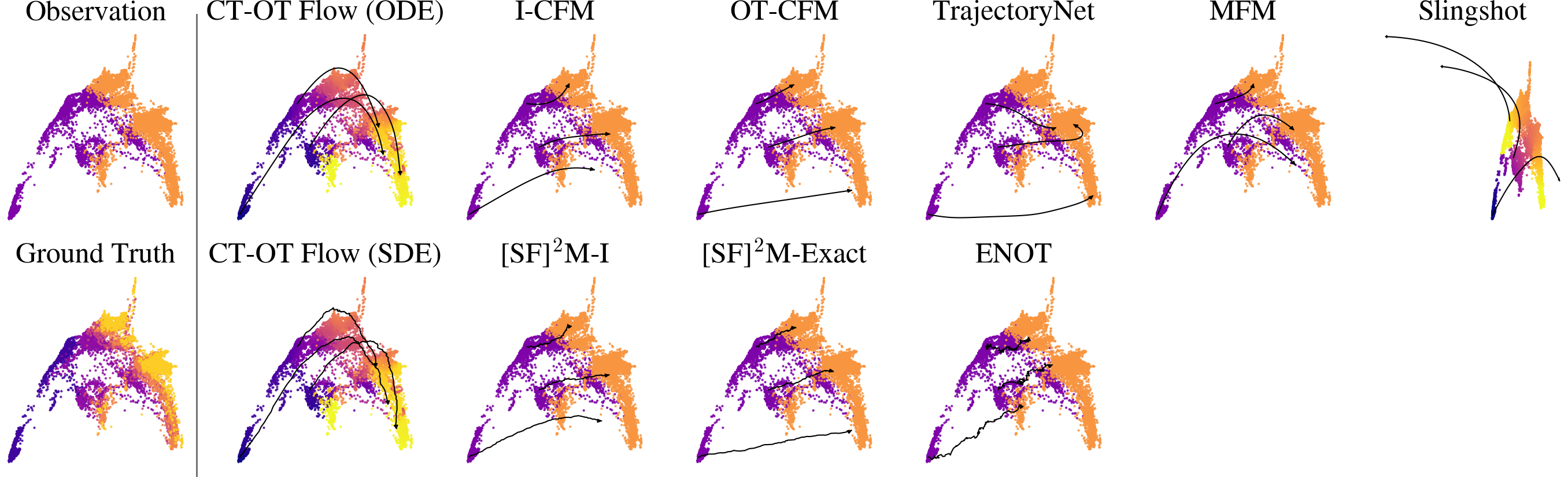}
    \caption{EB dataset}
  \end{subfigure}
  \caption{Estimated trajectories. The black lines indicate the true or estimated trajectories, while the color of each point in CT-OT Flow and Slingshot indicates its estimated high-resolution time label and pseudotime label, respectively.    \label{fig:scrna_trajectories}
  }
\end{figure}

\subsection{Application to meteorological data}
\label{sec:typhoon}
Finally, we evaluate CT-OT Flow on 528 typhoon trajectories recorded in Southeast Asia between 1951 and 2024.
The dataset records typhoon positions (longitude and latitude) at 6-hour intervals from their onset.
From the full dataset, we extracted 528 trajectories that span at least 150 hours and constructed temporal snapshots by assigning $t=0$ at 0 hours (onset), $t=1$ at 72 hours, and $t=2$ at 144 hours after onset.
Note that the index in each temporal snapshot was randomized, and the models only had access to individual data points, not the trajectories.
We split the trajectories into 70\% for training and 30\% for testing, and performed 10 evaluation runs with different neural network initializations.
See~App.~\ref{appendix:typhoon_dataset} for further details.

We evaluate each method by reconstructing trajectories in the test split, where the initial points are given at $t=0$.
Performance is measured using $\mathcal L_\text{DTW}$ and $\mathcal L_\text{Wass}$.
We treat the empirical data distribution at each time $t$ as the ground truth distribution $p^*_t(\boldsymbol x)$ for computing $\mathcal L_\text{Wass}$, and the full trajectories themselves as $X^*_j$ for $\mathcal L_\text{DTW}$.
Table~\ref{tab:typhoon} summarizes the evaluation results.
As shown in the table, CT-OT Flow improves the accuracy of trajectory prediction.
Note that CT-OT Flow (ODE) and CT-OT Flow (SDE) employ OT-CFM~\cite{tong2024improving} and $\text{[SF]}^2\text{M}$-Exact~\cite{tong2024simulation} for Step 3, respectively.
These differences highlight the improvements in prediction accuracy resulting from the use of CT-OT Flow.

\begin{table}
  \centering
  \caption{Estimation errors ($\mathcal L_{\text{DTW}}$ and $\mathcal L_{\text{Wass}}$) on the typhoon dataset.
  Bold: best within ODE/SDE groups.}
  \label{tab:typhoon}
  \begin{tabular}{ccc}
    \toprule
    Method$\downarrow$ Dataset$\rightarrow$                 & $\mathcal L_{\text{DTW}}$ & $\mathcal L_{\text{Wass}}$ \\   \toprule
    CT-OT Flow (ODE)                                        & \textbf{16.59}\std{0.25}  & \textbf{0.35}\std{0.05}    \\
    OT-CFM~\cite{tong2024improving}                         & 17.16\std{0.07}           & 0.48\std{0.01}             \\   \midrule
    CT-OT Flow (SDE)                                        & \textbf{16.48}\std{0.28}  & \textbf{0.36}\std{0.05}    \\
    $\text{[SF]}^2\text{M}$-Exact~\cite{tong2024simulation} & 17.10\std{0.07}           & 0.49\std{0.01}             \\
    \bottomrule
  \end{tabular}
\end{table}

\section{Conclusions and Limitations}
\label{sec:conclusion}
We introduced CT-OT Flow, a general framework that leverages POT to assign high-resolution time labels from discrete temporal snapshots and build a continuous-time data distribution, enabling more accurate reconstruction of underlying dynamics.
Our experiments on both synthetic and real-world datasets show that CT-OT Flow outperforms existing methods.
These improvements result from explicitly modeling temporal discretization and timestamp uncertainty.

\subsection*{Limitations}
\changed{
  CT-OT Flow relies on the assumption that the underlying distribution evolves smoothly across neighboring intervals, so that samples near the boundary of two contiguous intervals remain similar.
  This assumption can become less reliable under abrupt regime changes or strongly non-contiguous observations, in which case boundary identification and inferred high-resolution time labels may also become less reliable; see Apps.~\ref{appendix:non_contiguous} and~\ref{appendix:failure_case_for_boundary_extraction} for additional analysis.
  When the observation-time distribution $p(t)$ in Eqs.~\eqref{eq:times_backward} and~\eqref{eq:times_forward} is unknown, we assume it to be uniform within each interval.
  This approximation keeps the method practical, but if the true $p(t)$ is substantially non-uniform, it can bias the inferred time gaps and thus the estimated velocity magnitude (App.~\ref{appendix:non_uniform}).
  A promising direction for future work is to jointly infer $p(t)$ from external cues such as RNA velocity~\cite{la2018rna}.
  CT-OT Flow also assumes mass preservation at the probability-distribution level and does not explicitly model birth-death or abundance changes, for which an unbalanced OT extension would be more appropriate.
}

\changed{
  In addition, high-dimensional data can make boundary identification more difficult because pairwise distances tend to concentrate.
  In our experiments in App.~\ref{appendix:high_dimensional}, using a larger $\gamma$ partly alleviates this issue, and lower-dimensional representations may further help.
}

\changed{
  Our real-data evaluation enables quantitative comparison by aggregating several snapshots and then assessing whether finer-resolution snapshots can be reconstructed from these aggregated observations.
  Although this provides a controlled benchmark, it does not fully reflect the practical setting.
  App.~\ref{appendix:eb_unaggregated} reports a qualitative experiment on the original unaggregated EB dataset~\cite{moon2018embryoid}, which provides complementary evidence that CT-OT Flow can recover geometry-aligned trajectories in a less processed setting.
  Further evaluation in such real-world settings remains an important direction for future work.
}

\changed{
  Finally, CT-OT Flow assumes that each snapshot is an aggregation over a time interval, reflecting uncertainty in observation times.
  If the data instead consist of samples drawn directly from the instantaneous distribution $p_t(\boldsymbol{x})$ at each time $t$, our time-label inference step is unnecessary and is unlikely to provide additional benefit.
  In such cases, conventional methods that directly fit dynamics to $p_t(\boldsymbol{x})$ are more appropriate.
}

\clearpage

\appendix
\section{Background of partial optimal transport}
\label{appendix:pot}
In this section, we provide a formulation of optimal transport between two empirical distributions that consist of distinct sets of data points, with a particular focus on POT~\cite{bonneel2019spot,figalli2010optimal}, which is used in this work.
Optimal transport seeks a coupling between two distributions that minimizes a given cost function.
Let two finite point sets be $X=\{\boldsymbol x^{(1)},\dots, \boldsymbol x^{(N_x)}\}$ and $Y = \{\boldsymbol y^{(1)},\dots,\boldsymbol y^{(N_y)}\}$.
Here, $N_x,N_y$ are the respective number of data points, and $\boldsymbol x^{(i)},\boldsymbol y^{(j)}\in \mathbb R^d$ for all $i,j$.

Let $\hat p_X$ and $\hat p_Y$ be the empirical distributions corresponding to $X$ and $Y$, respectively.
Depending on the values of $\tau_x$ and $\tau_y$, different optimal transport regimes emerge for $\text{POT}_{(\tau_x, \tau_y)}(\hat p_X, \hat p_Y)$:
\begin{enumerate}
  \item ($\tau_x=\tau_y=1$). Full OT: all mass in $X$ and $Y$ is transported.
  \item ($\tau_x > \tau_y = 1$). One-sided POT: only a fraction $\frac{1}{\tau_x}$ of the total mass in $X$ is transported to the entirety of $Y$.
  \item ($\tau_y > \tau_x = 1$). One-sided POT: only a fraction $\frac{1}{\tau_y}$ of the total mass in $Y$ is transported to the entirety of $X$.
  \item ($\tau_x =\tau_y > 1$).  Two-sided POT: only subsets of points from both $X, Y$ are transported.
\end{enumerate}




\section{Non-uniform \texorpdfstring{$p(t)$}{p(t)} Assumption}
\label{appendix:non_uniform}
In the main text, we assumed $p(t)$ to be uniform within each interval due to lack of prior knowledge.
In reality, however, $p(t)$ can be non-uniform.
Retaining the uniform assumption when $p(t)$ is non-uniform causes the estimated velocity magnitude to deviate from the true value.

\subsection{One-dimensional illustrative example}
To illustrate this point, consider the simplest one-dimensional example $x(t)=t$, which has constant velocity $dx/dt=1$.
Suppose instead that $p(t)$ on $t\in [0,2]$ has the triangular form
\begin{align}
  \label{eq:non_uniform_pt}
  p(t) \propto
  \begin{cases}
    -a (t-1) + b & \text{if}\ t<1,   \\
    a (t-1) + b  & \text{otherwise},
  \end{cases}
\end{align}
with $b>0$.
When $a=0$, $p(t)$ reduces to the uniform case; when $a\neq 0$, the distribution becomes non-uniform and the data density varies over time.

Figure~\ref{fig:non_uniform_pt} compares histograms of data sampled from a uniform distribution ($a=0, b=1$) and a non-uniform distribution ($a=5,b=1$) (shown in the top panels).
The bottom panels depict the velocities inferred by CT-OT Flow when the input consists of two time intervals, $X_{[0,1]}$ and $X_{[1,2]}$.
Under the uniform $p(t)$, the estimated velocity remains close to $1$, which is consistent with the true velocity.
In contrast, when $p(t)$ is non-uniform, the estimated velocity tends to be lower in regions of higher data density (i.e., where $p(t)$ is larger) and higher in regions of lower density.
This observation reflects a fundamental difficulty, without prior knowledge of $p(t)$, one cannot distinguish whether the observed variability is due to a non-uniform distribution $p(t)$ or an inherently time-varying velocity.
\begin{figure}[t]
  \centering
  \includegraphics[width=0.5\textwidth]{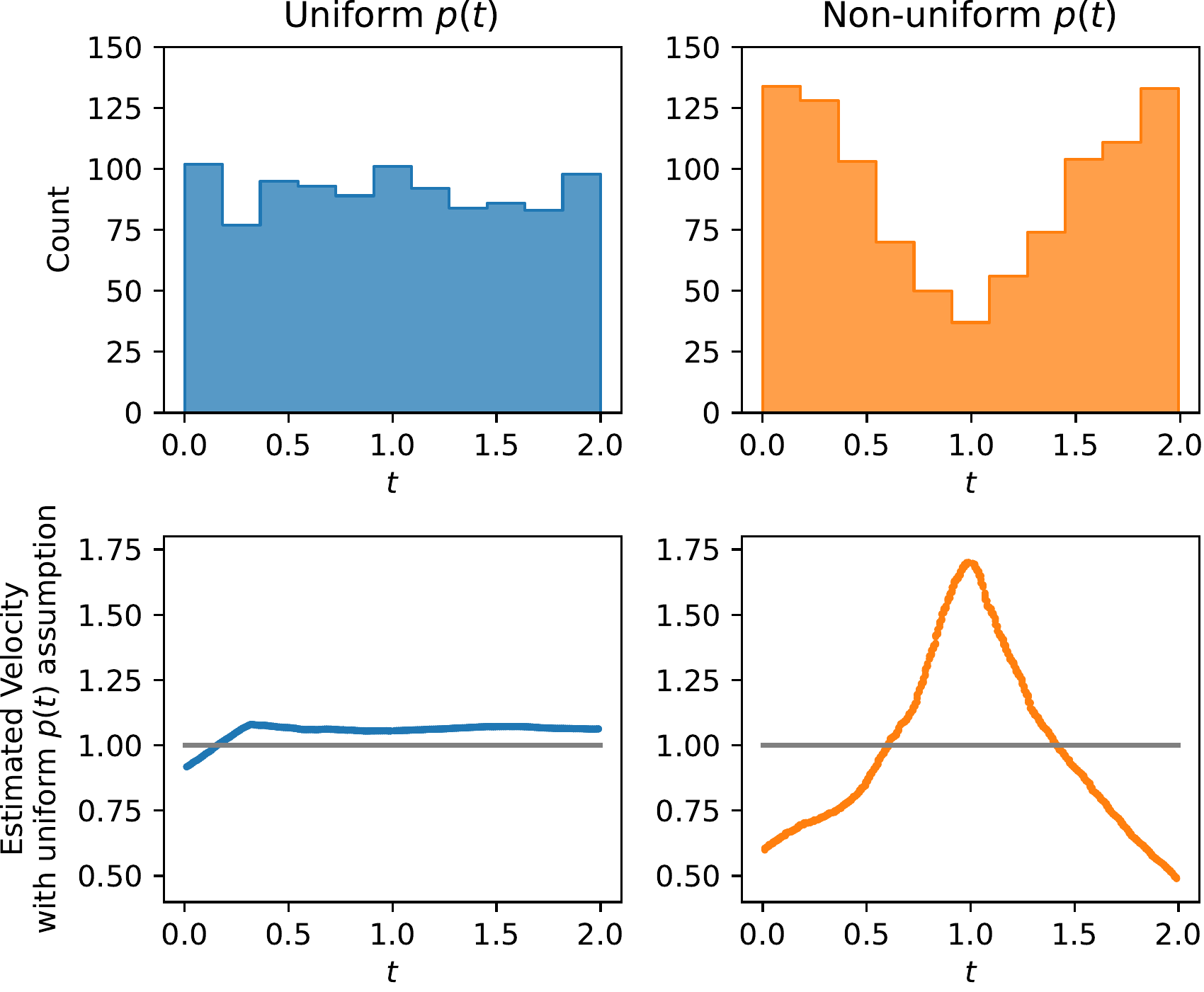}
  \caption{Estimated velocity (top) and sample histograms (bottom) if $p(t)$ is uniform (left) versus non-uniform (right).    \label{fig:non_uniform_pt}
  }
\end{figure}

When prior knowledge about $p(t)$ is available, CT-OT Flow incorporates it via Eqs.~\eqref{eq:times_backward}--\eqref{eq:times_forward}.
Figure~\ref{fig:non_uniform_pt_modify} compares, for $a=5,b=1$,
(left) high-resolution time labels and (right) velocities under uniform versus non-uniform $p(t)$ assumptions.
As shown, by incorporating the prior knowledge of $p(t)$, the accuracy of the estimated high-resolution time labels and the corresponding velocities can be improved.
However, in real-world scenarios, it is often difficult to obtain the inverse CDF of $p(t)$.
Note that the uniform $p(t)$ assumption remains a pragmatic choice when no prior information is available.
\begin{figure}
  \centering
  \includegraphics[width=0.5\textwidth]{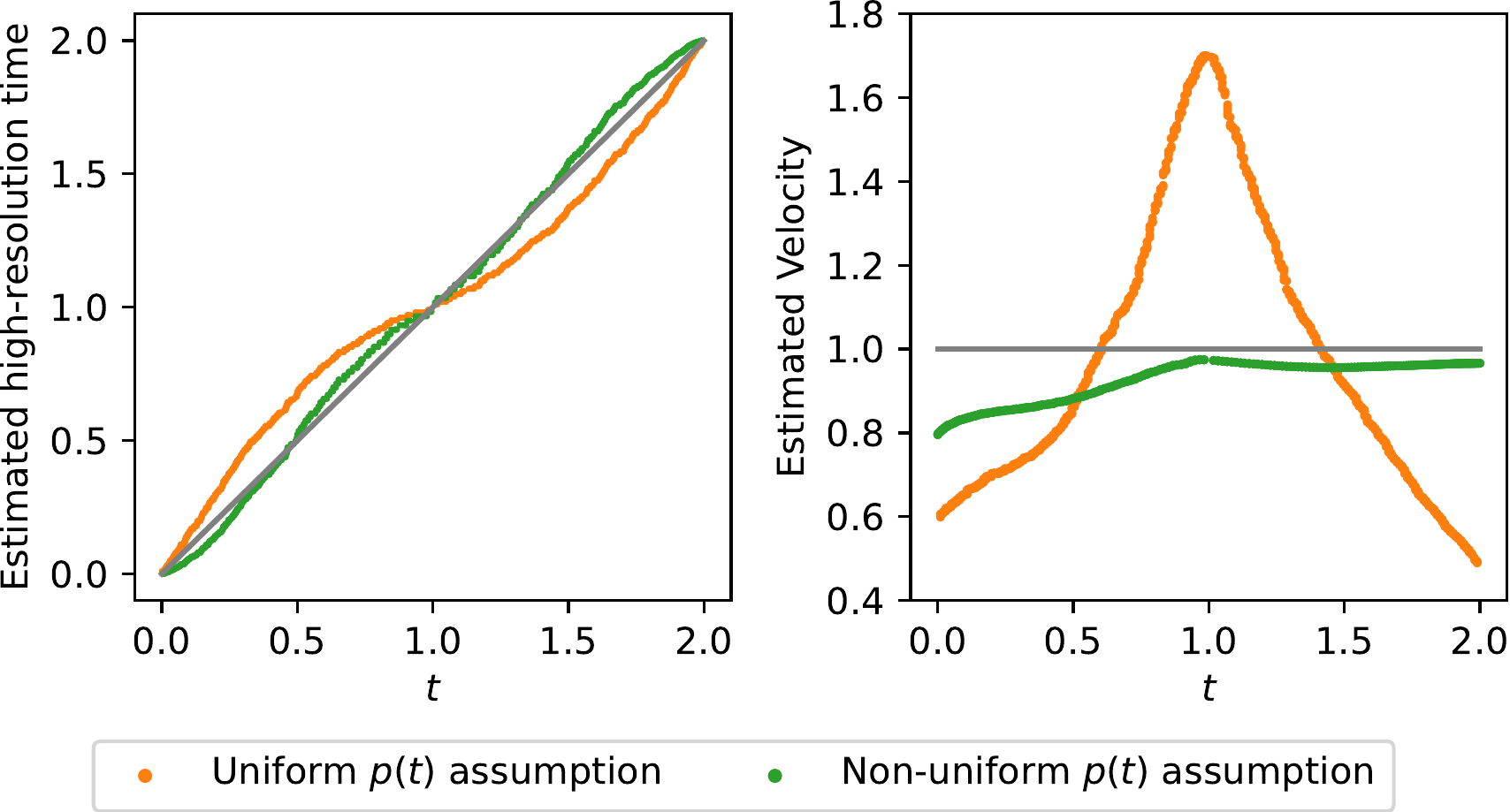}
  \caption{Estimated high-resolution time labels (left) and estimated velocities (right).}
  \label{fig:non_uniform_pt_modify}
\end{figure}

\subsection{Non-uniform Spiral dataset}
To further assess the impact of a mis-specified observation-time distribution,
we conducted an additional experiment on the Spiral dataset where the true $p(t)$ was designed as a mixture of Gaussian distributions with different means for each snapshot, while the model assumed $p(t)$ to be uniform.
Table~\ref{tab:spiral_nonuniform} summarizes the estimation errors.
Compared to the scenario with a truly uniform $p(t)$ (see Table~1 in the main text), the accuracy of CT-OT Flow degraded substantially.
As shown in Fig.~\ref{fig:gaussian_spiral}, the predicted trajectories often stalled before reaching the intended endpoint, consistent with the analysis above: the incorrect assumption about $p(t)$ led to systematic underestimation of the velocity magnitude.

\begin{table}[t]
  \centering
  \caption{Estimation errors on the Spiral dataset when the true $p(t)$ is non-uniform but the model assumes uniform $p(t)$. Mean $\pm$ std over multiple runs.}
  \label{tab:spiral_nonuniform}
  \begin{tabular}{lcc}
    \toprule
    Method & $\mathcal L_\text{DTW}$ & $\mathcal L_\text{Wass}$ \\
    \midrule
    CT-OT Flow (ODE)  & \textbf{32.15}\std{6.10} & \textbf{0.85}\std{0.15} \\
    OT-CFM            & 51.64\std{0.03} & 1.17\std{0.00} \\
    I-CFM             & 51.21\std{0.11} & 1.14\std{0.01} \\ \midrule
    CT-OT Flow (SDE)  & \textbf{31.89}\std{5.77} & \textbf{0.85}\std{0.16} \\
    $\text{[SF]}^2$M  & 51.54\std{0.07} & 1.17\std{0.01} \\
    $\text{[SF]}^2$M-I& 50.77\std{0.21} & 1.14\std{0.01} \\
    \bottomrule
  \end{tabular}
\end{table}

\begin{figure}[t]
  \centering
  \includegraphics[width=0.8\textwidth]{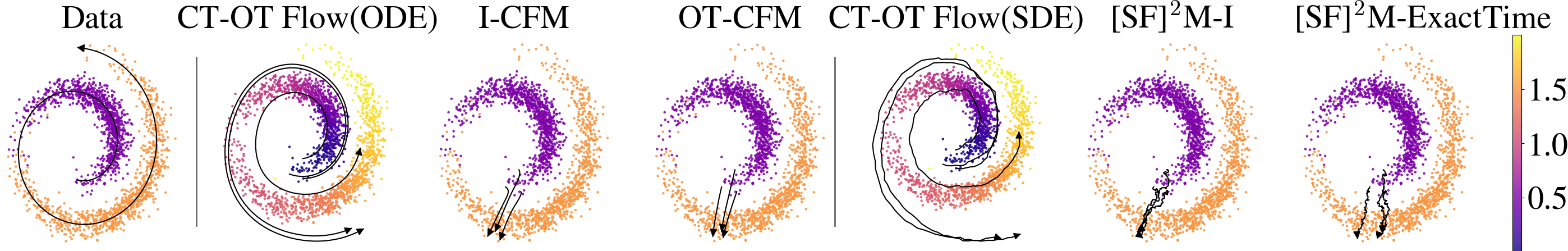}
  \caption{Estimated trajectories in non-uniform $p(t)$ setting. The black lines indicate the true or estimated trajectories, while the color of each point in CT-OT Flow indicates its estimated high-resolution time label.    \label{fig:gaussian_spiral}}
\end{figure}

\section{Continuous relaxation of the mixed-integer linear program}
\label{appendix:relaxation}
\subsection{Proof of Proposition\ref{prop:relaxation}}
\label{appendix:proof_relaxation}
Here, we prove Proposition~\ref{prop:relaxation} stated in \S\ref{sec:step1}.
\addtocounter{proposition}{-1}
\begin{proposition}
  The problem in Eq.~\eqref{eq:pw} is the continuous relaxation of the MILP in Eq.~\eqref{eq:MILP}.
\end{proposition}
\begin{proof}
  The MILP~\eqref{eq:MILP} selects subsets $S^-$ and $S^+$ from $X_{[t_{j-1},t_j]}$ and $X_{[t_j,t_{j+1}]}$, respectively, to minimize the 2-Wasserstein distance between the corresponding empirical distributions.
  For simplicity, let us denote $X^-=X_{[t_{j-1},t_j]}$ and $X^+=X_{[t_j,t_{j+1}]}$ with $N^-=|X^-|$ and $N^+=|X^+|$.
  Let $\boldsymbol x_-^{(i)}$ be the $i$-th point in $X^-$ and $\boldsymbol x_+^{(j)}$ be the $j$-th point in $X^+$.
  Define indicator vectors $\boldsymbol z^-\in \{0,1\}^{N^-}$ and $\boldsymbol z^+\in \{0,1\}^{N^+}$ such that:
  \begin{align}
    z_i^- =
    \begin{cases}
      1 & \text{if}\  \boldsymbol x^{(i)}_- \in S^- \\
      0 & \text{otherwise}
    \end{cases}, \ \
    z_j^+ =
    \begin{cases}
      1 & \text{if}\  \boldsymbol x^{(j)}_+ \in S^+ \\
      0 & \text{otherwise}.
    \end{cases}
  \end{align}
  The MILP in Eq.~\eqref{eq:MILP} can be written as:
  \begin{align}
    \label{eq:MILP_z}
    \min_{\boldsymbol z^-, \boldsymbol z^+} \min_{P'\in U'_{(\boldsymbol z^-, \boldsymbol z^+)}} \sum_{i=1}^{N^-} \sum_{j=1}^{N^+} P'_{ij}\|\boldsymbol x_-^{(i)} - \boldsymbol x_+^{(j)}\|_2^2, \ \ \text{s.t.} \ \sum_{i=1}^{N^-} z^-_i=N_z^-,\ \sum_{j=1}^{N^+} z^+_j=N_z^+,
  \end{align}
  where $N_z^-$ and $N_z^+$ are the number of the selected points that correspond to $|S^-| = \left\lceil\frac{|X_{[t_{j-1}, t_{j}]}|}{K}\right\rceil$ and $|S^+| = \left\lceil\frac{|X_{[t_{j}, t_{j+1}]}|}{K}\right\rceil$, respectively.
  The feasible set $U'_{(\boldsymbol z^-, \boldsymbol z^+)}$ is
  \begin{align}
    \label{eq:MILP_z_condition}
    U'_{(\boldsymbol z^-, \boldsymbol z^+)} = \left\{P'\in[0,1]^{N^-\times N^+}\ \middle |\ P'\boldsymbol 1_{N^+} = \frac{1}{N^-_z}\boldsymbol z^-,\ (P')^\top\boldsymbol 1_{N^-} = \frac{1}{N_z^+}\boldsymbol z^+,\ \boldsymbol 1_{N^-}^\top P'\boldsymbol 1_{N^+}=1\right\}.
  \end{align}
  Note that the row/column-sum constraints in Eq.~\eqref{eq:MILP_z_condition} force $\sum_{j} P'_{ij}=0$ whenever $z_i^- = 0$ and $\sum_{i} P'_{ij}=0$ whenever \(z_j^+ = 0\).
  Consequently $P'_{ij}=0$ if $\boldsymbol x_-^{(i)}\notin S^{-}$ or $\boldsymbol x_+^{(j)}\notin S^{+}$; the summations in Eq.~\eqref{eq:MILP_z} therefore effectively range only over pairs
  $(\boldsymbol x_-^{(i)},\boldsymbol x_+^{(j)})\in S^{-}\times S^{+}$.

  Next, we replace the indicator vectors $\boldsymbol z^-$ and $\boldsymbol z^+$ with continuous variables $\tilde{\boldsymbol z}^-\in [0,1]^{N^-}$ and $\tilde{\boldsymbol z}^+\in [0,1]^{N^+}$, relaxing the binary choice (selected/not selected)  to a fractional choice (selected at some proportion).
  Accordingly, the feasible set $U'_{(\tilde{\boldsymbol z}^-, \tilde{\boldsymbol z}^+)}$ is
  \begin{align}
    \label{eq:MILP_tilde_z_condition}
    U'_{(\tilde{\boldsymbol z}^-, \tilde{\boldsymbol z}^+)} = \left\{P'\in[0,1]^{N^-\times N^+}\ \middle |\ P'\boldsymbol 1_{N^+} = \frac{1}{N^-_z}\tilde{\boldsymbol z}^-,\ (P')^\top\boldsymbol 1_{N^-} = \frac{1}{N_z^+}\tilde{\boldsymbol z}^+,\ \boldsymbol 1_{N^-}^\top P'\boldsymbol 1_{N^+}=1\right\},
  \end{align}
  and $\sum_{i=1}^{N^-} \tilde{z}^-_i=N_z^-,\ \sum_{j=1}^{N^+} \tilde{z}^+_j=N_z^+$.
  These constraints can be transformed into a manner equivalent to partial transport conditions:
  \begin{align}
    \begin{split}                                & \left(\sum_{i=1}^{N^-} z^-_i=N_z^-\right) \land \left(P'\boldsymbol 1_{N^+} = \frac{1}{ N^-_z}\tilde{\boldsymbol z}^- \right) \land \left(\boldsymbol 1_{N^-}^\top P'\boldsymbol 1_{N^+}=1 \right) \\
      \Rightarrow & \left(P'\boldsymbol 1_{N^+} \leq \frac{1}{ N^-_z}\boldsymbol 1_{N^-}\right) \land \left(\boldsymbol 1_{N^-}^\top P'\boldsymbol 1_{N^+}=1 \right)
    \end{split} \\
    \begin{split}                                 & \left(\sum_{j=1}^{N^+} z^+_j=N_z^+ \right) \land \left((P')^\top\boldsymbol 1_{N^-} = \frac{1}{ N^+_z}\tilde{\boldsymbol z}^+ \right) \land \left(\boldsymbol 1_{N^-}^\top P'\boldsymbol 1_{N^+}=1\right) \\
      \Rightarrow & \left((P')^\top\boldsymbol 1_{N^-} \leq \frac{1}{ N^+_z}\boldsymbol 1_{N^+}\right) \land \left(\boldsymbol 1_{N^-}^\top P'\boldsymbol 1_{N^+}=1 \right)
    \end{split}
  \end{align}
  Obviously, $\tilde{\boldsymbol z}^-$ (and $\tilde{\boldsymbol z}^+$) that satisfies the left condition exists if the right condition holds.
  Hence, the original problem (Eq.~\eqref{eq:MILP_z}) transforms into
  \begin{align}
    \label{eq:partial_appendix}
    \min_{P\in U_{(\frac{N^-}{N_z^-}, \frac{N^+}{N_z^+})}} \sum_{i=1}^{N^-} \sum_{j=1}^{N^+} P_{ij}\|\boldsymbol x_-^{(i)} - \boldsymbol x_+^{(j)}\|_2^2,
  \end{align}
  where
  \begin{align}
    U_{(\frac{N^-}{N_z^-}, \frac{N^+}{N_z^+})} = \left\{P\in[0,1]^{N^-\times N^+}\ \middle |\ P\boldsymbol 1_{N^+} \leq \frac{1}{N_z^-}\boldsymbol 1_{N^-},\ P^\top\boldsymbol 1_{N^-} \leq \frac{1}{N_z^+}\boldsymbol 1_{N^+},\ \boldsymbol 1_{N^-}^\top P\boldsymbol 1_{N^+}=1\right\}.
  \end{align}

  Thus, the POT formulation in Eq.~\eqref{eq:partial_appendix} shares exactly the same objective function as the original MILP in Eq.~\eqref{eq:MILP_z}, while its feasible set contains the discrete feasible set obtained when the indicator variables are restricted to $\{0,1\}$.
  Hence Eq.~\eqref{eq:partial_appendix} is a continuous relaxation of Eq.~\eqref{eq:MILP_z}: every integer-feasible solution is feasible for the relaxed POT problem.
\end{proof}

\subsection{Continuous relaxation for \texorpdfstring{$k\geq 2$}{k>2}}
We now provide a proof of the following proposition concerning the continuous relaxation of the optimization problem in Eqs.~\eqref{eq:k2_1} and \eqref{eq:k2_2} in \S\ref{sec:step1}.
\begin{proposition}
  \label{prop:relaxation2}
  For $k=1,\dots ,K-2$, the two optimization problems below become $\text{POT}_{(1,K - k)}(\hat p_{S_k^-}, \hat p_{X_{[t_{j-1}, t_{j}]}\setminus \cup_{k'=1}^{{k}}S_{k'}^-})$ and $\text{POT}_{(1, K - k)}(\hat p_{S_k^+}, \hat p_{X_{[t_{j}, t_{j+1}]}\setminus \cup_{k'=1}^{{k}}S_{k'}^+})$, respectively, when relaxed to continuous variables:
  \begin{align}
    \tag{\ref{eq:k2_1}}
    \min_{ S_{k+1}^-\subset X_{[t_{j-1}, t_{j}]}\setminus \cup_{k'=1}^{{k}}S_{k'}^-} \mathcal W^2(\hat p_{S^-_{k}}(\boldsymbol x), \hat p_{S^-_{k+1}}(\boldsymbol x)) \ \ \  & \text{(backward)} \\
    \tag{\ref{eq:k2_2}}
    \min_{ S_{k+1}^+\subset X_{[t_{j}, t_{j+1}]}\setminus \cup_{k'=1}^{{k}}S_{k'}^+} \mathcal W^2(\hat p_{S^+_{k}}(\boldsymbol x), \hat p_{S^+_{k+1}}(\boldsymbol x)) \ \ \  & \text{(forward)}
  \end{align}
\end{proposition}
\begin{proof}
  Consider first the optimization problem Eq.~\eqref{eq:k2_1}.
  Let $\boldsymbol x_-^{(i)}\in X_{[t_{j-1}, t_{j}]}\setminus \cup_{k'=1}^{{k}}S_{k'}^-$
  and $\boldsymbol x_+^{(j)}\in S_k^-$.
  As in App.~\ref{appendix:proof_relaxation}, define an indicator vector $\boldsymbol z^-$ by
  \begin{align}
    z_i^- =
    \begin{cases}
      1 & \text{if}\  \boldsymbol x_-^{(i)} \in S_{k+1}^- , \\
      0 & \text{otherwise}.
    \end{cases}
  \end{align}
  Then this MILP can be written as
  \begin{align}
    \label{eq:MILP_partial}
    \min_{\boldsymbol z^-} \min_{P''\in U''_{(\boldsymbol z^-,1)}} \sum_{i=1}^{N_k^-} \sum_{j=1}^{N_k^+} P''_{ij}\|\boldsymbol x_-^{(i)} - \boldsymbol x_+^{(j)}\|_2^2, \ \ \text{s.t.} \ \sum_{i=1}^{N^-_k} z^-_i=N_z^-,
  \end{align}
  where $N_k^- = |X_{[t_{j-1}, t_{j}]}\setminus \cup_{k'=1}^{{k}}S_{k'}^-|$, $N_k^+ = |S_k^-|$, and
  \begin{align}
    \label{eq:MILP_partial_z_condition}
    U''_{(\boldsymbol z^-,1)} = \left\{P''\in[0,1]^{N^-\times N^+}\ \middle |\ P''\boldsymbol 1_{N^+} = \frac{1}{N^-_z}\boldsymbol z^-,\ (P'')^\top\boldsymbol 1_{N^-} = \frac{1}{N^+}\boldsymbol 1_{N^+},\ \boldsymbol 1_{N^-}^\top P''\boldsymbol 1_{N^+}=1\right\}.
  \end{align}
  Note that, unlike in the previous MILP, all data in $S_k^-$ are transported, i.e., $(P'')^\top\boldsymbol 1_{N^-} = \frac{1}{N^+}\boldsymbol 1_{N^+}$.
  By relaxing $\boldsymbol z^-$ to continuous values, we obtain:
  \begin{align}
    U_{(\frac{N^-}{N_z^-},1)} = \left\{P\in[0,1]^{N^-\times N^+}\ \middle |\ P\boldsymbol 1_{N^+} \leq \frac{1}{N_z^-}\boldsymbol 1_{N^-},\ P^\top\boldsymbol 1_{N^-} = \frac{1}{N^+}\boldsymbol 1_{N^+},\ \boldsymbol 1_{N^-}^\top P'\boldsymbol 1_{N^+}=1\right\}.
  \end{align}
  Setting $\frac{N^-}{N_z^-} = K-k$ yields the POT problem
  \begin{align}
    \text{POT}_{(K - k,1)}(\hat p_{X_{[t_{j-1}, t_{j}]}\setminus \cup_{k'=1}^{{k}}S_{k'}^-}, \hat p_{S_k^-})
  \end{align}
  A similar argument holds for Eq.~\eqref{eq:k2_2} (the forward case).
  Hence, by continuous relaxation, each step's MILP can be interpreted as a POT problem.
\end{proof}

\section{Algorithm for CT-OT Flow}
\label{appendix:algorithm}
Algorithm~\ref{algorithm:proposed} outlines the procedure for CT-OT Flow.

\begin{algorithm}[H]
  \caption{Dynamics estimation using CT-OT Flow}
  \label{algorithm:proposed}
  \begin{algorithmic}[1]
    \State \textbf{Input:}\quad Datasets $\{X_{[t_j, t_{j+1}]}\}_{j=1}^T$
    \State \textbf{Output:}\quad Continuous-time distribution $\tilde{p}_t(x)$ and ODE/SDE model
    \Statex \textbf{Step 1: high-resolution time label estimation}
    \For{$j = 1,\dots,T-1$}
    \State extract subsets $S^-$ and $S^+$ via Eq.~\eqref{eq:pw}
    \For{$k = 1$ to $K-2$}
    \State extract subsets $S^{\pm}_{k+1}$ via forward/backward selection via Eqs.\eqref{eq:k2_1} and \eqref{eq:k2_2}.
    \EndFor
    \State $S_K^- = X_{[t_{j-1}, t_{j}]} \setminus \cup_{k=1}^{K-1}S_k^-$
    \State $S_K^+ = X_{[t_{j}, t_{j+1}]} \setminus \cup_{k=1}^{K-1}S_k^+$
    \EndFor
    \State assign each data point $x^{(i)}$ a high-resolution time label $\tilde{t}^{(i)}$ via Eqs.~\eqref{eq:times_backward} and~\eqref{eq:times_forward}.
    \Statex \textbf{Step 2: kernel-based time-smoothing of data distributions}
    \State Construct $\tilde{p}_t(x)$ via Eq.~\eqref{eq:kernel_time_dist}.

    \Statex \textbf{Step 3: ODE/SDE training}
    \For{$n = 1$ to $N_{\text{iter}}$}
    \State draw sample pairs $(x^{(0)}, x^{(1)})$ from $\tilde{p}_t(x)$ and $\tilde{p}_{t+\delta t}(x)$
    \State Update model parameters under an ODE/SDE training loss (e.g., Eq.~\eqref{eq:loss_rf}).
    \EndFor
    \State \Return $\tilde{p}_t(x)$ and the trained ODE/SDE model
  \end{algorithmic}
\end{algorithm}

\section{Computation time}
\label{appendix:computational_time}
The runtime of CT-OT Flow depends on the number of data points per interval $N$ and the subdivision factor $K$.
We present three acceleration techniques: screening, mini-batching, and Sinkhorn iterations~\cite{cuturi2013sinkhorn}.
We then investigate how the actual computation time changes and memory usage as $N$ and $K$ vary using network simplex algorithm~\cite{ahuja1993network} with the techniques.

\subsection{Fast computation via screening}
\label{appendix:screening}
Screening reduces cost by pre-filtering irrelevant points before POT.
Although POT is applied over all points to capture boundary regions between contiguous intervals, only points near those boundaries meaningfully contribute to the transport plan.
Points that are far from all points in the opposite group are unlikely to be selected as transport pairs and are therefore redundant in the optimization.

Screening removes such irrelevant points prior to POT optimization, using only the pairwise distance matrix.
Let $D$ denote the pairwise distance matrix between two point sets, $\boldsymbol x^{(m)} \in X$ and $\boldsymbol y^{(n)} \in Y$.
The following points are excluded during the screening process:
\begin{align}
  \{\boldsymbol x^{(m)} \in X \mid \min_{n} D_{mn} > \alpha\}, \quad \{\boldsymbol y^{(n)} \in Y \mid \min_{m} D_{mn} > \beta\},
\end{align}
where $\alpha$ and $\beta$ are distance thresholds.

This pre-filtering step reduces the number of candidate points and accelerates computation.
The thresholds $\alpha$ and $\beta$ can be adjusted so that, for example, $c \, \frac{N}{K}$ data points (i.e., $c$ times the number of selected points) are retained from each interval for the POT computation.
If $c$ is sufficiently large, the solution to the POT problem remains almost unchanged, and the high-resolution time labels obtained by CT-OT Flow are also unaffected.

Figure \ref{fig:correlation_c} illustrates Spearman's rank correlation between the estimated high-resolution time labels and the ground truth with varying the screening parameter $c$ over 5 runs on Spiral dataset.
As shown in Fig.~\ref{fig:correlation_c}, setting $c=10$ already yields the high Spearman's rank correlation of over 0.99.

\begin{figure}
  \centering
  \includegraphics[width=0.4\textwidth]{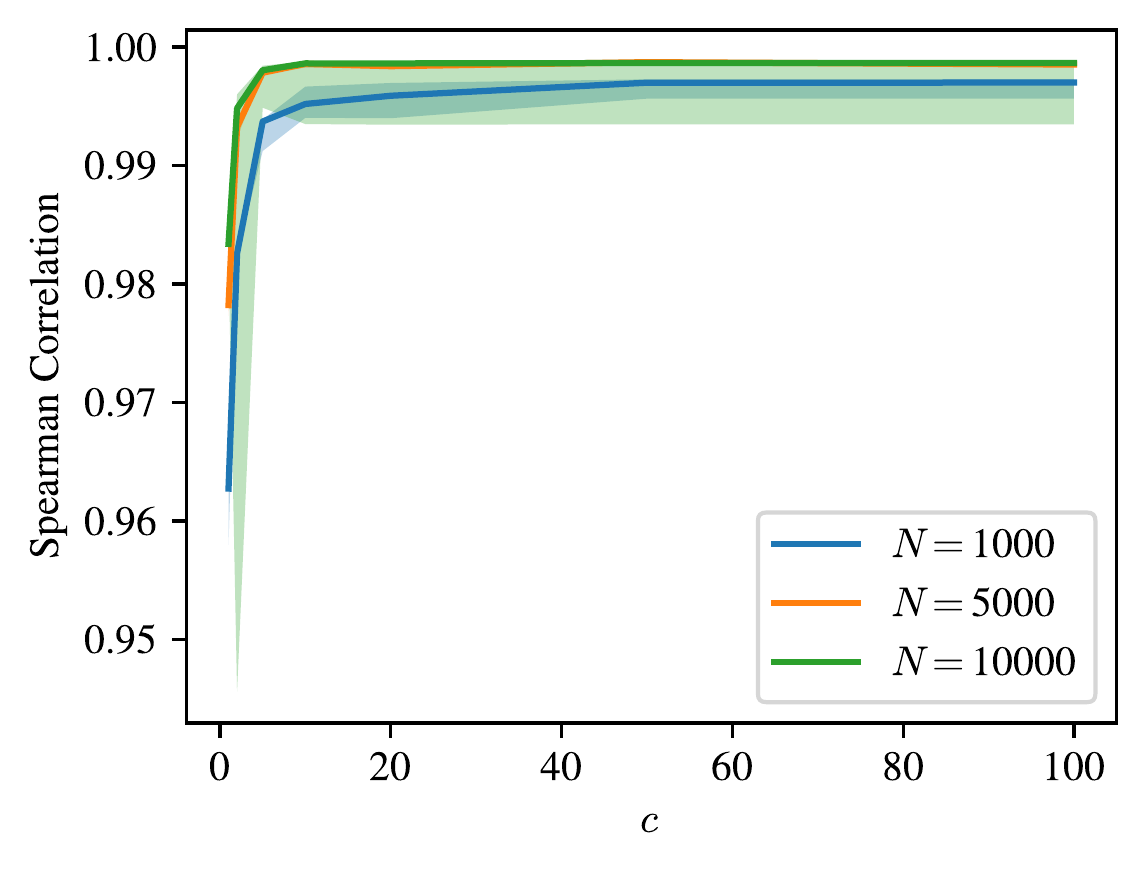}
  \caption{Spearman correlation between the estimated high-resolution time labels with screening setting ($c>0$) and ground truth. The areas correspond to 0.25 and 0.75 quantiles.}
  \label{fig:correlation_c}
\end{figure}

\subsection{Fast computation with mini-batches}
\label{appendix:random_subsets}
Mini-batching reduces cost by partitioning each interval’s data into disjoint batches of size $m$.
When the total number of points per interval $N$ is very large, applying CT-OT Flow to the entire set becomes prohibitively expensive; as shown in App.~\ref{appendix:empirical_runtime}, even with screening it is infeasible when $N>10^6$.
A simple remedy is to randomly split the dataset into $G = \lceil N/m \rceil$ mini-batches of size $m$ and run Step~1 of CT-OT Flow independently on each.
If the mini-batch size is sufficiently large, the induced error in the inferred high-resolution labels is expected to be negligible; the labels thus closely approximate those obtained by a full-data run.
For fixed $m$, total cost scales linearly with $N$, enabling scalability to massive datasets.
The mini-batch size $m$ controls the trade-off between accuracy and runtime.

Figure~\ref{fig:correlation_m} shows the Spearman's rank correlation between the ground truth and high-resolution time labels from the mini-batch setting as $m$ and $N$ vary.
Figure~\ref{fig:correlation_m} shows that with $m=1,000$ for $N=10,000$, the Spearman's rank correlation between the inferred high-resolution times from mini-batches and the ground truth exceeds 0.99.

\begin{figure}
  \centering
  \includegraphics[width=0.4\textwidth]{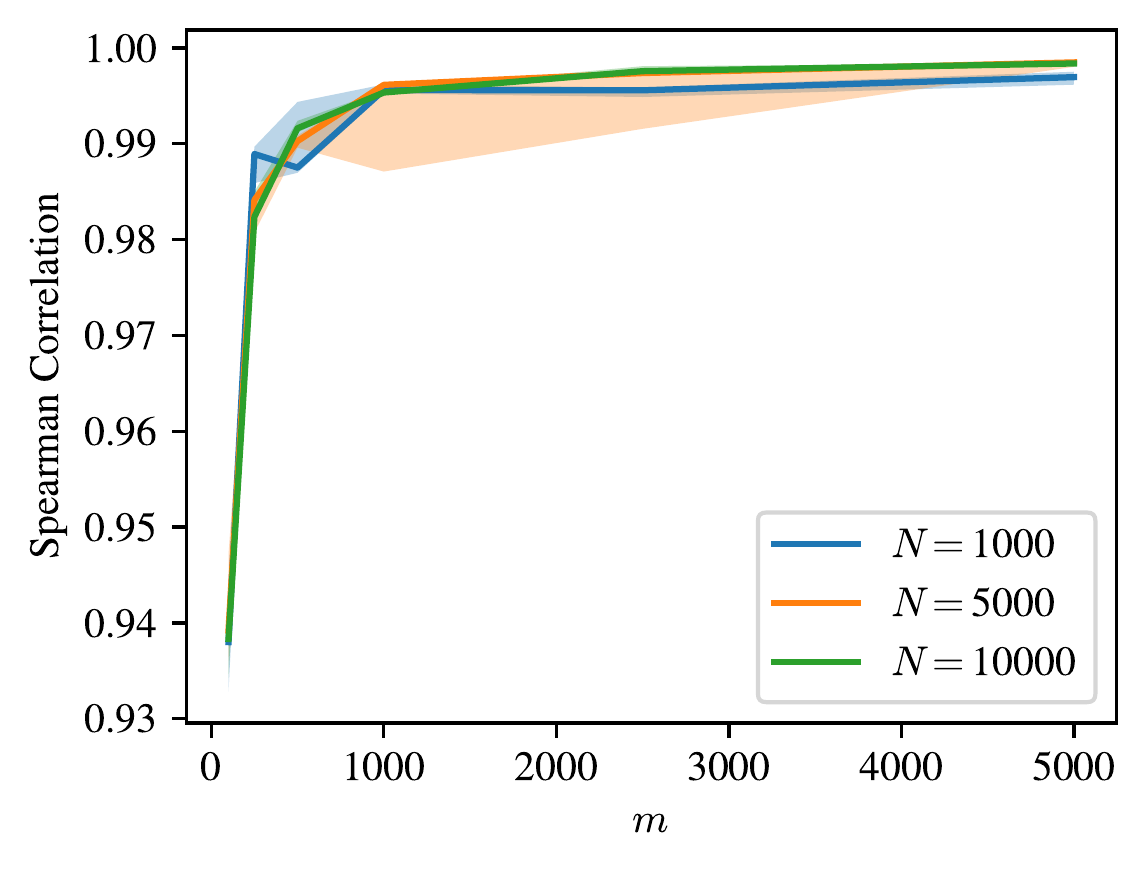}
  \caption{Spearman correlation between the high-resolution time labels with mini-batch setting and ground truth. The areas correspond to 0.25 and 0.75 quantiles.\label{fig:correlation_m}
  }
\end{figure}

\subsection{\changed{Effect of Sinkhorn entropy regularization}}
\label{app:sinkhorn-reg}

Sinkhorn-based optimal transport has attracted attention as an efficient approach for transport computation\cite{cuturi2013sinkhorn}.
By introducing an entropy regularization term, Sinkhorn replaces a sharp transport plan with a smoother coupling, while enabling efficient matrix-scaling updates that can be implemented on GPUs.

To examine how this entropy regularization affects temporal refinement in CT-OT Flow, we replace the original unregularized POT with Sinkhorn iterations and vary the entropy regularization coefficient.
Figure~\ref{fig:correlation_entropy} shows how the correlation between the estimated high-resolution time labels and the ground-truth time labels changes as the entropy regularization strength varies.
In most cases, the estimation accuracy deteriorates when entropy regularization is introduced.
This is because the smoothing effect of entropy regularization blurs the transport structure and makes it more difficult to identify data points located near the temporal boundaries.
In addition, when the entropy regularization is made weak in order to preserve a sharp transport structure, Sinkhorn iterations tend to become numerically less stable and typically require more iterations to converge~\cite{peyre2019computational}.

\begin{figure}
  \centering
  \includegraphics[width=0.4\textwidth]{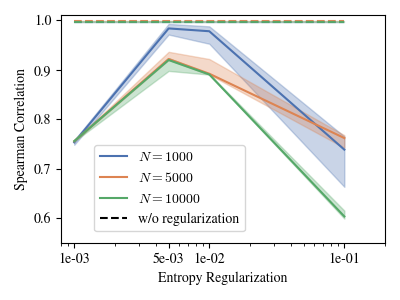}
  \caption{Spearman correlation between the estimated high-resolution time labels with entropy regularization and ground truth. The areas correspond to 0.25 and 0.75 quantiles.}
  \label{fig:correlation_entropy}
\end{figure}

\subsection{Empirical runtime and memory usage}
\label{appendix:empirical_runtime}
We benchmarked Step~1 runtimes using the network simplex algorithm and the Sinkhorn solver~\cite{cuturi2013sinkhorn} (entropic regularization $\epsilon=0.1$).
For Sinkhorn, we set a maximum of $10^4$ iterations and a convergence tolerance of $10^{-9}$.
We also applied the screening method (with $c=10$) and the mini-batch method (with $m=1,000$) using the network simplex solver.
Experiments with the simplex algorithm were run on a CPU (AMD EPYC 7742 64-core, single thread), while Sinkhorn iterations were executed on a GPU (NVIDIA A100 40 GB).

Figure~\ref{fig:computational_time} shows runtimes for (i) $K=100$ with varying $N$, and (ii) $N=5,000$ with varying $K$.
Runtimes increase monotonically with $N$ for all methods.
For $N>10,000$, Sinkhorn outperforms simplex, although it cannot run beyond $N\approx30,000$ due to GPU memory limits.
Even for $N=100,000$, screening ($c=10$) recovers high-resolution labels in tens of minutes on CPU only, while mini-batching yields labels in approximately 30 s.
The effect of $K$ on runtime is minor, but smaller $K$ (i.e., matching more points per iteration) slows convergence.
For a typical setting ($K\approx100$), screening and mini-batch methods exhibit faster runtime than the network simplex algorithm.
Additionally, for $N=1,000,000$, our mini-batch setting ($m=1,000$) achieves a Spearman's rank correlation of 0.994 between the high-resolution time labels and the ground truth with a runtime of approximately 480 s on a single CPU.

\begin{figure}
  \centering
  \includegraphics[width=0.7\textwidth]{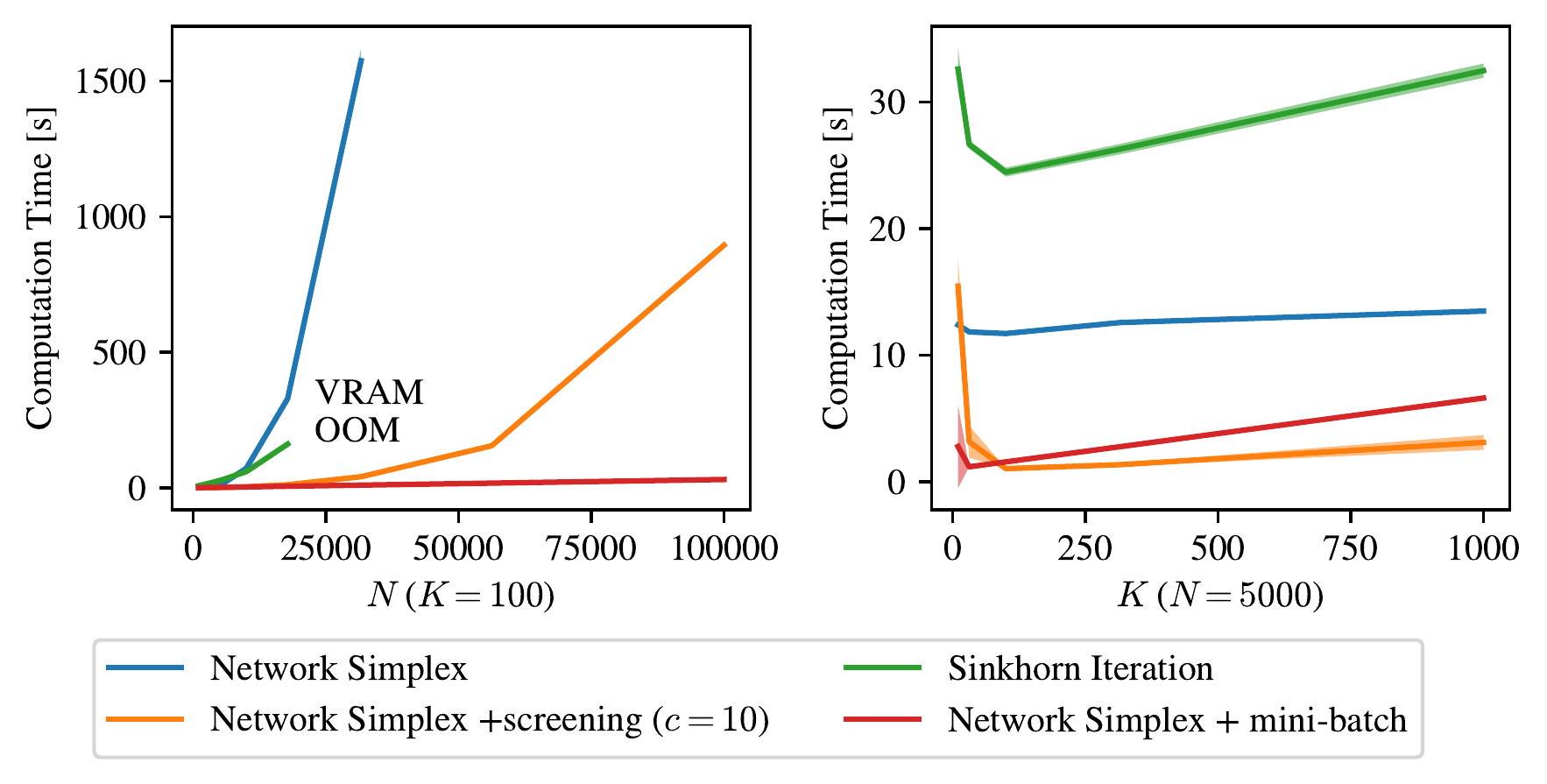}
  \caption{Runtimes of Step~1.
    Lines indicate means and shaded areas the 25th–75th quantiles over 5 runs.
    “OOM” marks out-of-memory failures.
  \label{fig:computational_time}}
\end{figure}

\begin{figure}
  \centering
  \includegraphics[width=0.7\textwidth]{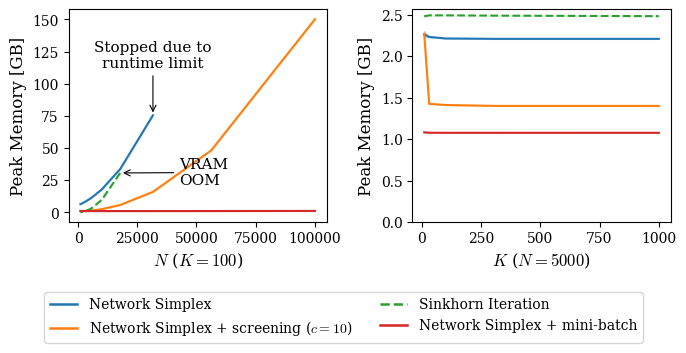}
  \caption{Peak memory usage of Step~1.
    Lines indicate means and shaded areas the 25th--75th quantiles over 5 runs.
    Solid lines denote CPU memory (peak RSS), while the dashed Sinkhorn line denotes peak GPU memory usage.
    “OOM” marks GPU out-of-memory failures.
  }
  \label{fig:memory}
\end{figure}

\changed{
  In addition to runtime, we measured the memory usage of Step~1 in terms of peak resident set size (RSS) on CPU and, for Sinkhorn, peak GPU memory usage.
  Figure~\ref{fig:memory} shows that CPU memory usage increases sharply with $N$ for the network simplex solver.
  Screening reduces the peak CPU memory, whereas the mini-batch variant remains nearly constant at around $1.1$--$1.2$~GB over the entire range.
  For Sinkhorn, GPU memory grows rapidly with $N$, reaching about $31$~GB at the largest feasible setting and resulting in GPU out-of-memory for $N \gtrsim 3\times 10^4$.
  Note that CT-OT Flow with the network simplex remains feasible for moderate problem sizes (e.g., $N\approx 20,000$) even without mini-batching or screening.
  The effect of $K$ is minor overall: CPU memory remains roughly stable across $K\in[10,1000]$ for the simplex-based methods, and Sinkhorn also exhibits nearly constant GPU memory usage over this range.
}

\section{Parameters}
\label{appendix:experiments_parameters}
We follow prior work in using the Euclidean distance as the ground metric for computing the 1-Wasserstein distance.
Table~\ref{tab:parameters} summarizes the hyperparameters used in our experiments.
For OT-CFM~\cite{tong2024improving}, $\text{[SF]}^2$M~\cite{tong2024simulation}, MFM~\cite{kapusniak2024metric},  ENOT~\cite{gushchin2024entropic}, and Slingshot~\cite{street2018slingshot}, the parameters were set according to the publicly available implementation at \\
\url{https://github.com/atong01/conditional-flow-matching} (accessed on February 18, 2025), \\
\url{https://github.com/KrishnaswamyLab/TrajectoryNet} (accessed on February 18, 2025), \\
\url{https://github.com/kksniak/metric-flow-matching} (accessed on March 5, 2025), \\
\url{https://github.com/ngushchin/EntropicNeuralOptimalTransport} (accessed on April 12, 2025), and\\
\url{https://github.com/mossjacob/pyslingshot} (accessed on July 25, 2025), respectively.

\begin{table}
  \centering
  \caption{List of default parameters.}
  \label{tab:parameters}
  \begin{tabular}{p{4.0cm}cc}\toprule
    Model                                                                         & Hyperparameters                      & Values                                         \\ \midrule
    \multirow[t]{5}{*}{CT-OT Flow}                                                                    & $K$                                  & 100 (unless otherwise stated)                  \\
    & $\gamma$                             & 0.005 (unless otherwise stated)                \\
    & $\delta t$                           & 0.1                                            \\
    & POT solver                           & Python Optimal Transport~\cite{flamary2021pot} \\
    & POT Computation Method               & Network simplex~\cite{ahuja1993network}        \\ \midrule
    \multirow[t]{8}{2cm}{I-/OT-CFM~\cite{tong2024improving}, $\text{[SF]}^2$M~\cite{tong2024simulation}}                                            & Gradient Norm Clipping               & 0.1                                            \\
    & Learning Rate                        & 1e-4                                           \\
    & Number of Iterations                 & 5000                                           \\
    & Batch Size                           & 128                                            \\
    & Number of Hidden Layers              & 3 (Fully Connected)                             \\
    & Number of Hidden Nodes               & 64                                             \\
    & Activation Function                  & SELU~\cite{klambauer2017self}                  \\
    & $ma$                             & 0.1                                            \\ \midrule
    \multirow[t]{8}{2cm}{MFM~\cite{kapusniak2024metric}}                                                & Batch Size                           & 128                                            \\
    & $\gamma$                             & 0.125                                          \\
    & $\rho$                               & 0.001                                          \\
    & n\_center                            & 100                                            \\
    & $\kappa$                             & 1.0                                            \\
    & Metric Network Learning Rate         & 1e-2                                           \\
    & Metric Network Optimizer             & Adam~\cite{kingma2014adam}                     \\ 
    & Metric  Network Number of Epochs     & 50                                             \\
    & Geopath Network Learning Rate        & 1e-4                                           \\
    & Geopath Network Optimizer            & Adam~\cite{kingma2014adam}                     \\
    & Geopath Network Weight Decay         & 1e-5                                           \\
    & Flow Network Learning Rate           & 1e-3                                           \\
    & Flow Network Optimizer               & AdamW~\cite{loshchilov2017decoupled}           \\
    & Flow Network Weight Decay            & 1e-3                                           \\
    & Flow Network $ma$                & 0.1                                            \\ \midrule
    \multirow[t]{8}{2cm}{ENOT~\cite{gushchin2024entropic}}                                              & Number of Iterations                 & 1000                                           \\
    & Number of Inner Iterations           & 10                                             \\
    & $\epsilon$                           & 0.1                                            \\
    & Shift Model Number of Hidden  Layers & 3 (Fully Connected)                             \\
    & Shift Model Number of Hidden Nodes   & 100                                            \\
    & Shift Model Optimizer                & Adam~\cite{kingma2014adam}                     \\
    & Shift Model Learning Rate            & 0.0001                                         \\
    & Shift Model Number of Steps          & 100                                            \\
    & Beta Net Number of Hidden  Layers    & 3 (Fully Connected)                             \\
    & Beta Net Number of Hidden Nodes      & 100                                            \\
    & Beta Net Optimizer                   & Adam~\cite{kingma2014adam}                     \\
    & Beta Net Learning Rate               & 0.0001                                         \\ \midrule
    \multirow[t]{8}{2cm}{TrajectoryNet~\cite{tong2020trajectorynet}}                                    & Mode                                 & Base                                           \\
    & Number of Iterations                 & 10,000 \\
    & Number of Hidden Layers (Fully Connected) & 3 \\
    & Number of Hidden Nodes & 64 \\
    & Activation Function & Tanh \\
    & Batch Size & 1000 \\
    & Optimizer & Adam \\
    & Learning Rate & 0.001 \\
    & Weight Decay & 1e-5 \\
    & atol & 1e-5 \\
    & rtol & 1e-5 \\
    \bottomrule
  \end{tabular}
\end{table}

\section{Handling non-contiguous time intervals}
\label{appendix:non_contiguous}
Although CT-OT Flow assumes contiguous intervals, it can be applied to non-contiguous intervals as in the Arch dataset (\S\ref{sec:artificial}).
Accuracy on non-contiguous data depends on whether the CT-OT Flow assumption, (i.e., subsets with similar true times have smaller Wasserstein distance), remains valid.

Figure~\ref{fig:appendix_non_contiguous} shows the estimated high-resolution time labels after removing the middle time-interval data from each synthetic dataset in \S\ref{sec:artificial}.
In the figure, gray points indicate excluded samples, while the color of each remaining point represents its estimated high-resolution time label.
As the figure shows, excluding these data points alters the results for the Spiral dataset, whereas it does not affect the Y-shaped and Arch datasets.
In the Spiral dataset with non-contiguous intervals, the central and outer subsets are close in Wasserstein distance despite being far apart in time.
As a result, the assumption is violated, and the estimated high-resolution time labels deviate from the ground truth.
By contrast, for the Y-shaped and Arch datasets, the assumption holds even with non-contiguous intervals, enabling accurate high-resolution time label estimation.

\clearpage

\begin{figure}
  \centering
  \begin{subfigure}[b]{\textwidth}
    \centering
    \includegraphics[width=0.7\textwidth]{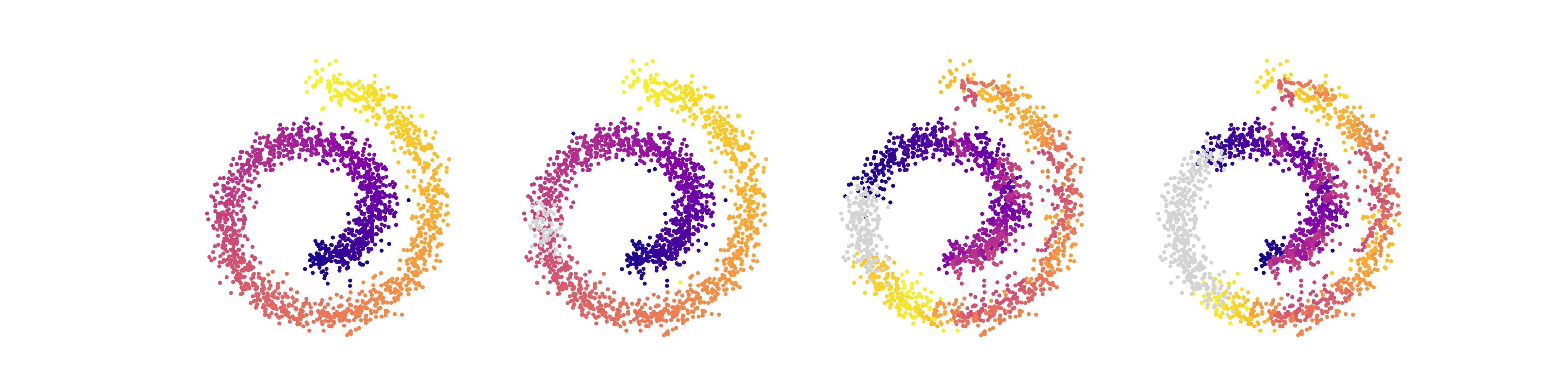}
    \caption{Spiral}
  \end{subfigure}
  \begin{subfigure}[b]{\textwidth}
    \centering
    \includegraphics[width=0.7\textwidth]{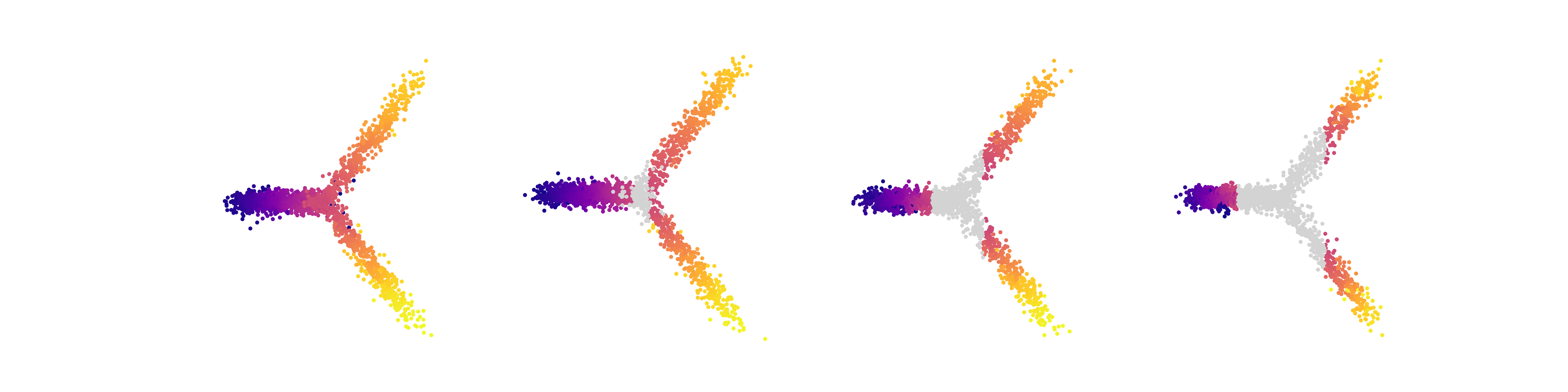}
    \caption{Y-shaped}
  \end{subfigure}
  \begin{subfigure}[b]{\textwidth}
    \centering
    \includegraphics[width=0.7\textwidth]{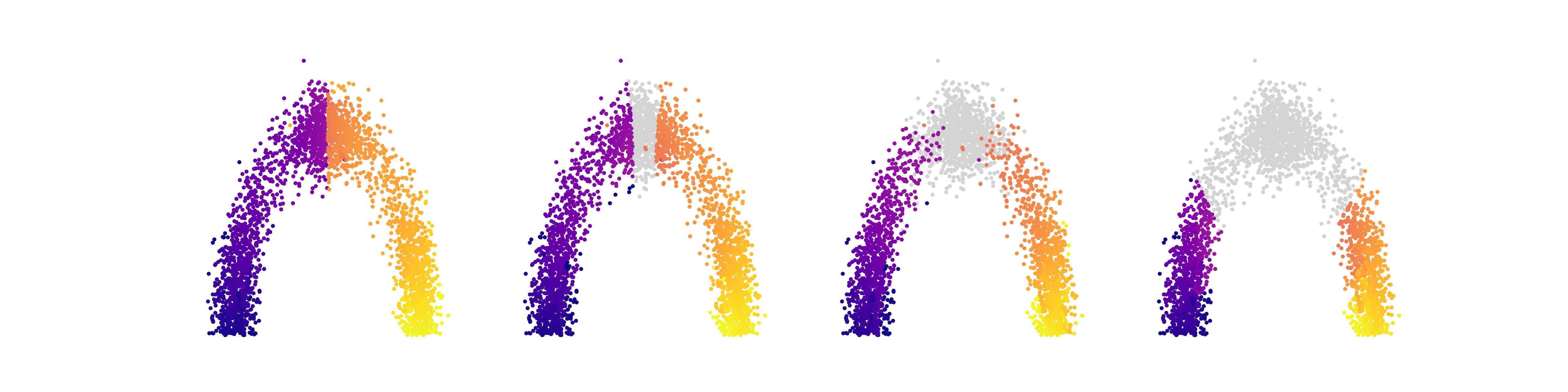}
    \caption{Arch}
  \end{subfigure}
  \caption{Estimated high-resolution time labels for non-contiguous time intervals.}
  \label{fig:appendix_non_contiguous}
\end{figure}

\section{\changed{Failure case for boundary extraction}}
\label{appendix:failure_case_for_boundary_extraction}
We provide an additional synthetic example to illustrate a failure mode of Step 1 in CT-OT Flow.
As discussed in App.~\ref{appendix:non_contiguous}, accurate high-resolution time label estimation depends on the assumption that subsets with similar true times have smaller Wasserstein distance.
When temporally distant states become geometrically close, this assumption can be violated, and the boundary subsets extracted by POT may no longer correspond to the true temporal boundary.

To examine this effect, we constructed a near-return example based on a circular trajectory.
As shown in Fig.~\ref{fig:failure_scatter}, the data consist of two contiguous snapshots sampled from the trajectory, while the end and the beginning of the trajectory are separated only by a small angular gap.
By varying this gap, we obtain settings in which, for smaller gaps, temporally discontinuous states near the end and beginning become geometrically close, thereby violating our assumption that subsets with similar true times have smaller Wasserstein distance.

We then applied CT-OT Flow and evaluated the inferred high-resolution time labels using Spearman's rank correlation with the ground-truth times over 10 runs.
Figure~\ref{fig:failure_spearman} shows that the rank correlation.
When the gap is large, CT-OT Flow recovers the temporal ordering almost perfectly.
However, as the gap decreases, temporally discontinuous states become increasingly difficult to distinguish based on geometry alone, and the inferred ordering deteriorates.

Representative examples are shown in Fig.~\ref{fig:failure_scatter}.
For a gap of 30$^\circ$, the inferred labels are consistent with the true temporal progression.
For 15$^\circ$, mild distortions appear near the boundary region.
For 0$^\circ$, the inferred labels are clearly misaligned, indicating that POT-based extraction can select a distributionally similar subgroup rather than the true temporal boundary.
These results clarify that CT-OT Flow is reliable when local temporal neighborhoods are identifiable from the geometry of the data, but may fail when temporally discontinuous states become geometrically similar.

\begin{figure}[t]
  \centering
  \begin{minipage}[t]{0.66\textwidth}
    \centering
    \includegraphics[width=\linewidth]{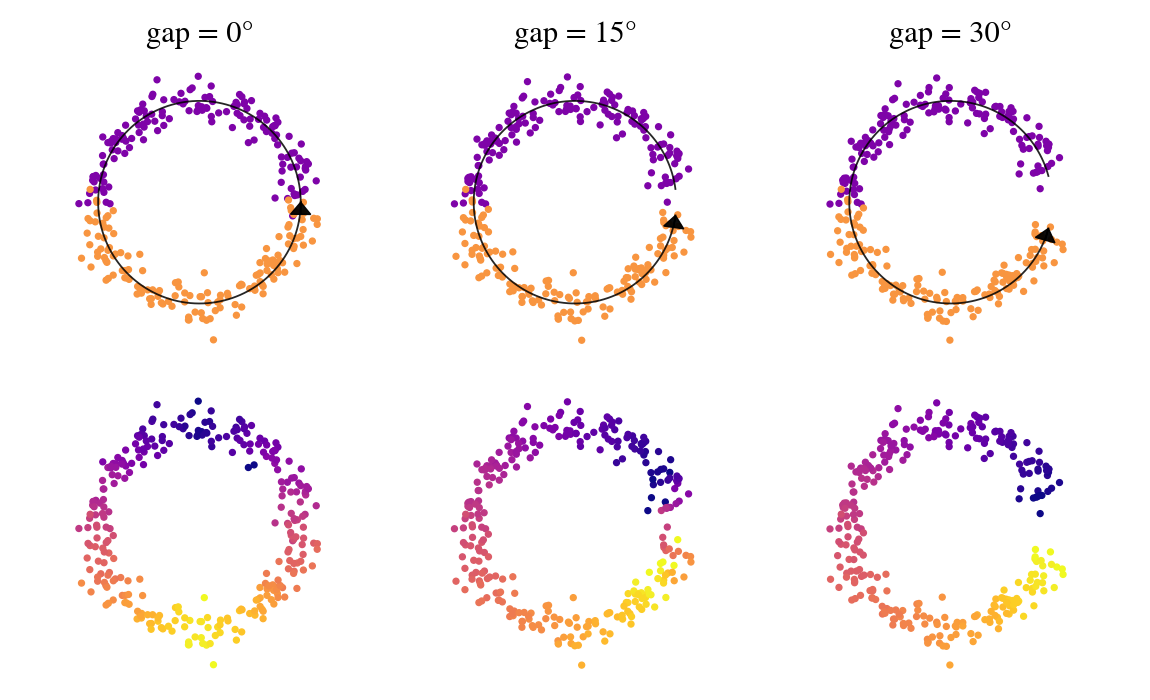}\par\vspace{4pt}
    \refstepcounter{figure}\label{fig:failure_scatter}
    \parbox[t]{0.8\linewidth}{\small\textbf{Figure \thefigure.} Input snapshots (top) and inferred time labels~(bottom) for gaps of 0$^\circ$, 15$^\circ$, and 30$^\circ$. Black lines denote the true trajectories.}
  \end{minipage}
  \hfill
  \begin{minipage}[t]{0.33\textwidth}
    \centering
    \includegraphics[width=\linewidth]{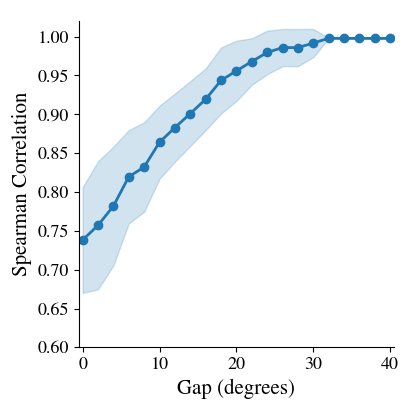}\par\vspace{4pt}
    \refstepcounter{figure}\label{fig:failure_spearman}
    \parbox[t]{\linewidth}{\small\textbf{Figure \thefigure.} Spearman correlation between the inferred high-resolution time labels and the ground-truth times in the near-return example, averaged over 10 runs.}
  \end{minipage}
\end{figure}

\section{Datasets}
\label{appendix:data_generation}
\subsection{Synthetic datasets}
We evaluate the ODE/SDE methods on the following three synthetic datasets:
\begin{enumerate}
  \item Spiral:  At each time $t$, the data distribution $p_t^*(\boldsymbol{x})$ is defined as a Gaussian with covariance $ma I$, where $I$ is the identity matrix and $ma = 0.1$. The trajectory $X_j^*$ consists of the means of these Gaussians. The time intervals are $[0,1]$ and $[1,2]$.
  \item Y-shaped: The distribution $p_t^*(\boldsymbol{x})$ is a mixture of two Gaussians. The trajectory $X_j^*$ follows the paths traced by the means of these two components. The time intervals are $[0,1]$ and $[1,2]$.
  \item Arch: We generate data following the procedure in~\cite{kapusniak2024metric}. The trajectory $X_j^*$ is chosen as an arc passing through the center of the arch. The time intervals are $[0,1]$ and $[2,3]$.
\end{enumerate}
The datasets are two-dimensional.
In all datasets, we add i.i.d. Gaussian noise $\sim\mathcal{N}(0,0.1)$ to each observation time.
As a result, some data points near the boundary of a time interval may fall into a different interval than their true time.
We set the number of samples per interval to 1000.

\subsection{scRNA-seq datasets}
We evaluated the ODE/SDE methods on two scRNA-seq datasets: Bifurcation and EB.
The Bifurcation dataset~\cite{Bargaje2017cell,sha2024reconstructing} consists of scRNA-seq measurements projected into four dimensions using principal component analysis (PCA).
We grouped the first three labels (0 d, 1 d, 1.5 d) as $t=0$, and then next three labels (2 d, 2.5 d, 3 d) as $t=1$.
The EB dataset~\cite{moon2018embryoid} consists of preprocessed scRNA-seq measurements projected into two dimensions using the nonlinear dimensionality reduction method PHATE~\cite{moon2019visualizing}.
We grouped the first two labels (Day 00-03, Day 06-09) as $t=0$, and then next three labels (Day 12-15, Day 18-21, Day 24-27) as $t=1$.
For both datasets, we randomly split each timestamp’s data into 70\% training and 30\% testing sets, and evaluate models using the hold-out test data.
These datasets are publicly available at \url{https://data.mendeley.com/datasets/v6n743h5ng/1} (Bifurcation; CC BY 4.0) and \url{https://github.com/yutongo/TIGON} (EB; MIT License).

\subsection{Typhoon trajectory dataset}
\label{appendix:typhoon_dataset}
This section details the experiments using meteorological data in \S\ref{sec:typhoon}.
We train a model using position data (longitude and latitude) sampled from typhoon tracks in Southeast Asia between 1951 and 2024.
The raw data were obtained from \url{https://www.jma.go.jp/jma/jma-eng/jma-center/rsmc-hp-pub-eg/besttrack.html} (accessed on April 24, 2025) under the license compatible with CC BY 4.0; \url{https://www.jma.go.jp/jma/en/copyright.html}.
The dataset records typhoon positions at 6-hour intervals from their onset.
Before training, the longitude and latitude coordinates were normalized to have zero mean and unit variance.
We used the Euclidean distance between points as the ground metric.

\section{Additional results}
\label{appendix:additional_results}

\subsection{Estimated trajectories}
Figure~\ref{fig:appendix_trajectory} provides additional visualizations of the estimated trajectories for the Y-shaped, Arch, Bifurcation, and EB datasets.
In each figure, the black lines indicate either the ground-truth or the estimated trajectories, depending on the method.
For CT-OT Flow, the color of each point represents its estimated high-resolution time label, illustrating the model's ability to resolve temporal ordering from coarse or noisy time labels.
In the Y-shaped and Arch datasets (Figs.~\ref{fig:appendix_trajectory}(a) and~(b)), CT-OT Flow successfully captures the geometry of the dynamics, even in the presence of temporal uncertainty.

\begin{figure}
  \centering
  \begin{subfigure}[b]{\textwidth}
    \centering
    \includegraphics[width=0.9\textwidth]{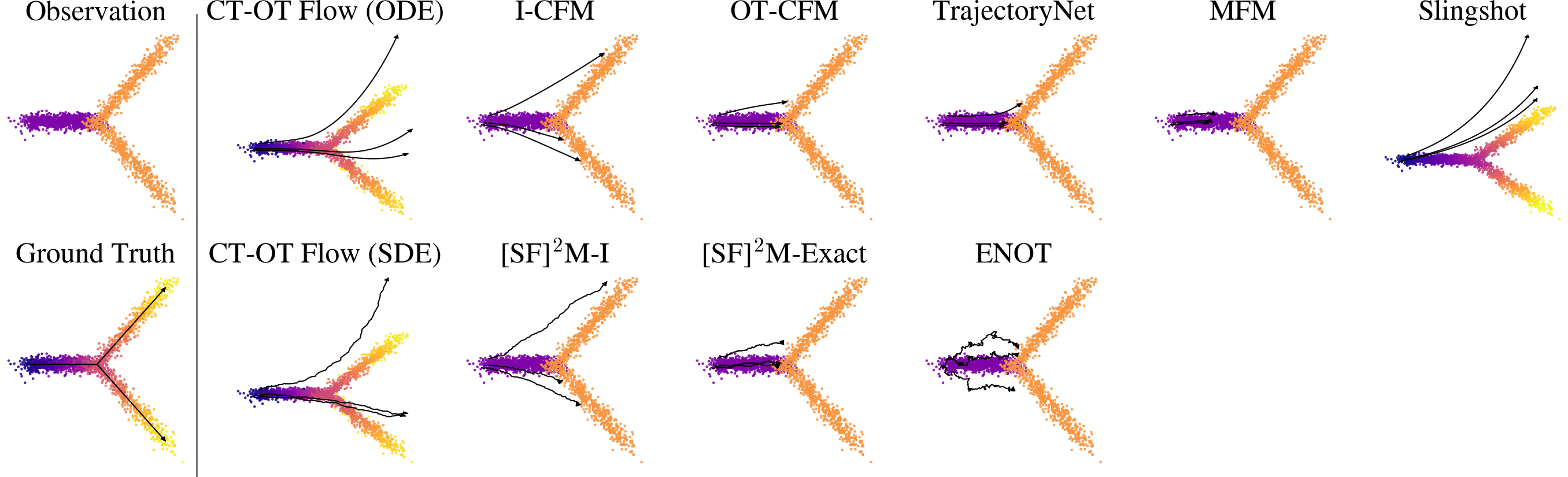}
    \caption{Y-shaped}
  \end{subfigure}
  \begin{subfigure}[b]{\textwidth}
    \centering
    \includegraphics[width=0.9\textwidth]{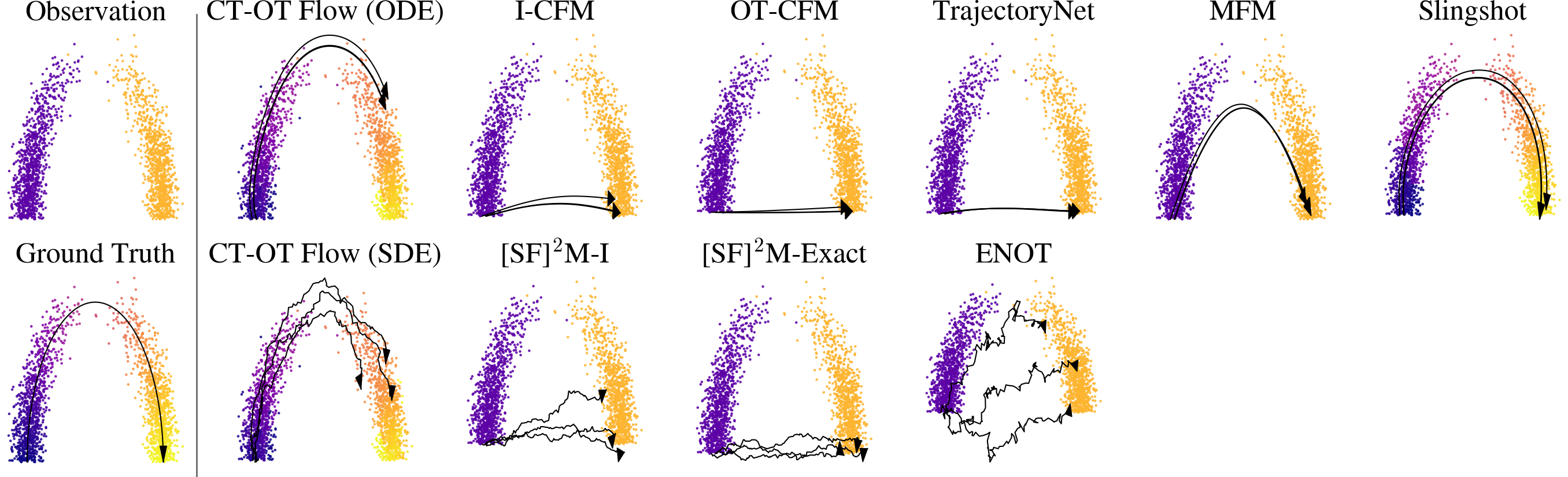}
    \caption{Arch}
  \end{subfigure}
  \caption{Estimated trajectories. The black lines indicate the true or estimated trajectories, while the color of each point in CT-OT Flow and Slingshot indicates its estimated high-resolution time label and pseudotime label, respectively.    \label{fig:appendix_trajectory}
  }
\end{figure}

\subsection{\changed{Temporal label estimation accuracy}}
\label{app:temporal-label-synthetic}

In this section, we evaluate the accuracy of point-wise temporal label estimation on the synthetic datasets.
For CT-OT Flow and Slingshot, we use the high-resolution time labels and pseudotime estimated by the respective models.
For the remaining methods that do not directly provide point-wise temporal labels, we infer them from the learned trajectories in a post-hoc manner.
Let $\hat X^{(i)}= [\hat{\boldsymbol x}_0^{(i)}, \dots, \hat{\boldsymbol x}_{\hat T}^{(i)}]$ denote the discretized trajectory, where $\hat X_t = \{\hat{\boldsymbol x}_t^{(i)}\}_{i=1}^{\hat N}$ is the set of simulated particles at time $t$.
For an observed sample $\boldsymbol x^{(j)}$, we define its time-dependent log-density score at time $t$ by a Gaussian kernel estimator,
\begin{align}
  \log \rho_t(\boldsymbol x^{(j)})
  =
  \log \left(
    \frac{1}{\hat N}
    \sum_{i=1}^{\hat N}
    \exp\!\left(
      -\frac{\|\boldsymbol x^{(j)} - \hat{\boldsymbol x}_t^{(i)}\|^2}{2h^2}
    \right)
  \right),
\end{align}
where $h>0$ is the kernel bandwidth.
We then convert these scores into a posterior distribution over discretized time points,
\begin{align}
  p(t \mid \boldsymbol x^{(j)})
  =
  \frac{\exp(\log \rho_t(\boldsymbol x^{(j)}))}
  {\sum_{t'} \exp(\log \rho_{t'}(\boldsymbol x^{(j)}))}.
\end{align}
Finally, the temporal label of $\boldsymbol x^{(j)}$ is defined as the posterior mean, $\hat{t}^{(j)} = \sum_t t \, p(t \mid \boldsymbol x^{(j)})$.
In our implementation, the bandwidth $h$ is chosen automatically from the trajectory and observation sets.
Specifically, we flatten all simulated trajectory points across time, compute for each observed sample its minimum Euclidean distance to the flattened trajectory set, and set $h$ to the median of these minimum distances.

Table~\ref{tab:artificial_time} summarizes Spearman and Pearson correlations between estimated time and ground-truth time.
Overall, CT-OT Flow yields the most stable and consistently accurate temporal estimates across datasets.
The difference is particularly clear on Spiral, where both the ODE and SDE variants of CT-OT Flow achieve nearly perfect agreement with ground-truth time, while several competing baselines exhibit weak correlations.
On Y-shaped and Arch, some baselines also attain high correlations, indicating that coarse temporal ordering can be recovered on simpler geometries.

\begin{table}[t]
  \centering
  \caption{Pseudotime correlations (mean$\pm$std) on the synthetic datasets. Bold: best within ODE/SDE groups.}
  \label{tab:artificial_time}
  \hspace{-5mm}
  \begin{tabular}{ccccccc}
    \toprule
    Dataset$\rightarrow$ & \multicolumn{2}{c}{Spiral} & \multicolumn{2}{c}{Y-shaped} & \multicolumn{2}{c}{Arch} \\
    \cmidrule(lr){2-3}
    \cmidrule(lr){4-5}
    \cmidrule(lr){6-7}
    Method$\downarrow$ Metric$\rightarrow$ & Spearman & Pearson & Spearman & Pearson & Spearman & Pearson \\
    \toprule
    CT-OT Flow (ODE) & \textbf{1.00}\std{0.00} & \textbf{1.00}\std{0.00} & 0.98\std{0.00} & 0.98\std{0.00} & \textbf{0.99}\std{0.00} & 0.99\std{0.00} \\
    I-CFM~\cite{tong2024improving} & 0.70\std{0.02} & 0.80\std{0.01} & 0.98\std{0.00} & 0.98\std{0.00} & 0.83\std{0.01} & 0.97\std{0.00} \\
    OT-CFM~\cite{tong2024improving} & 0.81\std{0.00} & 0.88\std{0.00} & \textbf{0.99}\std{0.00} & \textbf{0.99}\std{0.00} & 0.89\std{0.00} & 0.98\std{0.00} \\
    TrajectoryNet~\cite{tong2020trajectorynet} & 0.83\std{0.04} & 0.89\std{0.02} & 0.97\std{0.01} & 0.97\std{0.01} & 0.85\std{0.03} & 0.98\std{0.00} \\
    MFM~\cite{kapusniak2024metric} & 0.82\std{0.03} & 0.87\std{0.02} & \textbf{0.99}\std{0.00} & 0.98\std{0.00} & 0.97\std{0.00} & 0.99\std{0.00} \\
    Slingshot~\cite{street2018slingshot} & 0.61\std{0.00} & 0.64\std{0.00} & \textbf{0.99}\std{0.00} & \textbf{0.99}\std{0.00} & \textbf{0.99}\std{0.00} & \textbf{1.00}\std{0.00} \\
    \midrule
    CT-OT Flow (SDE) & \textbf{1.00}\std{0.00} & \textbf{1.00}\std{0.00} & 0.98\std{0.00} & 0.98\std{0.00} & \textbf{0.99}\std{0.00} & \textbf{0.99}\std{0.00} \\
    $\text{[SF]}^2\text{M}$-I~\cite{tong2024simulation} & 0.70\std{0.02} & 0.80\std{0.01} & 0.98\std{0.00} & 0.98\std{0.00} & 0.83\std{0.01} & 0.97\std{0.00} \\
    $\text{[SF]}^2\text{M}$-Exact~\cite{tong2024simulation} & 0.81\std{0.01} & 0.87\std{0.00} & \textbf{0.99}\std{0.00} & \textbf{0.99}\std{0.00} & 0.89\std{0.01} & 0.98\std{0.00} \\
    ENOT~\cite{gushchin2024entropic} & 0.76\std{0.01} & 0.85\std{0.01} & \textbf{0.99}\std{0.00} & \textbf{0.99}\std{0.00} & 0.91\std{0.01} & 0.98\std{0.00} \\
    \bottomrule
  \end{tabular}
\end{table}

\subsection{\changed{Evaluation without observed initial states}}
\label{app:conditioned-no-initial}

We next consider a more challenging setting in which true initial states are unavailable even at test time.
To better reflect this scenario, we evaluate two initial-state strategies.
In the first, we use the first observed snapshot itself as the initial state and simulate trajectories from that population.
In the second, we estimate the initial state rather than observing it directly.
As in App.~\ref{app:temporal-label-synthetic}, we first estimate point-wise time labels.
Based on these estimates, we select the earliest $10\%$ of observed points according to the estimated time values and use them as surrogate initial points.

The results are summarized in Tables~\ref{tab:artificial_unconditioned} and~\ref{tab:artificial_conditioned}.
When the first observed snapshot is directly used as the initial state (Table~\ref{tab:artificial_unconditioned}), CT-OT Flow generally performs worse than in the setting with true observed initial states (Table~\ref{tab:artificial}), reflecting the fact that the first snapshot does not coincide with the true starting population.

When the initial state is instead estimated from inferred time labels (Table~\ref{tab:artificial_conditioned}), CT-OT Flow improves markedly, and its performance becomes close to that under true initialization.
On Spiral, CT-OT Flow achieves the best performance in both DTW and Wasserstein distance for both ODE and SDE models.
Overall, these results suggest that even when true initial states are unavailable at test time, CT-OT Flow can still achieve strong dynamics estimation performance by estimating surrogate initial points based on the inferred high-resolution time labels.

\begin{table}[t]
  \centering
  \caption{Estimation errors (mean$\pm$std) on the synthetic datasets when the first observed snapshot is used as the initial state. Bold: best within ODE/SDE groups.}
  \label{tab:artificial_unconditioned}
  \hspace{-5mm}
  \begin{tabular}{ccccccc}
    \toprule
    Dataset$\rightarrow$ & \multicolumn{2}{c}{Spiral} & \multicolumn{2}{c}{Y-shaped} & \multicolumn{2}{c}{Arch} \\
    \cmidrule(lr){2-3}
    \cmidrule(lr){4-5}
    \cmidrule(lr){6-7}
    Method$\downarrow$ Metric$\rightarrow$ & $\mathcal L_\text{DTW}$ & $\mathcal L_\text{Wass}$ & $\mathcal L_\text{DTW}$ & $\mathcal L_\text{Wass}$ & $\mathcal L_\text{DTW}$ & $\mathcal L_\text{Wass}$ \\
    \toprule
    CT-OT Flow (ODE) & \textbf{20.75}\std{0.33} & 1.46\std{0.04} & 16.44\std{1.87} & 1.79\std{0.35} & \textbf{7.60}\std{0.64} & \textbf{0.33}\std{0.01} \\
    I-CFM~\cite{tong2024improving} & 49.25\std{0.17} & \textbf{1.02}\std{0.00} & 19.31\std{0.16} & \textbf{0.59}\std{0.00} & 20.82\std{0.15} & 0.49\std{0.00} \\
    OT-CFM~\cite{tong2024improving} & 52.99\std{0.03} & 1.09\std{0.00} & 20.48\std{0.21} & 0.65\std{0.00} & 21.32\std{0.04} & 0.49\std{0.00} \\
    TrajectoryNet~\cite{tong2020trajectorynet} & 53.19\std{0.21} & 1.09\std{0.01} & 20.24\std{0.42} & 0.62\std{0.02} & 23.45\std{2.27} & 0.54\std{0.06} \\
    MFM~\cite{kapusniak2024metric} & 53.02\std{0.34} & 1.09\std{0.01} & 20.86\std{0.43} & 0.65\std{0.01} & 10.86\std{0.55} & 0.35\std{0.01} \\
    Slingshot~\cite{street2018slingshot} & 27.68\std{0.52} & 1.47\std{0.04} & \textbf{15.49}\std{0.97} & 1.26\std{0.23} & 9.93\std{1.22} & 0.44\std{0.03} \\
    \midrule
    CT-OT Flow (SDE) & \textbf{21.52}\std{0.33} & 1.46\std{0.04} & \textbf{16.64}\std{1.94} & 1.79\std{0.34} & \textbf{8.48}\std{0.57} & \textbf{0.34}\std{0.01} \\
    $\text{[SF]}^2\text{M}$-I~\cite{tong2024simulation} & 48.65\std{0.17} & \textbf{1.02}\std{0.00} & 18.99\std{0.23} & \textbf{0.59}\std{0.01} & 20.62\std{0.14} & 0.50\std{0.00} \\
    $\text{[SF]}^2\text{M}$-Exact~\cite{tong2024simulation} & 52.48\std{0.05} & 1.09\std{0.00} & 20.34\std{0.24} & 0.65\std{0.00} & 21.08\std{0.10} & 0.50\std{0.00} \\
    ENOT~\cite{gushchin2024entropic} & 50.08\std{0.09} & 1.09\std{0.00} & 18.76\std{0.18} & 0.64\std{0.01} & 18.81\std{0.13} & 0.50\std{0.00} \\
    \bottomrule
  \end{tabular}
\end{table}

\begin{table}[t]
  \centering
  \caption{Estimation errors (mean$\pm$std) on the synthetic datasets when the initial state is estimated from inferred time labels. Bold: best within ODE/SDE groups.}
  \label{tab:artificial_conditioned}
  \hspace{-5mm}
  \begin{tabular}{ccccccc}
    \toprule
    Dataset$\rightarrow$ & \multicolumn{2}{c}{Spiral} & \multicolumn{2}{c}{Y-shaped} & \multicolumn{2}{c}{Arch} \\
    \cmidrule(lr){2-3}
    \cmidrule(lr){4-5}
    \cmidrule(lr){6-7}
    Method$\downarrow$ Metric$\rightarrow$ & $\mathcal L_\text{DTW}$ & $\mathcal L_\text{Wass}$ & $\mathcal L_\text{DTW}$ & $\mathcal L_\text{Wass}$ & $\mathcal L_\text{DTW}$ & $\mathcal L_\text{Wass}$ \\
    \toprule
    CT-OT Flow (ODE) & \textbf{8.24}\std{0.88} & \textbf{0.39}\std{0.07} & 11.79\std{1.12} & 0.65\std{0.21} & 6.77\std{0.96} & \textbf{0.23}\std{0.03} \\
    I-CFM~\cite{tong2024improving} & 54.05\std{0.87} & 1.19\std{0.03} & 20.41\std{0.38} & \textbf{0.56}\std{0.01} & 21.26\std{0.20} & 0.51\std{0.00} \\
    OT-CFM~\cite{tong2024improving} & 49.52\std{0.38} & 1.03\std{0.01} & 27.41\std{0.19} & 0.80\std{0.00} & 23.37\std{0.09} & 0.52\std{0.00} \\
    TrajectoryNet~\cite{tong2020trajectorynet} & 48.11\std{1.65} & 1.06\std{0.04} & 43.47\std{3.75} & 1.14\std{0.08} & 51.19\std{5.83} & 1.10\std{0.09} \\
    MFM~\cite{kapusniak2024metric} & 49.21\std{1.00} & 1.03\std{0.02} & 28.19\std{0.56} & 0.74\std{0.03} & 12.21\std{1.23} & 0.31\std{0.01} \\
    Slingshot~\cite{street2018slingshot} & 19.27\std{0.75} & 1.02\std{0.04} & \textbf{9.48}\std{0.83} & 0.66\std{0.15} & \textbf{4.41}\std{0.76} & 0.25\std{0.03} \\
    \midrule
    CT-OT Flow (SDE) & \textbf{10.00}\std{0.79} & \textbf{0.42}\std{0.06} & \textbf{12.11}\std{1.10} & 0.67\std{0.20} & \textbf{7.89}\std{0.88} & \textbf{0.25}\std{0.03} \\
    $\text{[SF]}^2\text{M}$-I~\cite{tong2024simulation} & 52.97\std{1.14} & 1.17\std{0.03} & 20.53\std{0.66} & \textbf{0.57}\std{0.01} & 21.18\std{0.30} & 0.51\std{0.00} \\
    $\text{[SF]}^2\text{M}$-Exact~\cite{tong2024simulation} & 48.97\std{0.79} & 1.03\std{0.01} & 27.20\std{0.23} & 0.80\std{0.01} & 23.05\std{0.34} & 0.53\std{0.00} \\
    ENOT~\cite{gushchin2024entropic} & 50.50\std{0.74} & 1.13\std{0.02} & 24.47\std{0.45} & 0.74\std{0.01} & 20.82\std{0.61} & 0.51\std{0.01} \\
    \bottomrule
  \end{tabular}
\end{table}

\subsection{\changed{Additional visualizations for hyperparameter sensitivity}}
\label{appendix:hyperparameter_details}
This appendix supplements the hyperparameter sensitivity analysis in \S\ref{sec:hyperparameters} with additional qualitative visualizations.
While the main text focuses on the practical conclusion that both $K$ and $\gamma$ should be chosen in an intermediate range, the figures in this section illustrate how these hyperparameters affect the inferred labels and the smoothed intermediate distributions.

\paragraph{Subdivision factor \texorpdfstring{$K$}{K}.}
Figure~\ref{fig:diff_K} visualizes the inferred high-resolution time labels for different values of $K$ on the Spiral, Y-shaped, and Arch datasets. When $K$ is small, each interval is partitioned only coarsely, so many neighboring samples share the same inferred label and the recovered temporal ordering remains rough. As $K$ increases, the labels become progressively finer and align more closely with the underlying trajectories. However, when $K$ is too large, each subset contains too few samples, making the inferred labels more sensitive to noise and local sampling fluctuations. This visual pattern matches the quantitative results in Figs.~\ref{fig:spearman} and~\ref{fig:ablation_diff_k}, where moderate values of $K$ achieve the best balance between temporal resolution and robustness.
For completeness, Table~\ref{tab:flow_matching_k_ablation} further summarizes the results for the default choice $K=100$ and its $\pm 50\%$ variants ($K=50,150$).

\begin{figure}
  \centering
  \begin{subfigure}[b]{\textwidth}
    \centering
    \includegraphics[width=0.8\textwidth]{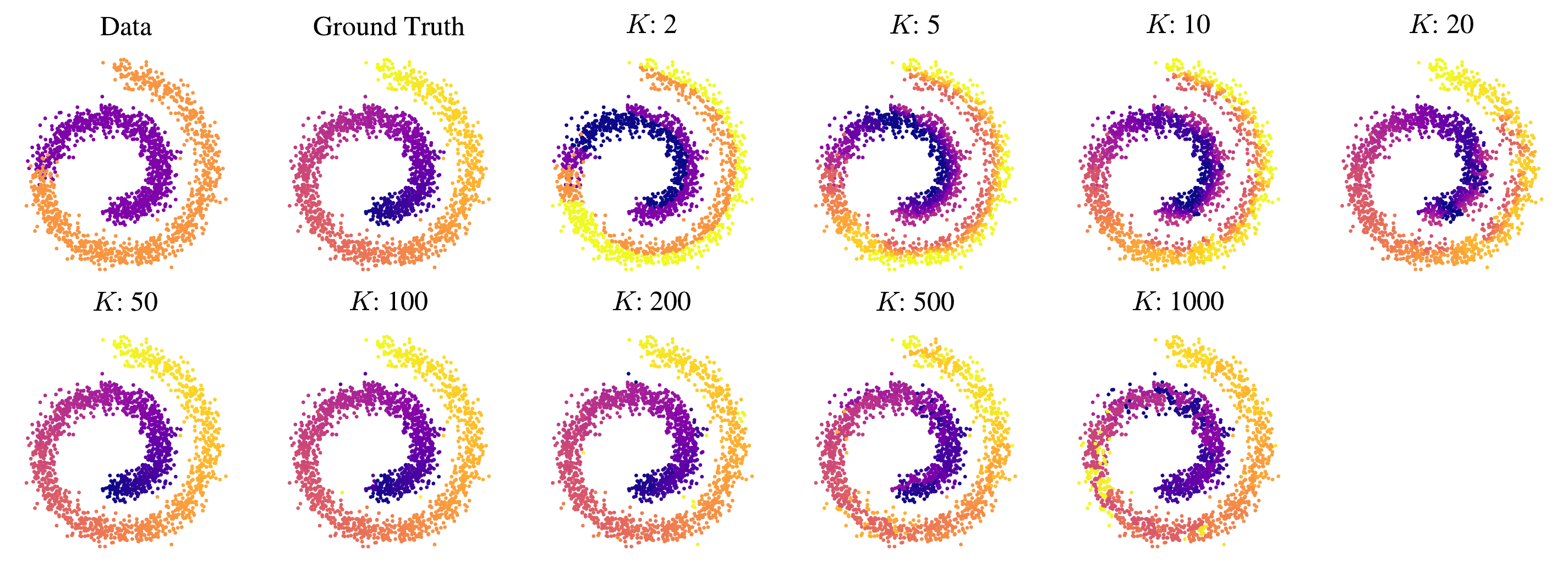}
    \caption{Spiral}
  \end{subfigure}
  \begin{subfigure}[b]{\textwidth}
    \centering
    \includegraphics[width=0.8\textwidth]{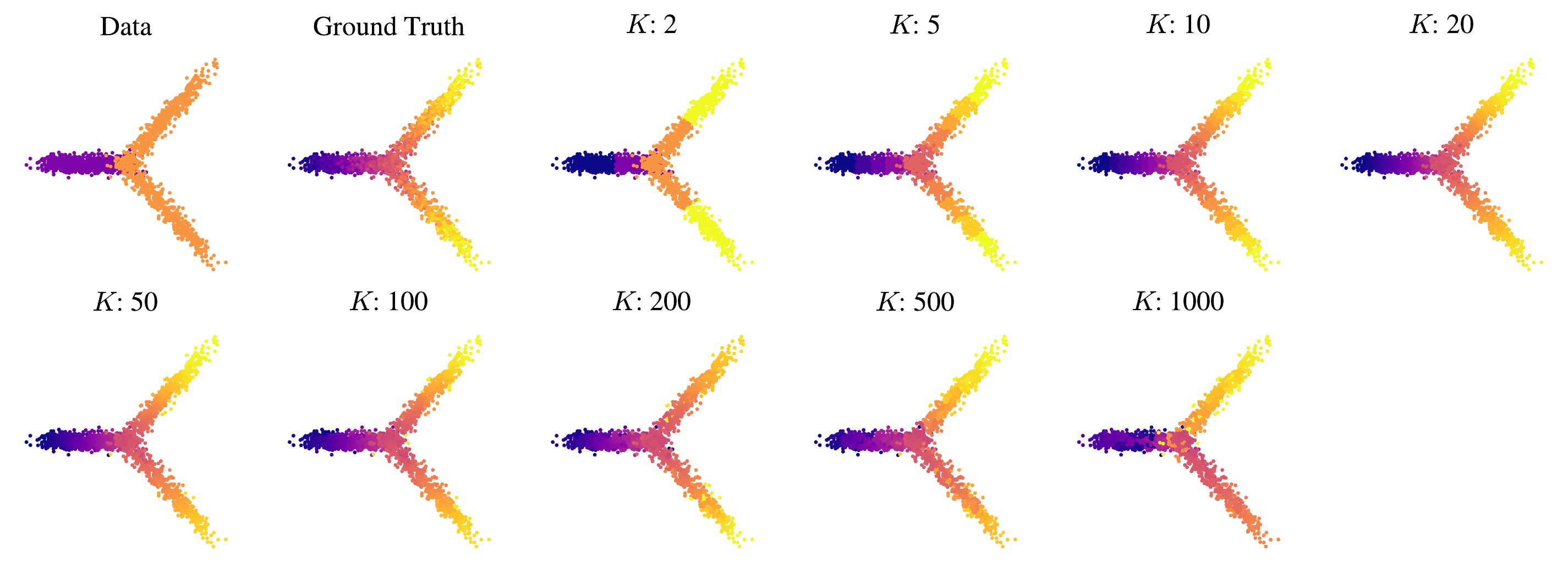}
    \caption{Y-shaped}
  \end{subfigure}
  \begin{subfigure}[b]{\textwidth}
    \centering
    \includegraphics[width=0.8\textwidth]{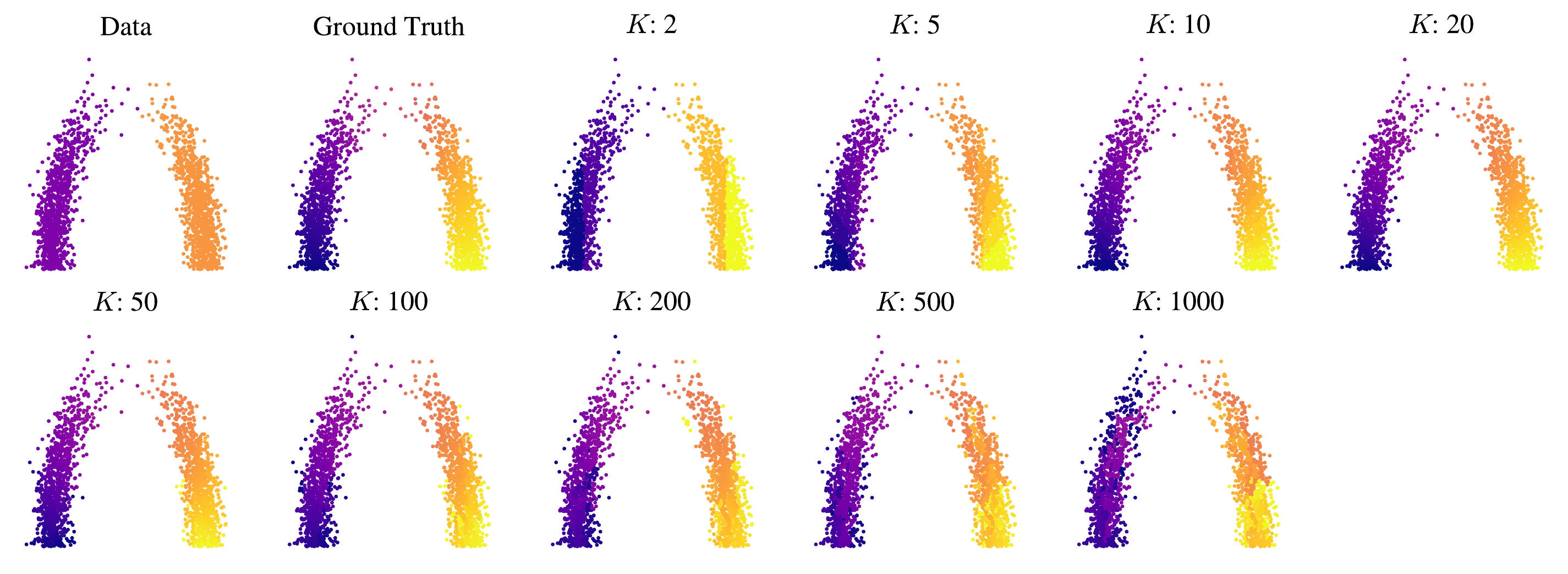}
    \caption{Arch}
  \end{subfigure}
  \caption{Estimated high-resolution time labels with varying $K$.    \label{fig:diff_K}
  }
\end{figure}

\begin{table}[t]
  \centering
  \caption{Estimation errors (mean$\pm$std) on the synthetic datasets with varying $K$. Bold: best within ODE/SDE groups across $K \in \{50, 100, 150\}$. }
  \label{tab:flow_matching_k_ablation}
  \hspace{-5mm}
  \begin{tabular}{cccccccc}
    \toprule
    Dataset$\rightarrow$ & & \multicolumn{2}{c}{Spiral} & \multicolumn{2}{c}{Y-shaped} & \multicolumn{2}{c}{Arch} \\
    \cmidrule(lr){3-4}
    \cmidrule(lr){5-6}
    \cmidrule(lr){7-8}
    Method$\downarrow$ & $K$ & $\mathcal L_\text{DTW}$ & $\mathcal L_\text{Wass}$ & $\mathcal L_\text{DTW}$ & $\mathcal L_\text{Wass}$ & $\mathcal L_\text{DTW}$ & $\mathcal L_\text{Wass}$ \\
    \toprule
    CT-OT Flow (ODE) & 50 & 36.39\std{1.21} & 1.25\std{0.09} & 12.11\std{1.78} & 0.64\std{0.26} & 10.43\std{2.57} & 0.33\std{0.06} \\
    & 100 & 9.63\std{1.25} & 0.31\std{0.05} & \textbf{10.80}\std{0.93} & \textbf{0.43}\std{0.11} & \textbf{6.81}\std{1.28} & \textbf{0.23}\std{0.05} \\
    & 150 & \textbf{9.08}\std{0.82} & \textbf{0.30}\std{0.03} & 12.89\std{1.76} & 0.69\std{0.22} & 8.16\std{1.59} & 0.26\std{0.05} \\
    \midrule
    CT-OT Flow (SDE) & 50 & 36.05\std{1.09} & 1.25\std{0.09} & 12.81\std{1.77} & 0.67\std{0.25} & 11.49\std{2.39} & 0.35\std{0.06} \\
    & 100 &  11.52\std{1.12} & 0.36\std{0.04} & \textbf{11.51}\std{1.04} & \textbf{0.47}\std{0.11} & \textbf{7.94}\std{1.07} & \textbf{0.25}\std{0.04} \\
    & 150 & \textbf{10.63}\std{0.80} & \textbf{0.34}\std{0.03} & 13.86\std{1.83} & 0.75\std{0.21} & 8.99\std{1.36} & 0.27\std{0.04} \\
    \bottomrule
  \end{tabular}
\end{table}

\paragraph{Kernel smoothing parameter \texorpdfstring{$\gamma$}{gamma}.}
Figure~\ref{fig:diff_gamma} visualizes the time-smoothed empirical distribution $\tilde p_t(\boldsymbol x)$ at $t=0.5$ for different values of $\gamma$. When $\gamma$ is very small, the temporal kernel places substantial weight only on samples whose inferred times are extremely close to $t$, so the estimated distribution remains sharp but can inherit errors from imperfect time-label inference.
In contrast, when $\gamma$ is large, the kernel averages over a broader temporal range, producing overly smooth distributions that blur phase-specific structure and weaken the benefit of Step~1. Thus, $\gamma$ controls a bias--variance trade-off: smaller values preserve temporal detail but are more sensitive to labeling errors, whereas larger values improve robustness at the cost of excessive smoothing. This interpretation is consistent with Fig.~\ref{fig:ablation_diff_gamma}, where intermediate values of $\gamma$ give the lowest prediction error.

\begin{figure}
  \centering
  \begin{subfigure}[b]{\textwidth}
    \centering
    \includegraphics[width=0.9\textwidth]{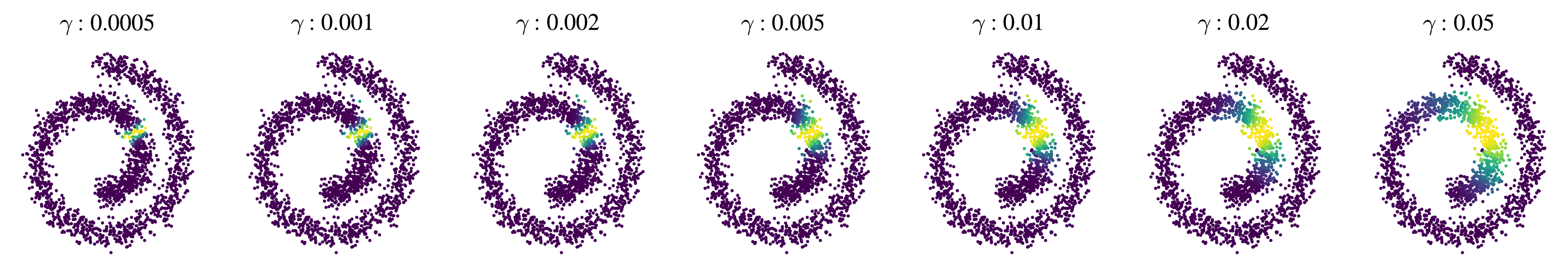}
    \caption{Spiral}
  \end{subfigure}
  \begin{subfigure}[b]{\textwidth}
    \centering
    \includegraphics[width=0.9\textwidth]{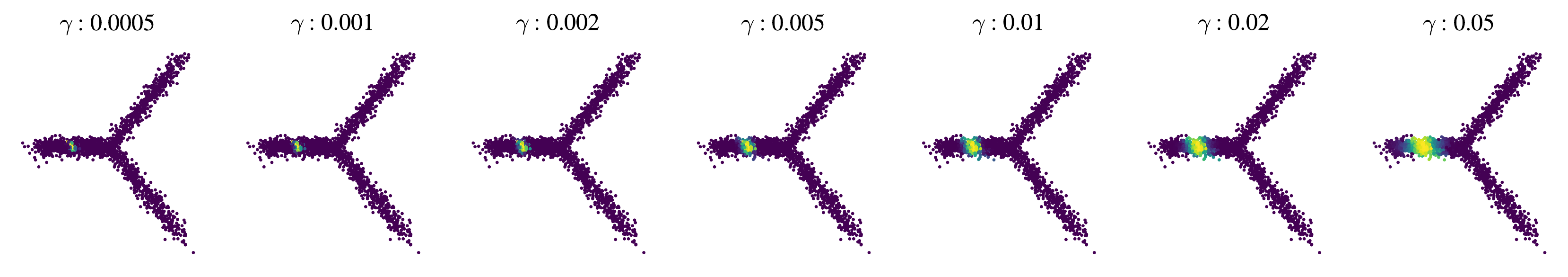}
    \caption{Y-shaped}
  \end{subfigure}
  \begin{subfigure}[b]{\textwidth}
    \centering
    \includegraphics[width=0.9\textwidth]{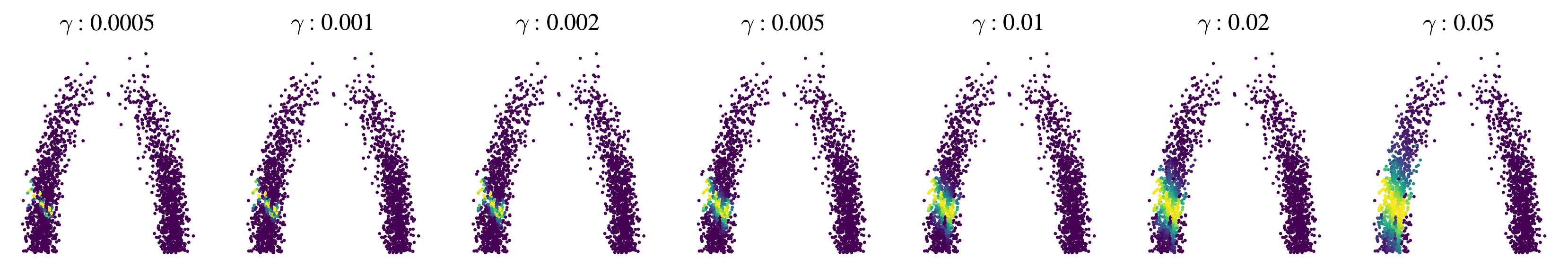}
    \caption{Arch}
  \end{subfigure}
  \caption{Estimated $\tilde p_t(\boldsymbol x)$, where $t=0.5$ with varying $\gamma$.    \label{fig:diff_gamma}
  }
\end{figure}

\subsection{\changed{Approximation quality of the POT-based boundary selection}}
\label{app:milp-pot}

The boundary-selection step in our method is implemented via a POT relaxation (Eq.~\eqref{eq:pw}), whereas the corresponding exact formulation can be written as an MILP (Eq.~\eqref{eq:MILP}).
Because the exact MILP is not scalable, we use the POT-based approximation in practice.
Here, we quantify the approximation error introduced by the relaxation in Eq.~\eqref{eq:pw} on small-scale synthetic problems for which the exact MILP  is solvable.

We evaluated the spiral, Y-shaped, and arch datasets with $N \in \{50,100,200,500\}$ samples per snapshot and subdivision factors $K \in \{5,10,20,40,100\}$ satisfying $K \le N/5$, using five random seeds for each setting.
For each instance, we compared the exact MILP solution with the subset obtained from the POT solution.

Across all tested settings, the POT-based subset selected exactly the same points as the exact MILP solution.
This suggests that the POT-based approximation is effectively exact in the benchmark regime considered here.
More generally, a discrepancy between MILP and POT can arise when several candidate points have identical or nearly identical transport costs, since in such cases the relaxed POT solution may distribute mass across multiple competing candidates.
Note that the effect of additionally introducing random mini-batches is discussed in \ref{appendix:random_subsets}.

\subsection{Evaluation on multiple snapshots}
\label{appendix:multiple_snapshots}

In the main text, most experiments were conducted on datasets grouped into two temporal snapshots.
Here, we further investigate the performance of CT-OT Flow when the number of snapshots increases.

\paragraph{Increased number of snapshots}
First, we considered the Spiral dataset divided into $4, 8$, and $16$ snapshots.
Models were trained using all snapshots, and performance was evaluated by comparing the simulated trajectories with the ground truth.
Table~\ref{tab:spiral_multik} reports the results.
CT-OT Flow consistently achieved lower errors than the baselines, indicating that it can scale to settings with more than two snapshots.

\begin{table}[t]
  \centering
  \caption{Estimation errors (mean$\pm$std) with varying number of snapshots.
  Bold: best within ODE/SDE groups.}
  \label{tab:spiral_multik}
  \begin{tabular}{ccccccc}
    \toprule
    Dataset$\rightarrow$                   & \multicolumn{2}{c}{4 snapshots} & \multicolumn{2}{c}{8 snapshots} & \multicolumn{2}{c}{16 snapshots}                                                                                            \\
    \cmidrule(lr){2-3}
    \cmidrule(lr){4-5}
    \cmidrule(lr){6-7}
    Method$\downarrow$ Metric$\rightarrow$ & $\mathcal L_\text{DTW}$         & $\mathcal L_\text{Wass}$        & $\mathcal L_\text{DTW}$          & $\mathcal L_\text{Wass}$    & $\mathcal L_\text{DTW}$      & $\mathcal L_\text{Wass}$    \\ \toprule
    CT-OT Flow (ODE)                       & \textbf{9.79}\std{0.90}     & \textbf{0.30}\std{0.04}     & \textbf{12.00}\std{1.19}     & \textbf{0.35}\std{0.03} & \textbf{10.36}\std{1.02} & \textbf{0.32}\std{0.05} \\
    I-CFM                                  & 38.71\std{3.27}                 & 1.02\std{0.08}                  & 33.30\std{3.14}                  & 1.01\std{0.08}              & 19.51\std{3.437}             & 0.49\std{0.09}              \\
    OT-CFM                                 & 41.96\std{1.50}                 & 1.10\std{0.05}                  & 31.36\std{2.54}                  & 0.87\std{0.06}              & 17.67\std{3.13}              & 0.44\std{0.09}              \\
    \midrule
    CT-OT Flow (SDE)                       & \textbf{11.41}\std{0.89}    & \textbf{0.35}\std{0.03}     & \textbf{12.93}\std{1.26}     & \textbf{0.38}\std{0.03} & \textbf{12.51}\std{1.37} & \textbf{0.37}\std{0.05} \\
    $\text{[SF]}^2$M-I                       & 39.45\std{3.01}                 & 1.03\std{0.08}                  & 30.45\std{2.05}                  & 1.04\std{0.10}              & 22.01\std{1.68}              & 0.68\std{0.06}              \\
    $\text{[SF]}^2$M-Exact                   & 41.97\std{1.41}                 & 1.13\std{0.05}                  & 28.31\std{1.58}                  & 0.87\std{0.06}              & 20.05\std{1.40}              & 0.60\std{0.07}              \\
    \bottomrule
  \end{tabular}
\end{table}

\paragraph{Even/odd train–test split}
Next, we evaluated a train–test split using seven snapshots of the Spiral dataset.
The even-indexed snapshots were used for training, while the odd-indexed snapshots were reserved for testing.
For each test snapshot, we extracted the simulated trajectories corresponding to its time range and compared them with the empirical distributions of the test data.
The final score was calculated as the average 1-Wasserstein distance over all test snapshots.
Table~\ref{tab:spiral_multisnapshot} shows that CT-OT Flow again achieved the lowest errors, demonstrating robustness even in partially observed (non-contiguous) scenarios.

\begin{table}[t]
  \centering
  \caption{Estimation errors (1-Wasserstein distance) on the Spiral dataset with seven snapshots (mean$\pm$ std).
    Models were trained on even-indexed snapshots and tested on odd-indexed snapshots.
  Bold: best within ODE/SDE groups and dimension.}
  \label{tab:spiral_multisnapshot}
  \begin{tabular}{lc}
    \toprule
    Method & $\mathcal L_\text{wass}$ \\
    \midrule
    CT-OT Flow (ODE)        & \textbf{0.38}\std{0.09} \\
    I-CFM                   & 0.87\std{0.06} \\ \midrule
    OT-CFM                  & 0.92\std{0.05} \\
    CT-OT Flow (SDE)        & \textbf{0.39}\std{0.09} \\
    $\text{[SF]}^2$M-I      & 0.85\std{0.07} \\
    $\text{[SF]}^2$M-Exact        & 0.92\std{0.05} \\
    \bottomrule
  \end{tabular}
\end{table}

\subsection{CT-OT Flow for high-dimensional data}
\label{appendix:high_dimensional}
We evaluated CT-OT Flow by applying PHATE~\cite{moon2019visualizing} to embed the EB dataset~\cite{moon2018embryoid} into $d\in\{2,4,8,16,32\}$ dimensions.
Table~\ref{tab:high-dim} summarizes the estimation errors $\mathcal L_\text{Wass}$ for each setting.
Our experiments reveal that when using a small kernel smoothing parameter ($\gamma=0.005$), CT-OT Flow’s performance degrades as the dimension increases.
We attribute this to the concentration of pairwise distances in high-dimensional spaces, which reduces the contrast required for accurate boundary-point extraction.
Conversely, increasing $\gamma$ (e.g., $\gamma=0.5$) yields robust $\mathcal L_{\mathrm{Wass}}$ across all $d$.
This is because a larger $\gamma$ enhances the robustness of high-resolution time label estimation against the reduced distance contrast (see also \S\ref{sec:hyperparameters}).
In practice, when applying CT-OT Flow to high-dimensional data, one may either increase the kernel smoothing parameter $\gamma$ or embed the data into a lower-dimensional space to effectively identify points near the interval boundaries.

\begin{table}
  \centering
  \caption{Estimation errors ($\mathcal L_{\text{Wass}}$) on the EB dataset with varying dimensions.
  Bold: best within ODE/SDE groups and dimension.}
  \label{tab:high-dim}
  \begin{tabular}{rccccc}
    \toprule
    Method$\downarrow$ Dimension$\rightarrow$               & 2                       & 4                       & 8                       & 16                      & 32                      \\
    \toprule
    CT-OT Flow (ODE) ($\gamma=0.005$)                       & 0.94\std{0.11}          & 1.13\std{0.06}          & 2.19\std{0.13}          & 6.41\std{0.88}          & 5.80\std{0.46}          \\
    CT-OT Flow (ODE) ($\gamma=0.05$)                        & 0.98\std{0.09}          & \textbf{1.07}\std{0.06} & \textbf{1.76}\std{0.08} & 4.21\std{0.14}  & 4.79\std{0.33}          \\
    CT-OT Flow (ODE) ($\gamma=0.5$)                         & \textbf{0.81}\std{0.05} & 1.16\std{0.03}          & 1.81\std{0.04}          & \textbf{2.75}\std{0.04}      & 4.34\std{0.04}          \\ \addlinespace
    I-CFM~\cite{tong2024improving}                          & 1.00\std{0.01}          & 1.32\std{0.01}          & 2.01\std{0.01}          & 2.86\std{0.01}          & 4.36\std{0.01}          \\
    OT-CFM~\cite{tong2024improving}                         & 1.06\std{0.01}          & 1.37\std{0.01}          & 2.01\std{0.01}          & 2.87\std{0.01}          & 4.32\std{0.01}          \\
    TrajectoryNet~\cite{tong2020trajectorynet}              & 1.10\std{0.04}          & 2.04\std{0.59} & 3.07\std{0.16}                   & 4.16\std{0.29}          &  5.41\std{0.27}           \\
    MFM~\cite{kapusniak2024metric}                          & 1.04\std{0.04}          & 1.34\std{0.03}          & 1.93\std{0.03}          & 2.80\std{0.03}          & \textbf{4.22}\std{0.02} \\
    Slingshot~\cite{street2018slingshot} & 2.74\std{0.20} & 1.90\std{0.08} & 3.53\std{0.24} & 20.89\std{3.34} & 9.61\std{1.58}\\
    \midrule
    CT-OT Flow (SDE) ($\gamma=0.005$)                       & 0.92\std{0.12}          & 1.11\std{0.07}          & 2.19\std{0.14}          & 6.39\std{0.85}          & 5.79\std{0.46}          \\
    CT-OT Flow (SDE) ($\gamma=0.05$)                        & 0.97\std{0.09}          & \textbf{1.04}\std{0.06} & \textbf{1.76}\std{0.08} & 4.25\std{0.16}          & 4.79\std{0.32}          \\
    CT-OT Flow (SDE) ($\gamma=0.5$)                         & \textbf{0.80}\std{0.05} & 1.14\std{0.03}          & 1.79\std{0.04}          & \textbf{2.72}\std{0.04} & 4.33\std{0.04}          \\ \addlinespace
    $\text{[SF]}^2\text{M}$-I~\cite{tong2024simulation}     & 0.99\std{0.01}          & 1.30\std{0.01}          & 2.00\std{0.01}          & 2.84\std{0.01}          & 4.34\std{0.01}          \\
    $\text{[SF]}^2\text{M}$-Exact~\cite{tong2024simulation} & 1.05\std{0.01}          & 1.36\std{0.01}          & 2.00\std{0.01}          & 2.86\std{0.01}          & \textbf{4.32}\std{0.01}          \\
    ENOT~\cite{gushchin2024entropic}                        & 1.03\std{0.02}          & 1.33\std{0.01}          & 1.99\std{0.08}          & 2.90\std{0.08}          & 4.53\std{0.25}          \\
    \bottomrule
  \end{tabular}
\end{table}

\subsection{\changed{Results on unaggregated EB dataset}}
\label{appendix:eb_unaggregated}
While the experiments in the main text aggregate contiguous snapshots to enable quantitative evaluation, scRNA-seq snapshots inherently represent distributions over time intervals because cell birth times cannot be synchronized.
We therefore also qualitatively evaluate CT-OT Flow on the original, unaggregated EB dataset.
As shown in Fig.~\ref{fig:unaggregated}, baseline methods (e.g., $\text{[SF]}^2$M) produce trajectories that depart from the data manifold, whereas CT-OT Flow follows paths aligned with the observed geometry.

\begin{figure}
  \centering
  \includegraphics[width=0.95\textwidth]{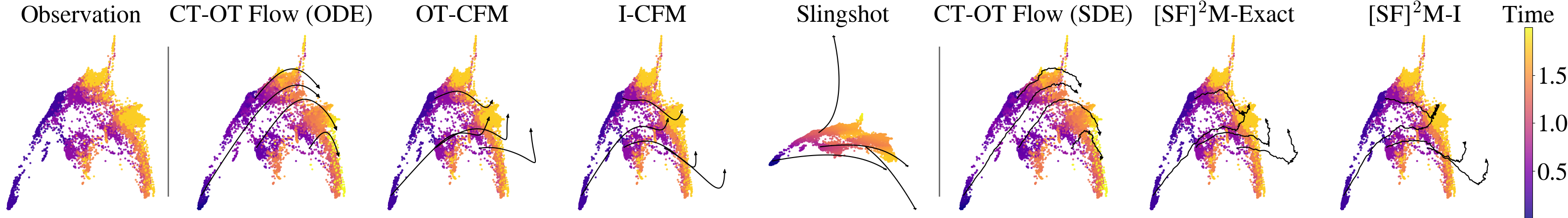}
  \caption{Estimated trajectories on unaggregated EB dataset. The black lines indicate the true or estimated trajectories, while the color of each point indicates its (high-resolution) time label. \label{fig:unaggregated}}
\end{figure}

\subsection{TrajectoryNet with stronger manifold regularization}
We additionally tested TrajectoryNet~\cite{tong2020trajectorynet} with stronger regularization toward manifold-conforming trajectories.
Specifically, we increased the weight of the top-$k$-NN regularization term (``top-$k$-reg''), which encourages trajectories to align with local manifold geometry.
Table~\ref{tab:trajectorynet_reg} summarizes the results on the Spiral, Y-shaped, and Arch datasets.
Across all settings, CT-OT Flow still achieved substantially lower errors.

\begin{table}[t]
  \centering
  \caption{Comparison with TrajectoryNet using different regularization strengths. }
  \label{tab:trajectorynet_reg}
  \begin{tabular}{ccccccc}
    \toprule
    Dataset$\rightarrow$                   & \multicolumn{2}{c}{Spiral} & \multicolumn{2}{c}{Y-shaped} & \multicolumn{2}{c}{Arch}                                                                                 \\
    \cmidrule(lr){2-3}
    \cmidrule(lr){4-5}
    \cmidrule(lr){6-7}
    Method$\downarrow$ Metric$\rightarrow$ & $\mathcal L_\text{DTW}$    & $\mathcal L_\text{Wass}$     & $\mathcal L_\text{DTW}$  & $\mathcal L_\text{Wass}$ & $\mathcal L_\text{DTW}$ & $\mathcal L_\text{Wass}$ \\ \toprule
    top-$k$-reg = 1                        & 51.68\std{0.12}            & 1.20\std{0.03}               & 20.74\std{2.08}          & 0.75\std{0.07}           & 34.29\std{1.27}         & 0.74\std{0.06}           \\
    top-$k$-reg = 100                      & 49.69\std{3.91}            & 1.17\std{0.10}               & 30.78\std{3.44}          & 0.78\std{0.07}           & 21.24\std{2.83}         & 0.52\std{0.07}           \\
    top-$k$-reg = 10000                    & 42.60\std{3.70}            & 0.99\std{0.04}               & 28.20\std{4.27}          & 0.77\std{0.10}           & 38.99\std{12.29}        & 0.84\std{0.22}           \\
    CT-OT Flow (ODE)                       & \textbf{9.64}\std{1.04}    & \textbf{0.31}\std{0.04}      & \textbf{10.79}\std{0.66} & \textbf{0.41}\std{0.10}  & \textbf{9.98}\std{2.88}          & \textbf{0.31}\std{0.08}           \\
    \bottomrule
  \end{tabular}
\end{table}

\clearpage

\bibliographystyle{cas-model2-names}
\bibliography{references}

@article{verma2024climode,
  title={Climode: Climate and weather forecasting with physics-informed neural odes},
  author={Verma, Yogesh and Heinonen, Markus and Garg, Vikas},
  journal={arXiv preprint arXiv:2404.10024},
  year={2024}
}

@inproceedings{choi2022graph,
  title={Graph neural controlled differential equations for traffic forecasting},
  author={Choi, Jeongwhan and Choi, Hwangyong and Hwang, Jeehyun and Park, Noseong},
  booktitle={Proceedings of the AAAI conference on artificial intelligence},
  volume={36},
  number={6},
  pages={6367--6374},
  year={2022}
}

@article{niu2023understanding,
  title={Understanding temporal and spatial patterns of urban activities across demographic groups through geotagged social media data},
  author={Niu, Haifeng and Silva, Elisabete A},
  journal={Computers, Environment and Urban Systems},
  volume={100},
  pages={101934},
  year={2023},
  publisher={Elsevier}
}

@article{street2018slingshot,
  title={Slingshot: cell lineage and pseudotime inference for single-cell transcriptomics},
  author={Street, Kelly and Risso, Davide and Fletcher, Russell B and Das, Diya and Ngai, John and Yosef, Nir and Purdom, Elizabeth and Dudoit, Sandrine},
  journal={BMC Genomics},
  volume={19},
  number={1},
  pages={477},
  year={2018},
  publisher={Springer}
}

@article{lee2025multimarginal,
  title={Multi-Marginal Stochastic Flow Matching for High-Dimensional Snapshot Data at Irregular Time Points},
  author={Lee, Justin and Moradijamei, Behnaz and Shakeri, Heman},
  journal={arXiv preprint arXiv:2508.04351},
  year={2025}
}

@inproceedings{banerjee2025efficient,
  title={Efficient Trajectory Inference in Wasserstein Space Using Consecutive Averaging},
  author={Banerjee, Amartya and Lee, Harlin and Sharon, Nir and Moosm{\"u}ller, Caroline},
  booktitle={International Conference on Artificial Intelligence and Statistics},
  year={2025},
}

@article{yick2008wireless,
  title     = {Wireless sensor network survey},
  author    = {Yick, Jennifer and Mukherjee, Biswanath and Ghosal, Dipak},
  journal   = {Computer Networks},
  volume    = {52},
  number    = {12},
  pages     = {2292--2330},
  year      = {2008},
  publisher = {Elsevier}
}

@article{zheng2014urban,
  author     = {Zheng, Yu and Capra, Licia and Wolfson, Ouri and Yang, Hai},
  title      = {Urban Computing: Concepts, Methodologies, and Applications},
  year       = {2014},
  issue_date = {September 2014},
  publisher  = {Association for Computing Machinery},
  address    = {New York, NY, USA},
  volume     = {5},
  number     = {3},
  issn       = {2157-6904},
  doi        = {10.1145/2629592},
  journal    = {ACM Trans. Intell. Syst. Technol.},
  month      = sep,
  articleno  = {38},
  numpages   = {55}
}

@article{Bargaje2017cell,
  author  = {Rhishikesh Bargaje  and Kalliopi Trachana  and Martin N. Shelton  and Christopher S. McGinnis  and Joseph X. Zhou  and Cora Chadick  and Savannah Cook  and Christopher Cavanaugh  and Sui Huang  and Leroy Hood },
  title   = {Cell population structure prior to bifurcation predicts efficiency of directed differentiation in human induced pluripotent cells},
  journal = {Proceedings of the National Academy of Sciences},
  volume  = {114},
  number  = {9},
  pages   = {2271-2276},
  year    = {2017},
  doi     = {10.1073/pnas.1621412114},
}

@inproceedings{neklyudov2023action,
  title        = {Action matching: Learning stochastic dynamics from samples},
  author       = {Neklyudov, Kirill and Brekelmans, Rob and Severo, Daniel and Makhzani, Alireza},
  booktitle    = {International conference on machine learning},
  pages        = {25858--25889},
  year         = {2023},
  organization = {PMLR}
}

@inproceedings{chen2023deep,
  title     = {Deep multi-marginal momentum {S}chr\"odinger bridge},
  author    = {Chen, Tianrong and Liu, Guan-horng and Tao, Molei and Theodorou, Evangelos A},
  booktitle = {Proceedings of the 37th International Conference on Neural Information Processing Systems},
  pages     = {57058--57086},
  year      = {2023}
}

@inproceedings{rohbeck2025modeling,
  title     = {Modeling Complex System Dynamics with Flow Matching Across Time and Conditions},
  author    = {Martin Rohbeck and Charlotte Bunne and Edward De Brouwer and Jan-Christian Huetter and Anne Biton and Kelvin Y. Chen and Aviv Regev and Romain Lopez},
  booktitle = {The Thirteenth International Conference on Learning Representations},
  year      = {2025},
}

@article{chen2023flow,
  title   = {Flow matching on general geometries},
  author  = {Chen, Ricky TQ and Lipman, Yaron},
  journal = {arXiv preprint arXiv:2302.03660},
  year    = {2023}
}

@article{bonneel2019spot,
  author     = {Bonneel, Nicolas and Coeurjolly, David},
  title      = {SPOT: sliced partial optimal transport},
  year       = {2019},
  issue_date = {August 2019},
  publisher  = {Association for Computing Machinery},
  address    = {New York, NY, USA},
  volume     = {38},
  number     = {4},
  issn       = {0730-0301},
  doi        = {10.1145/3306346.3323021},
  journal    = {ACM Trans. Graph.},
  month      = jul,
  articleno  = {89},
  numpages   = {13}
}

@article{leonard2013survey,
  title   = {A survey of the {S}chr\"odinger problem and some of its connections with optimal transport},
  author  = {L{\'e}onard, Christian},
  journal = {arXiv preprint arXiv:1308.0215},
  year    = {2013}
}

@article{Bergen2020,
  title     = {Generalizing RNA velocity to transient cell states through dynamical modeling},
  volume    = {38},
  issn      = {1546-1696},
  doi       = {10.1038/s41587-020-0591-3},
  number    = {12},
  journal   = {Nature Biotechnology},
  publisher = {Springer Science and Business Media LLC},
  author    = {Bergen, Volker and Lange, Marius and Peidli, Stefan and Wolf, F. Alexander and Theis, Fabian J.},
  year      = {2020},
  month     = aug,
  pages     = {1408-1414}
}

@article{la2018rna,
  title     = {RNA velocity of single cells},
  author    = {La Manno, Gioele and Soldatov, Ruslan and Zeisel, Amit and Braun, Emelie and Hochgerner, Hannah and Petukhov, Viktor and Lidschreiber, Katja and Kastriti, Maria E and L\"onnerberg, Peter and Furlan, Alessandro and others},
  journal   = {Nature},
  volume    = {560},
  number    = {7719},
  pages     = {494--498},
  year      = {2018},
  publisher = {Nature Publishing Group UK London}
}

@article{qiu2017reversed,
  title     = {Reversed graph embedding resolves complex single-cell trajectories},
  author    = {Qiu, Xiaojie and Mao, Qi and Tang, Ying and Wang, Li and Chawla, Raghav and Pliner, Hannah A and Trapnell, Cole},
  journal   = {Nature Methods},
  volume    = {14},
  number    = {10},
  pages     = {979--982},
  year      = {2017},
  publisher = {Nature Publishing Group UK London}
}

@article{trapnell2014dynamics,
  title     = {The dynamics and regulators of cell fate decisions are revealed by pseudotemporal ordering of single cells},
  author    = {Trapnell, Cole and Cacchiarelli, Davide and Grimsby, Jonna and Pokharel, Prapti and Li, Shuqiang and Morse, Michael and Lennon, Niall J and Livak, Kenneth J and Mikkelsen, Tarjei S and Rinn, John L},
  journal   = {Nature Biotechnology},
  volume    = {32},
  number    = {4},
  pages     = {381--386},
  year      = {2014},
  publisher = {Nature Publishing Group US New York}
}

@article{klein2024genot,
  title   = {{GENOT: Entropic (Gromov) Wasserstein flow matching with applications to single-cell genomics}},
  author  = {Klein, Dominik and Uscidda, Th{\'e}o and Theis, Fabian and Cuturi, Marco},
  journal = {Advances in Neural Information Processing Systems},
  volume  = {37},
  pages   = {103897--103944},
  year    = {2024}
}

@article{sha2024reconstructing,
  title     = {Reconstructing growth and dynamic trajectories from single-cell transcriptomics data},
  author    = {Sha, Yutong and Qiu, Yuchi and Zhou, Peijie and Nie, Qing},
  journal   = {Nature Machine Intelligence},
  volume    = {6},
  number    = {1},
  pages     = {25--39},
  year      = {2024},
  publisher = {Nature Publishing Group UK London}
}

@book{ahuja1993network,
  title     = {Network flows: theory, algorithms, and applications},
  author    = {Ahuja, Ravindra K and Magnanti, Thomas L and Orlin, James B and others},
  volume    = {1},
  year      = {1993},
  publisher = {Prentice hall Englewood Cliffs, NJ}
}

@article{figalli2010optimal,
  title     = {The optimal partial transport problem},
  author    = {Figalli, Alessio},
  journal   = {Archive for Rational Mechanics and Analysis},
  volume    = {195},
  number    = {2},
  pages     = {533--560},
  year      = {2010},
  publisher = {Springer}
}

@article{bw1986density,
author="SILVERMAN, BW",
title="Density estimation for statistics and data analysis",
journal="Monographs on Statistics and Applied Probability",
publisher="Chapman & Hall",
year="1986",
}

@article{korotin2023light,
  title   = {Light {S}chr\"odinger bridge},
  author  = {Korotin, Alexander and Gushchin, Nikita and Burnaev, Evgeny},
  journal = {arXiv preprint arXiv:2310.01174},
  year    = {2023}
}

@article{loshchilov2017decoupled,
  title   = {Decoupled weight decay regularization},
  author  = {Loshchilov, I},
  journal = {arXiv preprint arXiv:1711.05101},
  year    = {2017}
}

@article{gushchin2024entropic,
  title   = {Entropic neural optimal transport via diffusion processes},
  author  = {Gushchin, Nikita and Kolesov, Alexander and Korotin, Alexander and Vetrov, Dmitry P and Burnaev, Evgeny},
  journal = {Advances in Neural Information Processing Systems},
  volume  = {36},
  year    = {2024}
}

@inproceedings{lipman2023flow,
  title     = {Flow Matching for Generative Modeling},
  author    = {Lipman, Yaron and Chen, Ricky TQ and Ben-Hamu, Heli and Nickel, Maximilian and Le, Matt},
  booktitle = {11th International Conference on Learning Representations},
  year      = {2023}
}

@inproceedings{liu2023flow,
  title     = {Flow Straight and Fast: Learning to Generate and Transfer Data with Rectified Flow},
  author    = {Liu, Xingchao and Gong, Chengyue and Liu, Qiang},
  booktitle = {11th International Conference on Learning Representations},
  year      = {2023}
}

@article{sakoe1978dynamic,
  title     = {Dynamic programming algorithm optimization for spoken word recognition},
  author    = {Sakoe, Hiroaki and Chiba, Seibi},
  journal   = {IEEE Transactions on Acoustics, Speech, and Signal Processing},
  volume    = {26},
  number    = {1},
  pages     = {43--49},
  year      = {1978},
  publisher = {IEEE}
}

@inproceedings{tong2024simulation,
  title        = {Simulation-Free {S}chr\"odinger Bridges via Score and Flow Matching},
  author       = {Tong, Alexander Y and Malkin, Nikolay and Fatras, Kilian and Atanackovic, Lazar and Zhang, Yanlei and Huguet, Guillaume and Wolf, Guy and Bengio, Yoshua},
  booktitle    = {International Conference on Artificial Intelligence and Statistics},
  pages        = {1279--1287},
  year         = {2024},
  organization = {PMLR}
}

@article{kapusniak2024metric,
  title   = {Metric Flow Matching for Smooth Interpolations on the Data Manifold},
  author  = {Kapusniak, Kacper and Potaptchik, Peter and Reu, Teodora and Zhang, Leo and Tong, Alexander and Bronstein, Michael and Bose, Avishek Joey and Di Giovanni, Francesco},
  journal = {arXiv preprint arXiv:2405.14780},
  year    = {2024}
}

@article{flamary2021pot,
  title   = {Pot: Python optimal transport},
  author  = {Flamary, R{\'e}mi and Courty, Nicolas and Gramfort, Alexandre and Alaya, Mokhtar Z and Boisbunon, Aur{\'e}lie and Chambon, Stanislas and Chapel, Laetitia and Corenflos, Adrien and Fatras, Kilian and Fournier, Nemo and others},
  journal = {Journal of Machine Learning Research},
  volume  = {22},
  number  = {78},
  pages   = {1--8},
  year    = {2021}
}

@article{tong2024improving,
  title   = {Improving and generalizing flow-based generative models with minibatch optimal transport},
  author  = {Tong, Alexander and FATRAS, Kilian and Malkin, Nikolay and Huguet, Guillaume and Zhang, Yanlei and Rector-Brooks, Jarrid and Wolf, Guy and Bengio, Yoshua},
  journal = {Transactions on Machine Learning Research},
  year    = {2024}
}

@article{moon2019visualizing,
  title     = {Visualizing structure and transitions in high-dimensional biological data},
  author    = {Moon, Kevin R and van Dijk, David and Wang, Zheng and Gigante, Scott and Burkhardt, Daniel B and Chen, William S and Yim, Kristina and Elzen, Antonia van den and Hirn, Matthew J and Coifman, Ronald R and others},
  journal   = {Nature Biotechnology},
  volume    = {37},
  number    = {12},
  pages     = {1482--1492},
  year      = {2019},
  publisher = {Nature Publishing Group US New York}
}

@article{bunne2023learning,
  title     = {Learning single-cell perturbation responses using neural optimal transport},
  author    = {Bunne, Charlotte and Stark, Stefan G and Gut, Gabriele and Del Castillo, Jacobo Sarabia and Levesque, Mitch and Lehmann, Kjong-Van and Pelkmans, Lucas and Krause, Andreas and R{\"a}tsch, Gunnar},
  journal   = {Nature Methods},
  volume    = {20},
  number    = {11},
  pages     = {1759--1768},
  year      = {2023},
  publisher = {Nature Publishing Group US New York}
}

@inproceedings{tong2020trajectorynet,
  title      = {TrajectoryNet: A Dynamic Optimal Transport Network for Modeling Cellular Dynamics},
  shorttitle = {TrajectoryNet},
  booktitle  = {Proceedings of the 37th International Conference on Machine Learning},
  author     = {Tong, Alexander and Huang, Jessie and Wolf, Guy and {van Dijk}, David and Krishnaswamy, Smita},
  year       = {2020}
}

@article{koshizuka2022neural,
  title   = {Neural Lagrangian {S}chr\"odinger Bridge: Diffusion Modeling for Population Dynamics},
  author  = {Koshizuka, Takeshi and Sato, Issei},
  journal = {arXiv preprint arXiv:2204.04853},
  year    = {2022}
}

@article{peyre2019computational,
  title     = {Computational optimal transport: With applications to data science},
  author    = {Peyr{\'e}, Gabriel and Cuturi, Marco and others},
  journal   = {Foundations and Trends{\textregistered} in Machine Learning},
  volume    = {11},
  number    = {5-6},
  pages     = {355--607},
  year      = {2019},
  publisher = {Now Publishers, Inc.}
}

@article{moon2018embryoid,
  title   = {Embryoid Body data for {PHATE}},
  author  = {Moon, Kevin},
  journal = {Mendeley Data},
  volume  = {1},
  year    = {2018}
}

@article{kingma2014adam,
  title   = {Adam: A method for stochastic optimization},
  author  = {Kingma, Diederik P and Ba, Jimmy},
  journal = {arXiv preprint arXiv:1412.6980},
  year    = {2014}
}

@inproceedings{klambauer2017self,
  author    = {Klambauer, G\"unter and Unterthiner, Thomas and Mayr, Andreas and Hochreiter, Sepp},
  booktitle = {Advances in Neural Information Processing Systems},
  editor    = {I. Guyon and U. Von Luxburg and S. Bengio and H. Wallach and R. Fergus and S. Vishwanathan and R. Garnett},
  pages     = {},
  publisher = {Curran Associates, Inc.},
  title     = {Self-Normalizing Neural Networks},
  volume    = {30},
  year      = {2017}
}

@inproceedings{chen2018neural,
  author    = {Chen, Ricky T. Q. and Rubanova, Yulia and Bettencourt, Jesse and Duvenaud, David K},
  booktitle = {Advances in Neural Information Processing Systems},
  editor    = {S. Bengio and H. Wallach and H. Larochelle and K. Grauman and N. Cesa-Bianchi and R. Garnett},
  pages     = {},
  publisher = {Curran Associates, Inc.},
  title     = {Neural Ordinary Differential Equations},
  volume    = {31},
  year      = {2018}
}

@inproceedings{cuturi2013sinkhorn,
  author    = {Cuturi, Marco},
  booktitle = {Advances in Neural Information Processing Systems},
  editor    = {C.J. Burges and L. Bottou and M. Welling and Z. Ghahramani and K.Q. Weinberger},
  pages     = {},
  publisher = {Curran Associates, Inc.},
  title     = {Sinkhorn Distances: Lightspeed Computation of Optimal Transport},
  volume    = {26},
  year      = {2013}
}

@article{zheng2017massively,
  title     = {Massively parallel digital transcriptional profiling of single cells},
  author    = {Zheng, Grace XY and Terry, Jessica M and Belgrader, Phillip and Ryvkin, Paul and Bent, Zachary W and Wilson, Ryan and Ziraldo, Solongo B and Wheeler, Tobias D and McDermott, Geoff P and Zhu, Junjie and others},
  journal   = {Nature Communications},
  volume    = {8},
  number    = {1},
  pages     = {14049},
  year      = {2017},
  publisher = {Nature Publishing Group UK London}
}

@inproceedings{song2021scorebased,
  title     = {Score-Based Generative Modeling through Stochastic Differential Equations},
  author    = {Yang Song and Jascha Sohl-Dickstein and Diederik P Kingma and Abhishek Kumar and Stefano Ermon and Ben Poole},
  booktitle = {International Conference on Learning Representations},
  year      = {2021},
}

\end{document}